\documentclass[10pt,fullpage,letterpaper]{article}

\usepackage{etoolbox}
\newtoggle{arxiv}
\toggletrue{arxiv}

\usepackage{graphicx, setspace, latexsym,amsmath,amssymb,color}
\usepackage{amsthm}
\usepackage{quoting}
\usepackage{dsfont}
\usepackage{bbm}
\usepackage{longtable}
\usepackage{hyperref}
\usepackage{algorithm}
\usepackage{lineno}
\usepackage{makecell}
\usepackage{parskip}
\usepackage[normalem]{ulem}
\usepackage{caption}
\usepackage{subcaption}
\usepackage{algpseudocode}
\usepackage{cleveref}
\usepackage{todonotes}
\usepackage{wrapfig}
\usepackage{xcolor}
\usepackage{enumitem,kantlipsum}
\newcommand{\loose}{\looseness=-1}
\usepackage{xspace}
\usepackage{comment}
\usepackage{array}

\newcommand{\mc}[1]{\mathcal{#1}}

\newcommand{\lamdiff}{\lambda_{\mathrm{diff}}}
\newcommand{\lamfair}{\lambda_{\mathrm{fair}}}

\def \ours{\textsc{Fare}\xspace}
\def \E{\mathbb E}
\def \P{\mathbb P}
\def \R{\mathbb R}

\DeclareMathOperator*{\argmin}{arg\,min}

\def \st2{\texttt{\texttt{$($ST$)^2$}}}

\def \X{{\cal X}}

\def \V{{\cal V}}

\def \H{{\cal H}}

\def \1{\mathbbm{1}}
\def \0{\textbf{0}}

\newtheorem{theorem}{Theorem}[section]

\newtheorem{proposition}[theorem]{Proposition}

\newtheorem{definition}{Definition}

\usepackage[round, sort&compress]{natbib}

\usepackage[utf8]{inputenc} % allow utf-8 input
\usepackage[T1]{fontenc}    % use 8-bit T1 fonts
\usepackage{hyperref}       % hyperlinks
\usepackage{url}            % simple URL typesetting
\usepackage{booktabs}       % professional-quality tables
\usepackage{amsfonts}       % blackboard math symbols
\usepackage{nicefrac}       % compact symbols for 1/2, etc.
\usepackage{microtype}      % microtypography
\usepackage{xcolor}         % colors

\begin{document}

\author{%
  Romain Camilleri,
  Andrew Wagenmaker,
  Jamie Morgenstern,
  Lalit Jain,
  Kevin Jamieson\\
  University of Washington, Seattle, WA\\
  \texttt{\{camilr,ajwagen,jamiemmt,jamieson\}@cs.washington.edu,lalitj@uw.edu}
}
\title{Fair Active Learning in Low-Data Regimes}
\maketitle

\begin{abstract}

In critical machine learning applications, ensuring fairness is essential to avoid perpetuating social inequities. In this work, we address the challenges of reducing bias and improving accuracy in data-scarce environments, where the cost of collecting labeled data prohibits the use of large, labeled datasets. In such settings, active learning promises to maximize marginal accuracy gains of small amounts of labeled data. However, existing applications of active learning for fairness fail to deliver on this, typically requiring large labeled datasets, or failing to ensure the desired fairness tolerance is met on the population distribution.

To address such limitations, we introduce an innovative active learning framework that combines an exploration procedure inspired by posterior sampling with a fair classification subroutine. We demonstrate that this framework performs effectively in very data-scarce regimes, maximizing accuracy while satisfying fairness constraints with high probability. We evaluate our proposed approach using well-established real-world benchmark datasets and compare it against state-of-the-art methods, demonstrating its effectiveness in producing fair models, and improvement over existing methods.

\end{abstract}

% \usepackage{etoolbox}
% \newtoggle{arxiv}
% \toggletrue{arxiv}
% \togglefalse{arxiv}

\section{INTRODUCTION}
As machine learning models proliferate and are used in an ever-increasing number of applications with societal ramifications, it has become increasingly important to have robust methods for developing models that do not perpetuate existing social inequities. 
Over the last few years, a plethora of works in fair classification have provided a principled toolkit to develop classifiers and quantify their performance under various fairness metrics. 
These metrics, including equal opportunity and equalized odds, give a natural way to ensure that favorable outcomes such as model performance or predicted positive rates are equalized across different groups for a given protected feature. 
More precisely, given a distribution $\nu$ on $\mc{X}\times\mc{A}\times \mc{Y}$ (where $\mc{X}$ is the feature space, $\mc{A}$ the protected attribute space and $\mc{Y}$ the label space), a hypothesis class $\mc{H}$, a fairness metric $m_{\rm fair}$, a measure of its violation $L^{m_{\rm fair}}_\nu(h)$, and a fairness violation tolerance $\alpha$; the goal in fair classification is to return $\arg\min_{h \in \mc{H}} \E_{(x,a,y)\sim \nu}[h(x)\neq y] \text{ subject to }L^{m_{\rm fair}}_\nu(h) \leq \alpha$. 
 
 In practice, as $\nu$ is unknown, solving an empirical analog of this constrained classification problem on a training set
is a natural approach to learning classifiers that generalize well to a test set, while maintaining fairness guarantees. Indeed, the focus of much of the fairness literature has been to develop optimization methods to solve such a problem \citep{agarwal2018reductions, donini2018empirical, cotter2018training}. While this is a reasonable approach when a large amount of labeled training data is available, in many applications such large amounts of data are not available, and it can be prohibitively expensive to collect more. In such settings existing approaches may not be able to guarantee accurate classifiers, or may return classifiers that are in fact unfair on the population distribution.

%In this work, we are interested in fair classification in this \emph{low-data} regime. 
A promising approach to handle such low-data regimes and maximize the effectiveness of small amounts of labeled data is \emph{active learning}. Active learning methods aim to minimize the amount of labeled training data needed by only requesting labels for the most \emph{informative} examples, thereby significantly reducing the label complexity while ensuring similar accuracy of the learned classifier. While active learning methods have been applied to fair classification before, existing works either require large labeled datasets for pretraining, thereby eliminating the primary benefit of active learning, or are unable to satisfy the goal fairness constraint. 

In this work we aim to overcome these challenges and develop methods for fair active learning which do not require large pretraining datasets---truly operating in the low-data regime---and ensure fairness constraints are met. Our contributions are as follows:
\begin{enumerate}[leftmargin=*]
\item We propose a novel approach to fair active learning, \ours, 
% \awcomment{algname---we should name our alg maybe} 
which chooses which points to label by combining a posterior sampling-inspired randomized exploration procedure that aims to improve classifier accuracy, with a group-dependent sampling procedure to ensure fairness is met.
%and a reduction to fair learning oracles to obtain a fair classifier from this data. 
Notably, our approach does not require a large pretraining dataset, and is able to produce accurate and fair classifiers in the very low data regime. 
\item We evaluate our proposed method on a variety of standard benchmark datasets from the fairness community, and demonstrate that it yields large label complexity gains over passive approaches while ensuring fairness constraints are met, and also significantly outperforms the existing state-of-the-art approaches for fair active learning.
\end{enumerate}
To the best of our knowledge, our proposed approach is the first active learning procedure able to ensure fairness constraints are reliably met without requiring large amounts of labeled data.

\section{RELATED WORK}

\paragraph{Fairness.}
Algorithmic fairness has garnered significant interest in recent years (see \citet{barocas2017fairness, hort2022bia} for recent surveys). Approaches to mitigate fairness disparities can be grouped into three lines of work: pre-processing, in-processing, and post-processing. Pre-processing aims to remove disparate impact by modifying the training data\citep{kamiran2012data}, while post-processing modifies already learned classifiers to improve fairness \citep{hardt2016equality}. 
Of particular interest to our work is in-processing for bias mitigation, where the focus is on modifying the learning process to build fair classifiers \citep{zhang2018mitigating}. Most relevant to us within in-processing bias mitigation techniques are works that have approached fairness mitigations in classification as a constrained optimization problem \citep{agarwal2018reductions, donini2018empirical}. Our fairness metrics of interest---equal opportunity and equalized odds---were introduced as operationalizations of fairness concurrently by~\citet{hardt2016equality, kleinberg2016inherent}; see also \cite{kearns2018preventing}.
\paragraph{Active learning.}
The expense associated with labeling data has emerged as a significant obstacle in the practical implementation of machine learning methods. Motivated by this, there has been growing attention towards the concept of \emph{active} classification, which involves presenting the learner with a set of unlabeled examples, and tasking them with producing a precise hypothesis after querying as few labels as possible \citep{settles2011theories}.
Active learning has been studied extensively over the past five decades (see the survey \citet{hanneke2014theory}). Most active learning approaches select samples to label based on some notion of uncertainty (e.g., entropy of predictions, margin, disagreement \citep{cohn1994improving, beygelzimer2009importance}). 
Recent breakthroughs have connected best-arm identification for linear bandits with classification, opening up new possibilities for active learning via \emph{experiment design} \citep{katz2021improved, camilleri2022active, camilleri2021selective}. 

\paragraph{Fair active learning.}
The problem of fair active classification 
has been previously considered by recent efforts to reach a classifiers with good ``fairness-error'' trade-off given a label budget, including \citet{anahideh2021fair, sharaf2022promoting, fajri2022falcur}. As we will see experimentally, these works suffer from a variety of shortcomings: for example, poor generalization of their fairness violation, minimal accuracy gains over baseline methods, or limited ability to handle standard group fairness metrics. Furthermore, their objective is somewhat different than ours. While we aim to return a classifier with fairness violation below a desired tolerance (motivated by situations where it is critical to ensure our classifier satisfies a given fairness constraint), these works instead aim to quantify the general tradeoff between fairness and accuracy, without ensuring the returned classifier is below any tolerance. Last, these works assume the existence of large, pre-existing, labeled datasets: namely for their experiments on the \texttt{Adult income} dataset \citet{anahideh2021fair, sharaf2022promoting, fajri2022falcur} assume respectively that $2000, 15000, 3000$ labels are accessible. We will see that the gains from our active learning algorithms are instead visible after collecting $100$ labels.
Other works, such as \citet{cao2022fairness} focuses on fair active learning for decoupled models and \citet{shen2022metric, cao2022active}, have focused on the analogous problem of finding classifiers that meet \emph{metric}-fair constraints, while \citet{abernethy2020active, shekhar2021adaptive, cai2022adaptive, branchaudcharron2021active} have focused on data collection for \emph{min-max} fairness. The nature of min-max fairness 
does not explicitly constrain the differences in quantities between groups, instead improving the quantity for the worst-off group as much as possible. These, alongside the metric fairness constraints, are significantly different than the group fairness metrics we consider, and as such motivate an entirely different set of methods.

Another related line of work is that of bandits with constraints \citep{sui2015safe,kazerouni2017conservative,pacchiano2021stochastic,wang2022best,camilleri2022active}. As noted, classification can be modeled as a bandit problem and in some cases bandit algorithms can be applied to active learning for classification. Furthermore, imposing unknown constraints in bandit problems is similar to imposing fairness constraints in classification. To the best of our knowledge, however, existing work on constrained bandits does not consider constraints expressive enough to encode standard fairness metrics such as equalized odds and equal opportunity.\loose

\newcommand{\unif}{\mathrm{Unif}}
\newcommand{\cDmy}{\mathcal{D}_{\mathrm{tr}}^{\backslash y}}
\newcommand{\cD}{\mathcal{D}}

\section{PRELIMINARIES}\label{sec:FAL}

In this work, we focus on a binary classification scenario where each data point consists of three elements $(x, a, y)$. Here, $x \in \X \subset \R^d$ represents a $d$-dimensional feature vector, $a \in \{0, 1\}$ indicates a binary protected attribute which partitions our data into two \emph{groups}, and $y \in \{0, 1\}$ denotes a label. 
In the general classification paradigm, we assume that the training set $\mc{D} = \{(x_1, a_1, y_1), \ldots,(x_n, a_n, y_n)\} \sim \nu \in \triangle_{\X\times \{0, 1\}\times \{0, 1\}}$ is a set of $n$ examples sampled from a target distribution $\nu$.
% \awcomment{not super clear if one can observe $a$ or not during testing}
The objective is to learn from the training set $\mc{D}$ a classifier $h: \X \mapsto \{0, 1\}$ among a hypothesis set $\H$ (e.g. linear classifiers or random forests) which has the lowest risk $R_{\nu}(h)$ possible on the target distribution. Here the risk is defined for any distribution $\nu\in \triangle_{\X\times \{0, 1\}\times \{0, 1\}}$ as $R_\nu(h) := \E_{(x, a, y) \sim \nu}[\1\{h(x)\neq y\}]$.
%Note that although the feature vector $x$ may include the protected attribute $a$ as one of its features, or contain other features that are indicative of $a$, the protected attribute $a$ may not be observed during testing.\iffalse \jmcomment{not necessary}  The fair classification task can still be solved without such observation \citep{agarwal2018reductions}.\fi
% \awcomment{I'm still confused by this. don't we always assume we know $a$? how does one do fair classification if you don't even know the protected attribute?}
% \awcomment{I think you should define loss at the start of section 3} among a function class $\H$. 
\vspace{-.2cm}

\subsection{Definitions of Fairness}
In this work we consider in particular two well-known definitions of fairness:
%that have been considered in previous work: 
Equal Opportunity---also called True Positive Rate Parity (TPRP)---and Equalized Odds (EO), though our method extends to other notions of fairness as well. We formally define these here. 
%\vspace{-.2cm}
\iftoggle{arxiv}{}{\vspace{-.7cm}}

% \begin{definition}[Fairness Definitions (EO, TPRP)]
%      Given a tolerance $\alpha \in [0, 1]$ and target distribution $\nu$, a classifier $h \in \H$ is said to satisfy True Positive Rate Parity up to $\alpha$ on $\nu$ if 
%      % \awcomment{the second term should have $h(x)$ not $h(1)$, right?} 
% \vspace{-.2cm}
%     \begin{align}\label{eq:tprp}
%         &| P_{(x, a, y)\sim\nu}(h(x)=1|a=0, y=1) \nonumber\\
%         &\qquad- P_{(x, a, y)\sim\nu}(h(x)=1|a=1, y=1) | \leq \alpha.
% \vspace{-.3cm}
%     \end{align}
%     A classifier satisfies Equalized Odds up to $\alpha$ on a distribution $\nu$ if, in addition to satisfying \eqref{eq:tprp} it also satisfies
%     %and False Positive Rate Parity up to $\alpha$ on a distribution $\nu$ if  
%     % \awcomment{this equation is identical to the one above}
% \vspace{-.3cm}
%     \begin{align}\label{eq:fprp}
%         &| P_{(x, a, y)\sim\nu}(h(x)=1|a=0, y=0) \nonumber\\
%         &\qquad- P_{(x, a, y)\sim\nu}(h(x)=1|a=1, y=0) | \leq \alpha.
%     \end{align}
%     %A classifier $h \in \H$ satisfies Equalized Odds up to $\alpha$ on a distribution $\nu$ if it satisfies both True Positive Rate Parity and False Positive Rate Parity up to $\alpha$ on $\nu$.
%     % 
%     % \awcomment{it's sort of confusing when I read this defn because you say you are going to define EO, and then the first thing you define is TPRP. maybe just clarify? Also, if TPRP/FPRP are really just eq (1) and (2), then it might be good to define them, and then define EO as being that they both hold}
% \end{definition}
% \vspace{-.3cm}

\begin{definition}[Fairness Definitions (EO, TPRP)]
     Given a tolerance $\alpha \in [0, 1]$ and target distribution $\nu$, a classifier $h \in \H$ is said to satisfy True Positive Rate Parity up to $\alpha$ on $\nu$ if 
     % \awcomment{the second term should have $h(x)$ not $h(1)$, right?} 
% \vspace{-.2cm}
    \iftoggle{arxiv}{
    \begin{align}\label{eq:tprp}
        &| P_{(x, a, y)\sim\nu}(h(x)=1|a=0, y=1) - P_{(x, a, y)\sim\nu}(h(x)=1|a=1, y=1) | \leq \alpha.
% \vspace{-.3cm}
    \end{align}}{
    \begin{align}\label{eq:tprp}
        &| P_{(x, a, y)\sim\nu}(h(x)=1|a=0, y=1) \nonumber\\
        &\qquad- P_{(x, a, y)\sim\nu}(h(x)=1|a=1, y=1) | \leq \alpha.
% \vspace{-.3cm}
    \end{align}
    }
    A classifier satisfies Equalized Odds up to $\alpha$ on a distribution $\nu$ if, in addition to satisfying \eqref{eq:tprp} it also satisfies
    %and False Positive Rate Parity up to $\alpha$ on a distribution $\nu$ if  
    % \awcomment{this equation is identical to the one above}
% \vspace{-.3cm}
    \iftoggle{arxiv}{
    \begin{align}\label{eq:fprp}
        &| P_{(x, a, y)\sim\nu}(h(x)=1|a=0, y=0) - P_{(x, a, y)\sim\nu}(h(x)=1|a=1, y=0) | \leq \alpha.
    \end{align}
    }{
    \begin{align}\label{eq:fprp}
        &| P_{(x, a, y)\sim\nu}(h(x)=1|a=0, y=0) \nonumber\\
        &\qquad- P_{(x, a, y)\sim\nu}(h(x)=1|a=1, y=0) | \leq \alpha.
    \end{align}
    }
    %A classifier $h \in \H$ satisfies Equalized Odds up to $\alpha$ on a distribution $\nu$ if it satisfies both True Positive Rate Parity and False Positive Rate Parity up to $\alpha$ on $\nu$.
    % 
    % \awcomment{it's sort of confusing when I read this defn because you say you are going to define EO, and then the first thing you define is TPRP. maybe just clarify? Also, if TPRP/FPRP are really just eq (1) and (2), then it might be good to define them, and then define EO as being that they both hold}
\end{definition}
% \vspace{-.3cm}

% Note that the definitions of TPRP and FPRP are analogous as the condition for TPRP (resp. FPRP) to hold is the one of equation~\eqref{eq:tprp} (resp. equation~\eqref{eq:fprp}) \awcomment{what does this mean? unclear}. For the sake of brevity we defer the formal definitions of TPRP and FPRP to the appendix. 
If $\alpha = 0$, EO states that the prediction $h(x)$ is conditionally independent of the protected attribute $a$ given the label $y$. 
With these definitions of fairness in mind, we also define the fairness violation of a given classifier as the left-hand sides of equations \eqref{eq:tprp} and \eqref{eq:fprp}.
\vspace{-.1cm}

\begin{definition}[Fairness violation]
We define the EO (resp. TPRP) violation of classifier $h$ on distribution $\nu$ as  
% \awcomment{a minor notational thing, but it might look nicer to use $L_\nu^{\rm TP}$ instead of TPRP (and similarly for FPRP)}
% \awcomment{the you have here is always 0}
\iftoggle{arxiv}{
\begin{align*}
    L^{\rm EO}_\nu(h) & := \max_{z\in\{0, 1\}}| P_{(x, a, y)\sim\nu}(h(x)=1|a=0, y=z) - P_{(x, a, y)\sim\nu}(h(x)=1|a=1, y=z) |, \\
    L^{\rm TP}_\nu(h) & := | P_{(x, a, y)\sim\nu}(h(x)=1|a=0, y=1) 
    - P_{(x, a, y)\sim\nu}(h(x)=1|a=1, y=1) |.
    %L^{\rm FP}_\nu(h) & := | P_{(x, a, y)\sim\nu}(h(x)=1|a=0, y=0) 
    %\nonumber\\ &\qquad - P_{(x, a, y)\sim\nu}(h(x)=1|a=1, y=0) |.
\end{align*}
}{
\begin{align*}
    L^{\rm EO}_\nu(h) & := \max_{z\in\{0, 1\}}| P_{(x, a, y)\sim\nu}(h(x)=1|a=0, y=z) \nonumber\\ &\qquad - P_{(x, a, y)\sim\nu}(h(x)=1|a=1, y=z) |, \\
    L^{\rm TP}_\nu(h) & := | P_{(x, a, y)\sim\nu}(h(x)=1|a=0, y=1) 
    \nonumber\\ &\qquad - P_{(x, a, y)\sim\nu}(h(x)=1|a=1, y=1) |.
    %L^{\rm FP}_\nu(h) & := | P_{(x, a, y)\sim\nu}(h(x)=1|a=0, y=0) 
    %\nonumber\\ &\qquad - P_{(x, a, y)\sim\nu}(h(x)=1|a=1, y=0) |.
\end{align*}
}
\end{definition}
Given some threshold $\alpha$, a \emph{fair classifier} is a classifier with fairness violation below $\alpha$.\loose

\newcommand{\cDtest}{\mathcal{D}_{\mathrm{test}}}
\subsection{Problem Statement}
Classical machine learning typically deals with the setting where the learner has access to a fixed, labeled dataset, $\mc{D}_{\mathrm{tr}}$, and must learn as accurate a classifier as possible from this data. In this work, we are interested in the \emph{active} setting where the goal of the learner is to train on as few labeled data points as possible to obtain a desired accuracy.
In particular, in the pool-based active learning setting, the task of fair active classification is the following sequential problem. 
% \todor{I need to assume that the target distribution is known in order to define IPS estimates. How do we argue that? Can we just assume that $\nu_{\rm te}$ is uniform for the whole section?} 
First, the learner is given an unlabeled training pool of data $\cDmy \subseteq \X \times \mathcal{A}$ and some fairness metric $m_{\rm{fair}} \in \{EO, TP\}$ with target fairness violation $\alpha$. At each time $t=1,2, \ldots, T$ the agent then chooses any unlabeled point from the pool $(x_t,a_t) \in \cDmy$ and requests its label $y_t\in\{0, 1\}$. 
After requesting $T$ labels, the agent outputs a classifier $h\in \H$. Its performance is evaluated via the two following metrics: error loss $R_{\nu}(h)$ and fairness violation $L^{m_{\rm{fair}}}_{\nu}(h)$, for $\nu$ the population distribution. 
Note that we assume that the learner may see the true protected attribute before querying the label for a point---see \cite{awasthi2020equalized} for a discussion of the case when the protected attribute is noisy.\loose

\section{FAIR ACTIVE LEARNING}
In this section, we present our approach to fair active classification, \ours.

\subsection{Fair Learning with Fixed Datasets}\label{sec:passive}
% \awcomment{I feel like we shouldn't call this section fair passive learning but instead should frame it just as how to learn given a fixed dataset, and then make clear that this is a subroutine of our alg.}
Before considering the active setting, we first consider the question of finding a fair classifier on a fixed dataset.
As the general classification paradigm (i.e. classification without fairness constraints) is known to potentially cause disparities when applied to sensitive tasks \citep{barocas2016big}, significant effort has been invested to develop effective algorithms that balance the goal of classification (learn the most accurate classifier) with fairness (learn a classifier with low fairness violation) on static datasets. %This corresponds to the problem of finding the classifier with lowest risk among the classifiers with fairness violation that does not exceed a given threshold \cite{agarwal2018reductions}. 
Given a target distribution $\nu$, a fairness metric denoted $m_{\rm fair}\! \in \!\{EO, TP\}$ and a fairness violation tolerance $\alpha \in [0, 1]$, this fair classification problem can be stated as the following:\loose
\begin{equation}
\begin{aligned}\label{eq:opti}
& \underset{h\in\H}{\text{minimize}}
& & R_{\nu}(h) 
% \\
% & 
\quad
\text{subject to}
& & L^{m_{\rm{fair}}}_{\nu}(h) \leq \alpha.
\end{aligned}
\end{equation}
%\vspace{-.4cm}
%\subsection{Estimating fairness violations}\label{sec:estimation}
In practice, one cannot solve \eqref{eq:opti} directly, as the population, $\nu$, which $R_{\nu}(h)$ and $L_\nu^{m_{\rm{fair}}}(h)$ depend on, is unknown.
% \awcomment{a general comment: in the future I would macro all complicated definitions like $L^{m_{\rm{fair}}}$. it makes it easier to write the paper and then easier to quickly change if you want to use different notation} 
Instead, we consider empirical estimates of the risk and fairness constraint. As is standard throughout machine learning, we rely on the plug-in estimate of the empirical risk, $\widehat{R}_{\mc{D}}(h) := \frac{1}{n}\sum_{i=1}^n\1\{h(x_i)\neq y_i\}$.
Similarly, throughout the fairness literature, a plug-in estimator is typically also used to estimate the fairness violation \citep{agarwal2018reductions, donini2018empirical, cotter2018training}.
%We motivate here how to establish the plug-in estimate of 
As an example, consider the case of estimating TPRP. Let $\mc{D} = \{(x_1, a_1, y_1), \ldots,(x_n, a_n, y_n)\}$ denote a set of data and recall that the True Positive Rate (TPR) of each group $z\in\{0, 1\}$ can be written as \loose
\iftoggle{arxiv}{
\begin{align}\label{eq:tpr_def}
&P_{(x, a, y)\sim\nu}(h(x)=1|a=z, y=1) = \frac{\E_{(x, a, y)\sim\nu}[\1\{h(x)=1,y=1,a=z\}]}{\E_{(x, a, y)\sim\nu}[\1\{y=1,a=z\}]}.
\end{align}}{
\begin{align}\label{eq:tpr_def}
&P_{(x, a, y)\sim\nu}(h(x)=1|a=z, y=1)\nonumber \\
&\qquad= \frac{\E_{(x, a, y)\sim\nu}[\1\{h(x)=1,y=1,a=z\}]}{\E_{(x, a, y)\sim\nu}[\1\{y=1,a=z\}]}.
\end{align}}

A natural approach to empirically estimate the TPRP is then to simply replace the population quantities with the empirical quantities in \eqref{eq:tpr_def} to estimate the TPR for each group, and then compute the absolute value of the difference of these TPRs.
%use $\frac{1}{n}\sum_{i=1}^{n}\1\{h(x_i)=z\}\cdot\1\{y_i=z\}\cdot\1\{a_i=1\}$ for the numerator and $\frac{1}{n}\sum_{i=1}^{n}\1\{y_i=z\}\cdot\1\{a_i=1\}$ for the denominator. This enables to build the plug-in estimate of True Positive Rate Parity, which we define next.
This yields the following empirical estimate of the TPRP violation of a classifier $h$ on the data $\mc{D}$:
\iftoggle{arxiv}{
\begin{align}\label{eq:tp_est}
\begin{split}
    \widehat{L}^{\rm TP}_\mc{D}(h) &:= \Bigg | \sum_{i=1}^{n}\frac{\1\{h(x_i)=1,y_i=1,a_i=1\}}{\sum_{i=1}^{n}\1\{y_i=1,a_i=1\}}  - \sum_{i=1}^{n}\frac{\1\{h(x_i)=1,y_i=1,a_i=0\}}{\sum_{i=1}^{n}\1\{y_i=1,a_i=0\}} \Bigg |.
\end{split}
\end{align}}{
\begin{align}\label{eq:tp_est}
\begin{split}
    \widehat{L}^{\rm TP}_\mc{D}(h) &:= \Bigg | \sum_{i=1}^{n}\frac{\1\{h(x_i)=1,y_i=1,a_i=1\}}{\sum_{i=1}^{n}\1\{y_i=1,a_i=1\}} \\
    &\qquad - \sum_{i=1}^{n}\frac{\1\{h(x_i)=1,y_i=1,a_i=0\}}{\sum_{i=1}^{n}\1\{y_i=1,a_i=0\}} \Bigg |.
\end{split}
\end{align}
}
%\vspace{-.3cm}
We can estimate the \emph{false-positive rate parity} (FPRP), $\widehat{L}^{\rm FP}_\mc{D}(h)$, analogously to \eqref{eq:tp_est} but with $y_i = 1$ replaced by $y_i = 0$, and estimate the EO violation as the maximum of the empirical estimate of the TPRP violation and the empirical estimate of the FPRP violation, $\widehat{L}^{\rm EO}_\mc{D}(h) = \max\{\widehat{L}^{\rm TP}_\mc{D}(h), \widehat{L}^{\rm FP}_\mc{D}(h)\}$.

%Note that one can analogously define the empirical estimate of the FPRP violation, $\widehat{L}^{\rm FP}_\mc{D}(h)$, by conditioning on $\1\{y_i=0\}$ (instead of $\1\{y_i=1\}$ for TPRP) and define the empirical estimate of the EO violation as the maximum of empirical estimate of the TPRP violation and the empirical estimate of the FPRP violation, $\widehat{L}^{\rm EO}_\mc{D}(h) = \max\{\widehat{L}^{\rm TP}_\mc{D}(h), \widehat{L}^{\rm FP}_\mc{D}(h)\}$. 
\paragraph{Empirical fair classification.}
Equipped with these empirical estimates, we return to the fair classification problem, \eqref{eq:opti}. 
Given a training set $\mc{D} = \{(x_1, a_1, y_1), \ldots,(x_n, a_n, y_n)\} \sim \nu$ sampled from a distribution $\nu \in \triangle_{\X\times \{0, 1\}\times \{0, 1\}}$, a fairness metric denoted $m_{\rm fair} \in \{EO, TP\}$ and fairness tolerance $\alpha \in [0, 1]$, one can use the empirical estimates of the risk and the fairness violation to approximate \eqref{eq:opti} with the following \emph{empirical} fair classification optimization problem: \loose
%\vspace{-.1cm}
\begin{equation}
\begin{aligned}\label{eq:empirical_opti}
\textstyle & \underset{h\in\H}{\text{minimize}}
& & \widehat{R}_\mc{D}(h)
% \\
% & 
\quad
\text{subject to}
& & \widehat{L}^{m_{\rm{fair}}}_\mc{D}(h) \leq \alpha.
\end{aligned}
\end{equation}
%\vspace{-.3cm}
%where we define the well-known empirical risk $\widehat{R}_D(h) := \frac{1}{n}\sum_{i=1}^n\1\{h(x_i)\neq y_i\}$. 
Note that solving such a problem is a common approach to fair classification, and can be solved efficiently \citep{donini2018empirical, agarwal2018reductions}. This optimization problem will form the starting-point of our proposed approach, and our algorithms will assume access to a solver for it,  which we call the empirical fair oracle---\texttt{EFO}. In our experiments we take an approach analogous to \citet{agarwal2018reductions} to solve \eqref{eq:empirical_opti}. %\romain{Add pseudo-code for this?} 

\subsection{Estimation Error and Sampling Bias}
In this section we address two additional issues that arise in ensuring our returned classifier is fair. First, estimation error in the fairness constraint, and second, bias introduced by sampling data points in a non-uniform fashion.

\paragraph{Conservative fairness estimates.}
Since $\widehat{L}^{m_{\rm{fair}}}_\mc{D}(h)$ is only an empirical estimate of $L^{m_{\rm{fair}}}_{\nu}(h)$, ensuring that $\widehat{L}^{m_{\rm{fair}}}_\mc{D}(h) \le \alpha$ does not guarantee that $L^{m_{\rm{fair}}}_{\nu}(h) \le \alpha$, our end goal. 
The following result gives a precise quantification of the deviation between $\widehat{L}^{m_{\rm{fair}}}_\mc{D}(h)$ and $L^{m_{\rm{fair}}}_{\nu}(h)$ in the case where $m_{\rm{fair}} = EO$.

\begin{proposition}\label{cor:fairness_est_simple}
    Let the train set be  $\mc{D} = \{(x_1, a_1, y_1), \ldots,(x_n, a_n, y_n)\} \sim \nu$. Then it holds with probability $1-\delta$ that, with $c_\delta := 8 \sqrt{2\log(2/\delta)}$:
\iftoggle{arxiv}{
    \begin{align*}
        & |L^{\rm EO}_\nu(h) - \widehat{L}^{\rm EO}_\mc{D}(h)|  \le \frac{c_\delta}{\sqrt{n}}  \cdot \max_{0\leq j, k \leq 1}\frac{1}{\frac{1}{n}\sum_{i=1}^n \1\{y_i = k, a_i = j \}}+\mc{O}\left(\frac{1}{n}\right).
    \end{align*}
}{
    \begin{align*}
        & |L^{\rm EO}_\nu(h) - \widehat{L}^{\rm EO}_\mc{D}(h)| \\
        & \quad \le \frac{c_\delta}{\sqrt{n}}  \cdot \max_{0\leq j, k \leq 1}\frac{1}{\frac{1}{n}\sum_{i=1}^n \1\{y_i = k, a_i = j \}}+\mc{O}\left(\frac{1}{n}\right).
    \end{align*}
    }
\end{proposition}
Analogous results hold for TPRP. This bound inspires two important aspects of our approach. First, to ensure fairness is met, it suggests setting the tolerance in \eqref{eq:empirical_opti} to a conservative value less than $\alpha$, in particular subtracting a $\mathcal{O}(\frac{1}{\sqrt{n}})$ term off of $\alpha$. Adjusting $\alpha$ by this margin has been demonstrated in the past to produce fair classifiers \citep{woodworth2017learning,thomas2019preventing}, and we show in \Cref{fig:correction_ablation} that it is also critical in our active setting. Second, \Cref{cor:fairness_est_simple} suggests that in order to estimate the fairness, we need to collect samples for \emph{each protected attribute}, since our estimation error scales inversely with the minimum number of samples collected for either protected attribute. This observation is critical in motivating our active sampling procedure, as we outline in the following section.

% , to ensure that the fairness is met, a common approach is to set the tolerance in \eqref{eq:empirical_opti} to some value less than $\alpha$. In particular, a widely employed heuristic is to replace the constraint in \eqref{eq:empirical_opti} with $\widehat{L}^{m_{\rm{fair}}}_\mc{D}(h) \leq \alpha - \frac{1}{\sqrt{n}}$, where $n$ is the size of the training dataset $\mc{D}$, which ensures that the constraint is satisfied with a margin that is on the order of the estimation error $| L^{m_{\rm{fair}}}_{\nu}(h) - \widehat{L}^{m_{\rm{fair}}}_\mc{D}(h)|$ \citep{woodworth2017learning,thomas2019preventing}. We illustrate the importance of the $\frac{1}{\sqrt{n}}$-correction term in \Cref{fig:correction_ablation}, where we compare our methods with and without this heuristic, and show that omitting this correction causes the returned classifier to be unfair.

\paragraph{Sampling bias correction.} 
%Before describing our proposed algorithm, we provide additional details on the estimators used. 
In the active learning paradigm, at every step the learner samples a data point $(x_t,a_t) \in \cDmy$ from some (chosen) distribution, $\nu_t^{\rm tr} \in \triangle_{\cDmy}$, $(x_t,a_t) \sim \nu_t^{\rm tr}$. For example, the learner may place higher weight on points that are \emph{informative}, increasing the number of samples from around the decision boundary. While this will ultimately improve the learner's ability to classify, the distribution of the sampled dataset no longer matches that of the original training dataset. This will result in the plug-in estimator for the fairness constraint, for example \eqref{eq:tp_est}, to be biased. We correct for this mismatch using importance weights.
For the risk, we recall the definition of the well-known IPS estimator (empirical risk re-weighted with importance weights): $\widehat{R}_{\mc{D}, \nu^{\rm tr}, \nu}(h) := \frac{1}{n}\sum_{i=1}^n\frac{\nu_i}{\nu^{\rm tr}_i}\1\{h(x_i)\neq y_i\}$, for $(x_i,a_i) \sim \nu^{\mathrm{tr}}$ and $y_i$ and associated label, and $\nu_i$ the population weight of point $i$\footnote{In general this is unknown but, assuming $\cDmy \sim \nu$, it suffices to simply set $\nu_i = 1/|\cDmy|$} and $\nu^{\mathrm{tr}}_i$ the probability $\nu^{\mathrm{tr}}$ samples point $i$. 
It is straightforward to see that this is an unbiased estimator of the true risk.
We define the estimator for EO with importance weights next.

% not be the right estimate of their corresponding true fairness metrics. 
% \awcomment{this is sort of confusing. it's not obvious at this point that the above is even the right estimate when the distributions are the same (in particular, the ratio of the empirical quantities would probably jump out to people). furthermore, as we discuss in a bit, it's actually not the right estimate, because it's biased} 
% Importantly, this is the case in active learning, where the agent actively choose the training distribution. We thus need to use importance weights in order to obtain appropriate estimates. 
% \awcomment{I think more background/explanation on why we need importance weights here is needed. Why, when you are doing active learning, are the distributions not the same? This is specific to the AL method to some degree}

\begin{definition}[Empirical EO violation with importance weights]\label{def:ips_fairness}
Consider a dataset drawn i.i.d from $\nu^{\rm tr}$,  $\mc{D} := \{(x_1, a_1, y_1), \ldots,(x_n, a_n, y_n)\} \sim \nu^{\rm tr}$. The empirical estimate of the EO violation of a classifier $h$ on the target distribution $\nu$ can be evaluated as 
\vspace{-.1cm}
% \awcomment{elsewhere you state that we don't know the distributions in practice, but here you are using them in the estimator. this should be clarified}
\iftoggle{arxiv}{
\begin{align*}
    \!\widehat{L}^{\rm EO}_{\mc{D}, \nu^{\rm tr}, \nu}\!(h)\! &:=\!\!\! \!\max_{z\in\{0, 1\}}\!\Bigg| \frac{\sum_{i=1}^{n}\!\frac{\nu_i}{\nu^{\rm tr}_i}\!\1\{h(x_i)\!=\!1,y_i\!=\!z,a_i\!=\!1\!\}}{\sum_{i=1}^{n}\frac{\nu_i}{\nu^{\rm tr}_i}\1\{y_i=z,a_i=1\}} -\!\frac{\sum_{i=1}^{n}\!\frac{\nu_i}{\nu^{\rm tr}_i}\!\1\{h(x_i)\!=\!1,y_i\!=\!z,a_i\!=\!0\}}{\sum_{i=1}^{n}\frac{\nu_i}{\nu^{\rm tr}_i}\1\{y_i=z,a_i=0\}}\Bigg|.
\end{align*}
}{
\begin{align*}
    \!\widehat{L}^{\rm EO}_{\mc{D}, \nu^{\rm tr}, \nu}\!(h)\! &:=\!\!\! \!\max_{z\in\{0, 1\}}\!\Bigg| \frac{\sum_{i=1}^{n}\!\frac{\nu_i}{\nu^{\rm tr}_i}\!\1\{h(x_i)\!=\!1,y_i\!=\!z,a_i\!=\!1\!\}}{\sum_{i=1}^{n}\frac{\nu_i}{\nu^{\rm tr}_i}\1\{y_i=z,a_i=1\}}
    \\
    % &\qquad\qquad\qquad\qquad\qquad\qquad\qquad
    &\qquad\!-\!\frac{\sum_{i=1}^{n}\!\frac{\nu_i}{\nu^{\rm tr}_i}\!\1\{h(x_i)\!=\!1,y_i\!=\!z,a_i\!=\!0\}}{\sum_{i=1}^{n}\frac{\nu_i}{\nu^{\rm tr}_i}\1\{y_i=z,a_i=0\}}\Bigg|.
\end{align*}
}
%where we shorten $\nu_i := \nu(x_i, a_i, y_i)$. 
\end{definition}
We define the importance-weighted TPRP violation analogously, but only for $z = 1$. While this estimate is not truly unbiased, both the numerator and denominators are unbiased, leading to accurate estimates of the fairness.
In the following, when applying our fairness oracle \texttt{EFO} in the active setting, we assume it is applied on the importance-weighted fairness and loss estimates.

\begin{figure}
    \hspace{0.4in}
    %\centering
    % \includegraphics[width=\textwidth]{new_final_plots/2d_lambdas.png}
    \begin{subfigure}{0.5\textwidth}
    \includegraphics[width=\textwidth]{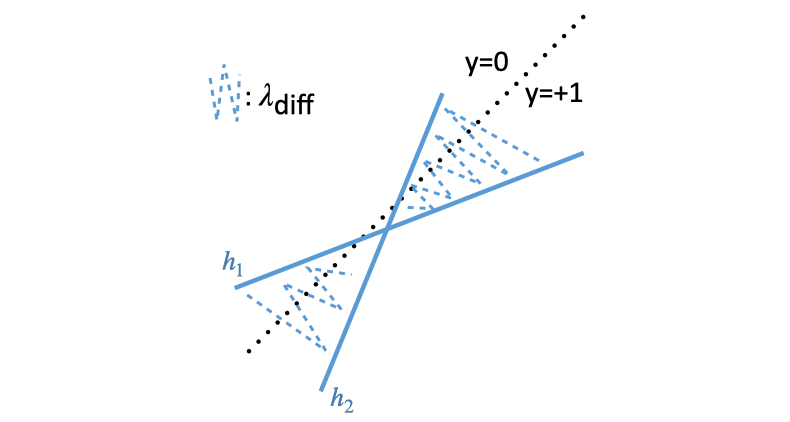}
    \caption{$\lamdiff$}
    \label{fig:2d_lambdaxy_xy}
    \end{subfigure}
    % \hfill
    \begin{subfigure}{0.25\textwidth}
    \includegraphics[width=\textwidth]{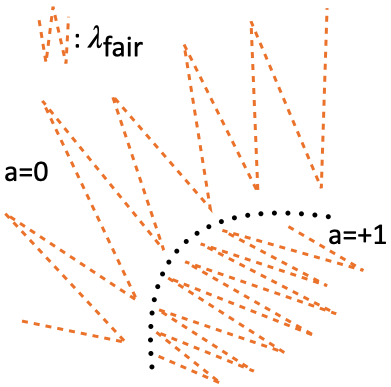}
    \caption{$\lamfair$}
    \label{fig:2d_lambdaxy_g}
    \end{subfigure}
    \caption{Sampling distributions of \ours when $k = 2$. $\lamdiff$ places mass on disagreement region of learned classifiers in order to collect points increasing accuracy. $\lamfair$ places equal amounts of mass on each group in order to learn fairness value.}
    \label{fig:2d_lambdaxy}
\end{figure}

\subsection{Fair Active Learning}
% \todor{Cite \cite{camilleri2022active} somewhere}
We now provide our algorithm for fair active classification, Algorithm~\ref{alg:active_cc}.
% guarantee that the output classifier is fair with high probability. \awcomment{we are not actually guaranteeing it's fair whp unless we use the exact bound predicted by theory---should clarify}
Algorithm~\ref{alg:active_cc} proceeds in rounds.
In each round, we choose data points to label by sampling from two distributions: $\lamdiff$, which focuses on improving the \emph{accuracy}, and $\lamfair$, which focuses on improving the \emph{fairness estimates}. We describe our choice of each of these distributions below.

\begin{algorithm}[h!]
\caption{\ours (Fair Active Randomized Exploration)} % \romain{should we make it more obvious that we use \Cref{def:ips_fairness} in the \texttt{EFO}?}}
\label{alg:active_cc}
\begin{algorithmic}[1]
\Require{Batch size $n$, number of rounds $L$, classifiers per round $k$, perturbation rate $\sigma$, fairness metric $m_{\rm{fair}}$, fairness tolerance $\alpha$, unlabeled data $\cDmy$}
\State Sample $(x^{(0)}_1, a^{(0)}_1), \ldots, (x^{(0)}_n, a^{(0)}_n) \sim \unif(\cDmy)$, request labels for sampled points
\State $\cD_0 \leftarrow \{ (x^{(0)}_i, a^{(0)}_i, y^{(0)}_i) \}_{i=1}^n$ 
\State $\cDmy \leftarrow \cDmy \backslash \{ (x^{(0)}_i, a^{(0)}_i) \}_{i=1}^n $
\For{$\ell = 1, \ldots, L-1$}
\Statex {\color{blue} \texttt{// Compute $\lamdiff$}}
\For{$i = 1, \ldots, k$}
\State $h_i = \texttt{EFO}(\widetilde{\cD}_{\ell-1}, \alpha - \frac{1}{\sqrt{n\cdot\ell}})$ where $\widetilde{\cD}_{\ell-1}$ generated by flipping each label of $\cD_{\ell-1}$ w.p. $\sigma$
\EndFor
\State{Compute $\lamdiff$ allocation:
% $\lambda_{\text{fair}} = \arg\min_{\lambda\in\triangle_\X}\max_{1\leq i \leq k}\sum_{x\in \X} \frac{\1\{h_i(x)\neq 0\}}{\lambda_x}$ and 
$$\lamdiff \leftarrow \argmin_{\lambda\in\triangle_{\cDmy}}\max_{1\leq i\neq j\leq k}\sum_{(x,a)\in \cDmy}\frac{\1\{h_i(x)\neq h_j(x)\}}{\lambda_x}$$}
\Statex {\color{blue} \texttt{// Compute $\lamfair$}}
\iftoggle{arxiv}{
\State $\lamfair \leftarrow \frac{1}{2} \unif(\{ (x,a) \in \cDmy \ : \ a = 0\}) + \frac{1}{2} \unif(\{ (x,a) \in \cDmy \ : \ a = 1\})$
}{
\State $\lamfair \leftarrow \frac{1}{2} \unif(\{ (x,a) \in \cDmy \ : \ a = 0\})$ 
\Statex $\qquad \qquad \qquad + \frac{1}{2} \unif(\{ (x,a) \in \cDmy \ : \ a = 1\})$}
\Statex {\color{blue} \texttt{// Sample points and update classifier}}
\State{Sample $(x^{(\ell)}_i,a^{(\ell)}_i) \sim \frac{1}{2} \lambda_{\text{diff}}+ \frac{1}{2} \lambda_{\text{fair}}$, $i=1,\ldots,n$}
\State{Observe corresponding labels $y^{(\ell)}_1, \ldots, y^{(\ell)}_n$}
\State{$\cD^{\ell} \leftarrow \cD^{\ell - 1} \cup \{ (x^{(\ell)}_i,a^{(\ell)}_i,y_i^{(\ell)}) \}_{i=1}^n$}
\State $\cDmy \leftarrow \cDmy \backslash \{ (x^{(\ell)}_i,a^{(\ell)}_i) \}_{i=1}^n$
%\State{$h^{(\ell)} = \texttt{EFO}(\cD^{\ell}, \alpha - \frac{1}{\sqrt{n\cdot(\ell+1)}})$}
%\State{\textbf{Report} performance metrics of $h^{(\ell)}$ ($R_{\nu^{\rm te}}(h)$ and $L^{m_{\rm{fair}}}_{\nu^{\rm te}}(h)$)}
\EndFor
\State{\textbf{Return} $\widehat{h} = \texttt{EFO}(\cD^{L}, \alpha - \frac{1}{\sqrt{n\cdot L}})$}
%\State{\textbf{Return} $\texttt{EFO}(\{(x^{(\ell)}_t, y^{(\ell)}_t)\}_{1\leq t \leq n, 0\leq \ell \leq L})$}
\end{algorithmic} 
\end{algorithm}

% In each round of \Cref{alg:active_cc} we perform \textit{randomized exploration} by training a set of $k$ fair classifiers $\widehat{h}_i, i\in [k]$ on perturbations of the training data already collected. In particular, to generate these perturbations, while training each classifier $\widehat{h}_i$ we flip the label of each data point with probability $\sigma$.
% %flipping the labels on our existing dataset with some probability $\sigma$, and then training $k$ fair classifiers $\widehat{h}_i, i\in [k]$, on the resulting datasets. 
% Given these classifiers, we compute $\lambda_{\text{diff}}$---a distribution over available unlabeled training points that aims to distinguish between the $k$ classifiers---and $\lambda_{\text{fair}}$ a distribution over the points for which most the $k$ classifiers predict a positive label.
% % \awcomment{the positive label is only the case if we are doing TPRP, right? Otherwise we might want to sample negative labels? if so we should make that clear here and also in the algorithm statement}. 
% This process is repeated at every round of the interactive learning procedure. 

\paragraph{Improving accuracy via randomized exploration.}
In each round of \Cref{alg:active_cc}, to determine which points are most likely to improve accuracy, we perform \textit{randomized exploration} by training a set of $k$ fair classifiers $\widehat{h}_i, i\in [k]$, on perturbations of the training data already collected. In particular, to generate these perturbations, while training each classifier $\widehat{h}_i$ we flip the label of each data point with probability $\sigma$.
Given these classifiers, we compute $\lambda_{\text{diff}}$, which aims to sample unlabeled training points that effectively distinguish between the $k$ classifiers.

% \begin{table*}[t]
% \begin{center}
% \begin{tabular}{ |c|c|c| } 
% \hline
% Dataset & Protected Attribute & Number of Points \\
% \hline\hline
% \texttt{Drug Consumption} \citep{fehrman2017factor} & Gender & 1885 \\
% \hline
% \texttt{Bank} \citep{moro2014data} & Education Level & 11,162 \\
% \hline
% \texttt{German Credit} \citep{hofmann1994statlog} & Gender & 1,000 \\
% \hline
% \texttt{Adult Income} \citep{lichman2013UCI} & Gender & 48,842 \\
% \hline
% \texttt{Compas} \citep{lichman2013UCI} & Gender & 5,278 \\
% \hline
% \texttt{Community and Crime} \citep{redmond2002data} & Race & 1,902 \\
% \hline
% \end{tabular}
% \end{center}
% \caption{Benchmark datasets}
% \label{tab:datasets}
% \end{table*}

As described in a variety of works \citep{osband2016generalization,osband2018randomized,russo2019worst,osband2019deep,kveton2019randomized, camilleri2022active}, randomized exploration emulates sampling from a posterior distribution over the optimal classifier. The sampling distribution $\lambda_{\text{diff}}$ is such that the weights will be large for the points $x$ about which the $k$ classifiers disagree most. Indeed, taking $k=2$ for illustration, we have $\lambda_{\text{diff}} = \arg\min_{\lambda\in\triangle_\X}\sum_{x\in \X}\frac{\1\{h_1(x)\neq h_2(x)\}}{\lambda_x}$. If $h_1(x) = h_2(x)$ then $\frac{\1\{h_1(x)\neq h_2(x)\}}{\lambda_x} = 0$ for any $\lambda_x > 0$. In order to minimize $\sum_{x\in \X}\frac{\1\{h_1(x)\neq h_2(x)\}}{\lambda_x}$, one can set $\lambda_x$ to be very small at regions of $\X$ where $h_1=h_2$ and very large at regions of $\X$ where $h_1\neq h_2$. See \Cref{fig:2d_lambdaxy_xy} for an illustration of this.
Given this, if we can ensure $\widehat{h}_i, i\in [k]$ disagree on points close to the true decision boundary, then our sampling procedure will ensure that we sample such points, which will enable us to effectively learn an accurate classifier. %\awcomment{I wonder if it would be worthwhile spending more time on this? In case people are not familiar with uncertainty sampling and such. This might be too much work but if we could have a figure illustrating this and the uniform exploration---like some 2d plot of different points showing where we want to sample---that could be really helpful}
With this in mind, we hope to create $k$ classifiers that have a decision boundary close to the true decision boundary, yet this is precisely what will be created by posterior sampling, which our procedure mimics. As we will see in the experiments, this sampling strategy effectively collects labels that are informative, increasing accuracy of the learned classifier.

\paragraph{Improving fairness via attribute-dependent exploration.}
In addition to learning the decision boundary  to obtain a classifier with high accuracy, we must also learn the value of the fairness constraint to ensure our final classifier is fair. While $\lamdiff$ ensures that we sample points close to the decision boundary, it makes no guarantee that we sample points which allow us to accurately estimate our fairness constraint---our choice of $\lamfair$ ensures that we do sample enough to accurately estimate the fairness.

% \romain{TODO: explain $\lamfair$}

\iftoggle{arxiv}{
% \begin{wrapfo}
% \begin{table}
% \vspace{-1cm}
\begin{wraptable}{r}{0.5\textwidth}
\begin{center}
\begin{tabular}{ |m{13em}|m{4.1em}|m{2.75em}| } 
\hline
Dataset & Protected Attribute & Dataset Size \\
\hline\hline
\texttt{Drug Consumption} \citep{fehrman2017factor} & Gender & 1885 \\
\hline
\texttt{Bank} \citep{moro2014data} & Education Level & 11,162 \\
\hline
\texttt{German Credit} \citep{hofmann1994statlog} & Gender & 1,000 \\
\hline
\texttt{Adult Income} \citep{lichman2013UCI} & Gender & 48,842 \\
\hline
\texttt{Compas} \citep{lichman2013UCI} & Gender & 5,278 \\
\hline
\texttt{Community and Crime} \citep{redmond2002data} & Race & 1,902 \\
\hline
\end{tabular}
\end{center}
\caption{Benchmark datasets}
\label{tab:datasets}
\end{wraptable}
}{}
 As shown in \Cref{cor:fairness_est_simple}, if we wish to estimate the fairness value of a given classifier, we must ensure that we have collected sufficiently many data points from each group $j \in \{0,1\}$. $\lamdiff$ is not guaranteed to sample such points---for example, if we have severe group imbalance, the overall accuracy may be maximized by ignoring the group with many fewer samples, in which case $\lamdiff$ will focus on only sampling the larger group. To address this, we choose $\lamfair$ to sample an equal number of samples from each group, which will ensure that our fairness estimate will converge to the population fairness, as guaranteed by \Cref{cor:fairness_est_simple}. See \Cref{fig:2d_lambdaxy_g} for an illustration of this.
 As we demonstrate in \Cref{sec:ablation}, this sampling is absolutely critical if our goal is to learn a fair classifier---without this attribute-dependent sampling, naive active learning methods fail to produce fair classifiers.

\iftoggle{arxiv}{
% \begin{wrapfo}
% \begin{table}
% \vspace{-1cm}
% \begin{wraptable}{r}{0.5\textwidth}
% \begin{center}
% \begin{tabular}{ |m{13em}|m{4.1em}|m{2.75em}| } 
% \hline
% Dataset & Protected Attribute & Dataset Size \\
% \hline\hline
% \texttt{Drug Consumption} \citep{fehrman2017factor} & Gender & 1885 \\
% \hline
% \texttt{Bank} \citep{moro2014data} & Education Level & 11,162 \\
% \hline
% \texttt{German Credit} \citep{hofmann1994statlog} & Gender & 1,000 \\
% \hline
% \texttt{Adult Income} \citep{lichman2013UCI} & Gender & 48,842 \\
% \hline
% \texttt{Compas} \citep{lichman2013UCI} & Gender & 5,278 \\
% \hline
% \texttt{Community and Crime} \citep{redmond2002data} & Race & 1,902 \\
% \hline
% \end{tabular}
% \end{center}
% \caption{Benchmark datasets}
% \label{tab:datasets}
% \end{wraptable}
}{
\begin{table}
\begin{center}
\begin{tabular}{ |m{13em}|m{4.1em}|m{2.75em}| } 
\hline
Dataset & Protected Attribute & Dataset Size \\
\hline\hline
\texttt{Drug Consumption} \citep{fehrman2017factor} & Gender & 1885 \\
\hline
\texttt{Bank} \citep{moro2014data} & Education Level & 11,162 \\
\hline
\texttt{German Credit} \citep{hofmann1994statlog} & Gender & 1,000 \\
\hline
\texttt{Adult Income} \citep{lichman2013UCI} & Gender & 48,842 \\
\hline
\texttt{Compas} \citep{lichman2013UCI} & Gender & 5,278 \\
\hline
\texttt{Community and Crime} \citep{redmond2002data} & Race & 1,902 \\
\hline
\end{tabular}
\end{center}
\caption{Benchmark datasets}
\label{tab:datasets}
\end{table}
}
\section{EXPERIMENTS}
Finally, we demonstrate the effectiveness of \ours experimentally on standard fairness datasets.

% \subsection{Experimental Methodology}\label{sec:method}
% Before proceeding to our results, we briefly cover experimental details for our  work.

\paragraph{Implementation details.}
% \awcomment{describe how you chose hyperparameters like $\sigma$, etc, and also mention how we turn off half-half sampling}
% \paragraph{Computationally efficient solver of \eqref{eq:empirical_opti}.} 
% In order to compute a classifier that is empirically fair, we resort to the reduction algorithm from \citet{agarwal2018reductions}, as mentioned in \Cref{sec:passive}. We refer to it as the
% % \texttt{EMPIRICALLY FAIR ORACLE}
% Empirically Fair Oracle (\texttt{EFO}). 
For all experiments, we use logistic regression classifiers without regularization and partition the dataset into a $75\%/25\%$ train$/$test split. We ran a grid-search over the hyperparameters of \ours to set $\sigma=0.1$ and $k=10$. We set the fairness tolerance to $\alpha - 1/\sqrt{n}$ to account for estimation error in the fairness constraint.
All experiments were run on a Intel Xeon 6226R CPU with 64 cores.

% \kevin{Do we say we're using linear preditor class anywhere? Regularized?}

\vspace{-.3cm}

\paragraph{Datasets.} 
In our experiments, we consider six datasets commonly used in the fairness literature, listed in \Cref{tab:datasets}. To ensure consistency, we standardized the data to have a mean of zero and a variance of one.

%\romain{can transform the below into a table if space is needed}

\newcommand{\FAL}{\textsc{Fal}\xspace}
\newcommand{\FALCUR}{\textsc{Falcur}\xspace}
\newcommand{\Panda}{\textsc{Panda}\xspace}

\iftoggle{arxiv}{
\begin{figure*}
\begin{minipage}[c]{0.5\linewidth}
    \centering
    \includegraphics[width=\textwidth]{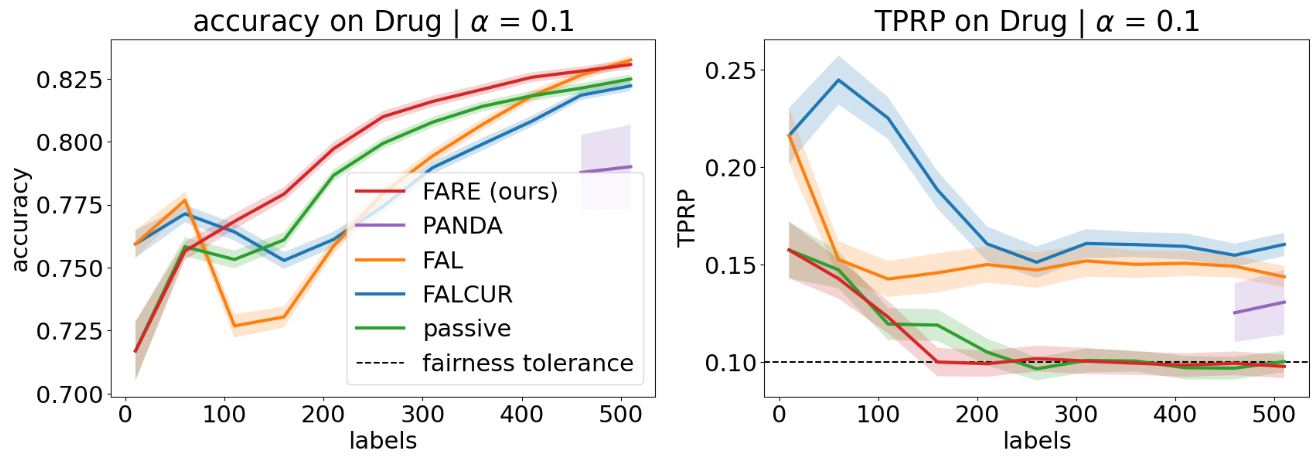}
    \caption{Performance on \texttt{Drug Consumption}}
    \label{fig:drug}
\end{minipage}
\hfill
\begin{minipage}[c]{0.5\linewidth}
    \centering
    \includegraphics[width=\textwidth]{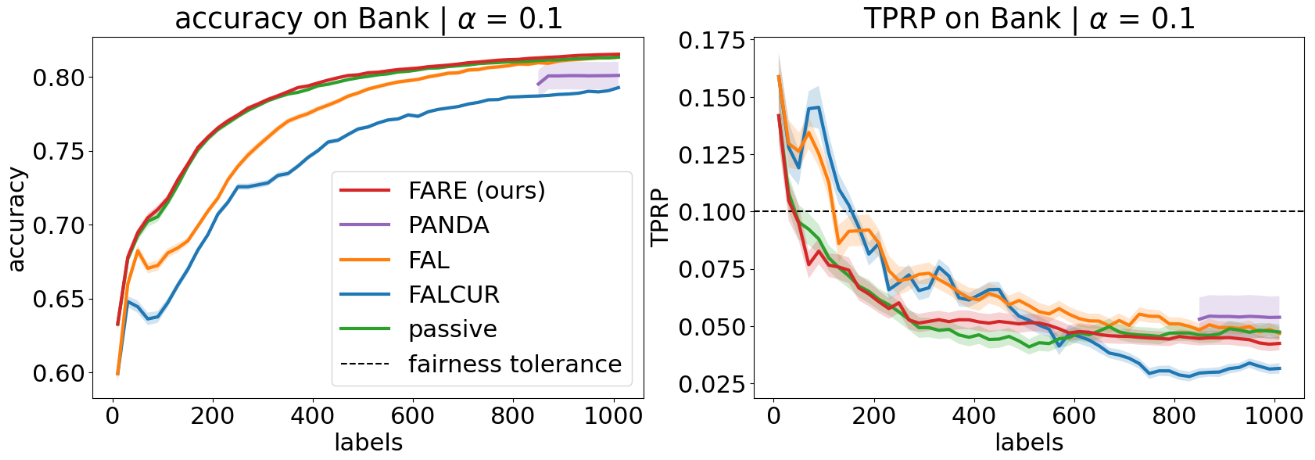}
    \caption{Performance on \texttt{Bank}}
    \label{fig:bank}
\end{minipage}
\\
\begin{minipage}[c]{0.5\linewidth}
    \centering
    \includegraphics[width=\textwidth]{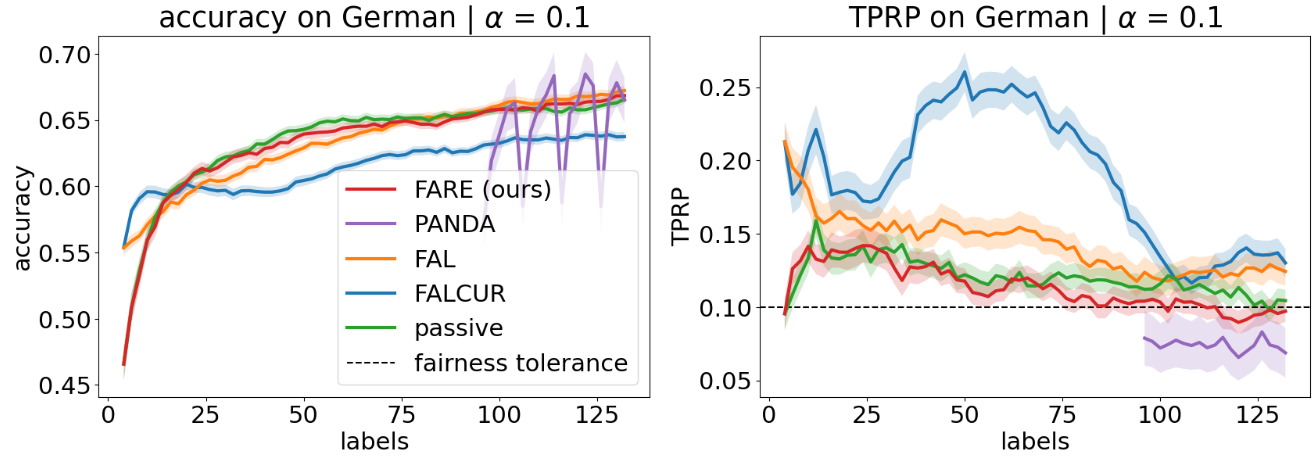}
    \caption{Performance on \texttt{German Credit}}
    \label{fig:german}
\end{minipage}
\hfill
\begin{minipage}[c]{0.5\linewidth}
    \centering
    \includegraphics[width=\textwidth]{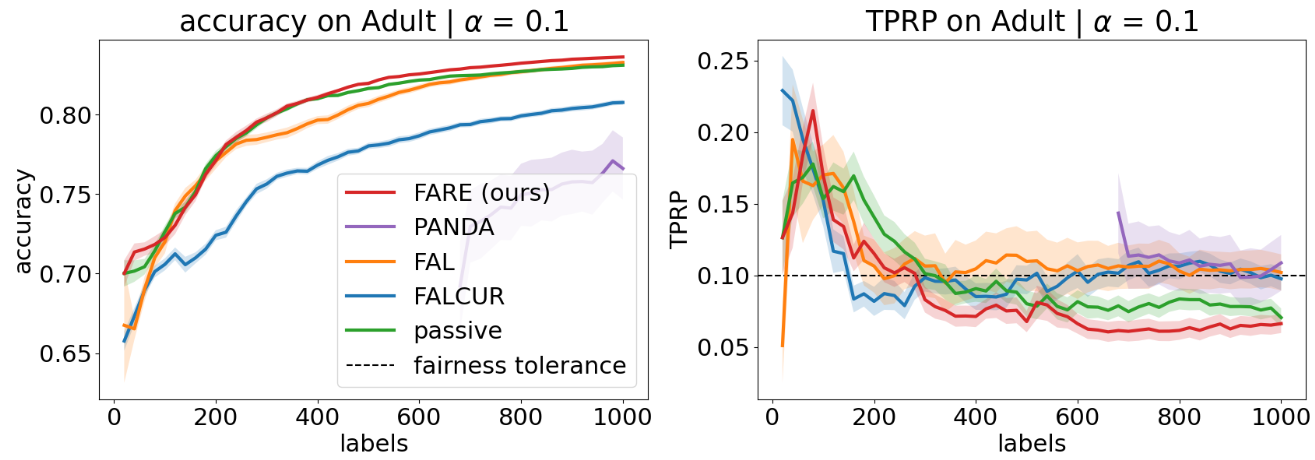}
    \caption{Performance on \texttt{Adult Income}}
    \label{fig:adult}
\end{minipage}
\\
\begin{minipage}[c]{0.5\linewidth}
    \centering
    \includegraphics[width=\textwidth]{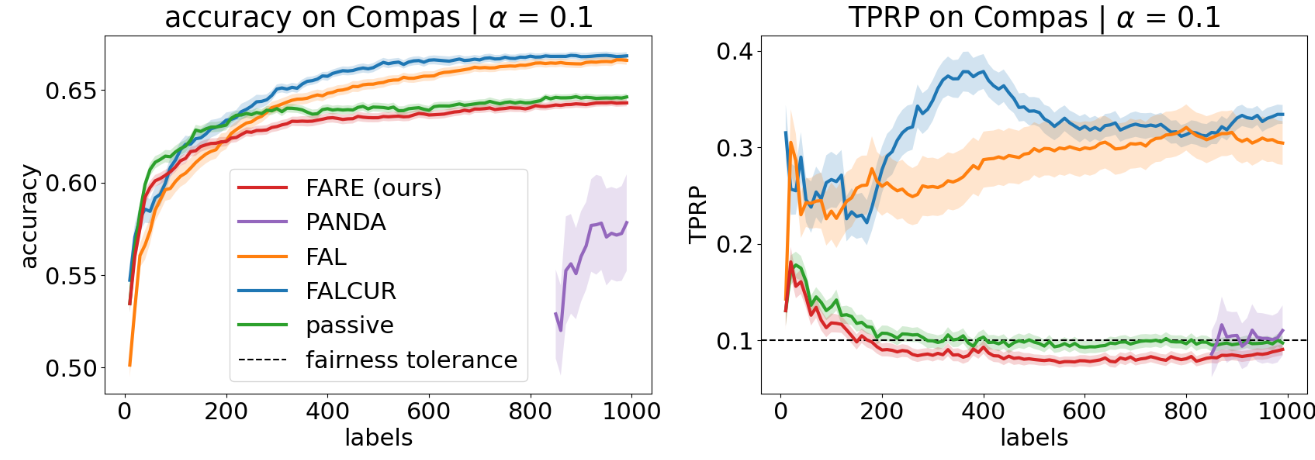}
    \caption{Performance on \texttt{Compas}}
    \label{fig:compas}
\end{minipage}
\hfill
%\begin{figure}[t]
\begin{minipage}[c]{0.5\linewidth}
    \centering
    \includegraphics[width=\textwidth]{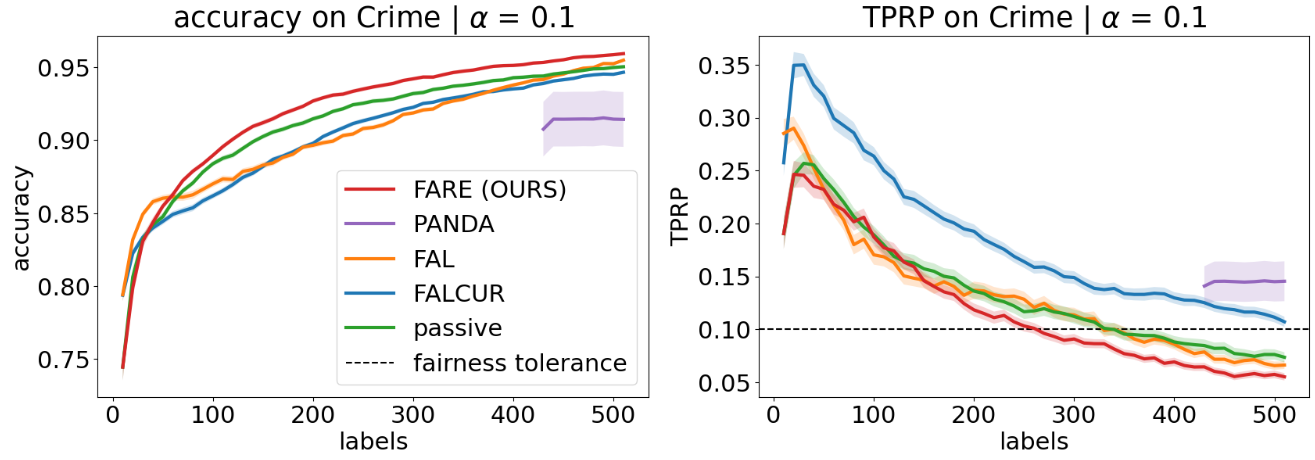}
    \caption{Performance on \texttt{Community and Crime}}
    \label{fig:crime}
\end{minipage}
\end{figure*}
}{
\begin{figure*}
\begin{minipage}[c]{0.45\linewidth}
    \centering
    \includegraphics[width=\textwidth]{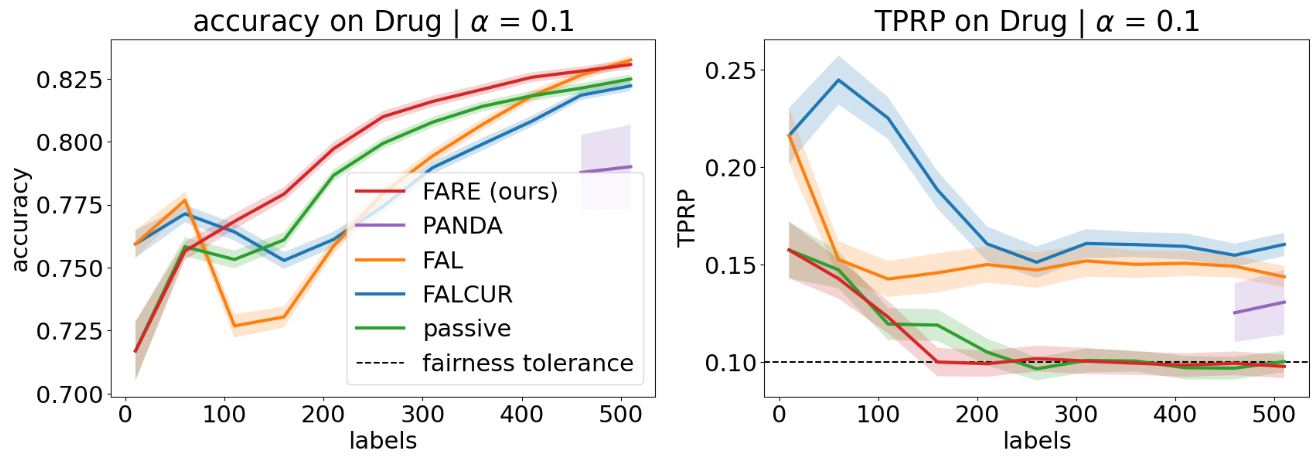}
    \caption{Performance on \texttt{Drug Consumption}}
    \label{fig:drug}
\end{minipage}
\hfill
\begin{minipage}[c]{0.45\linewidth}
    \centering
    \includegraphics[width=\textwidth]{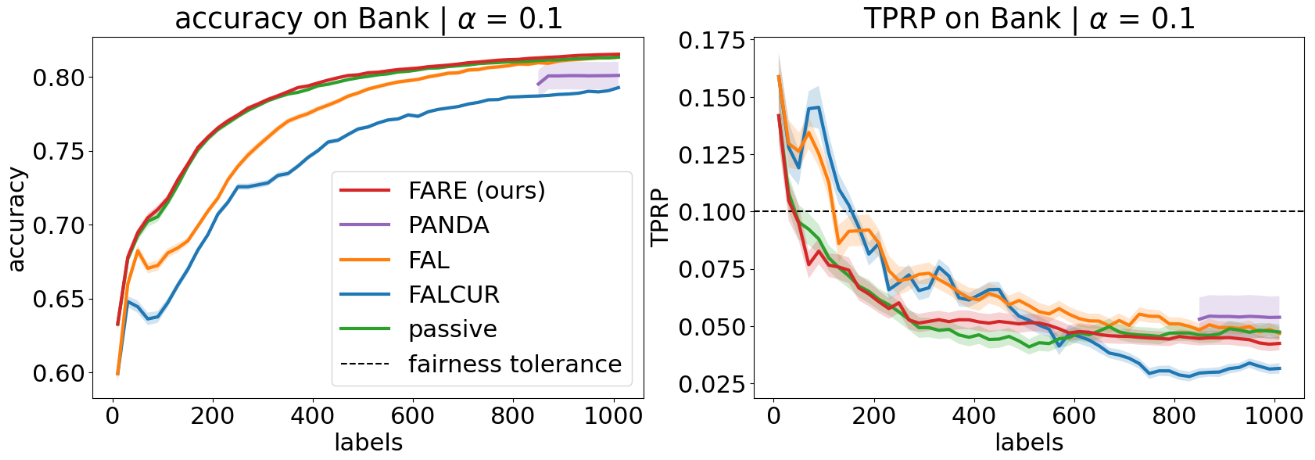}
    \caption{Performance on \texttt{Bank}}
    \label{fig:bank}
\end{minipage}
\\
\begin{minipage}[c]{0.45\linewidth}
    \centering
    \includegraphics[width=\textwidth]{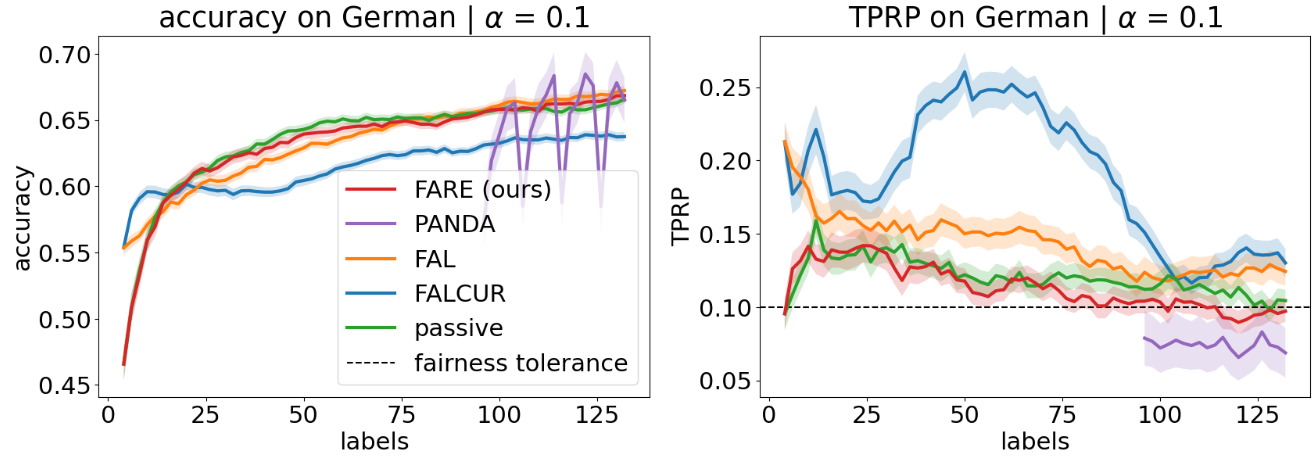}
    \caption{Performance on \texttt{German Credit}}
    \label{fig:german}
\end{minipage}
\hfill
\begin{minipage}[c]{0.45\linewidth}
    \centering
    \includegraphics[width=\textwidth]{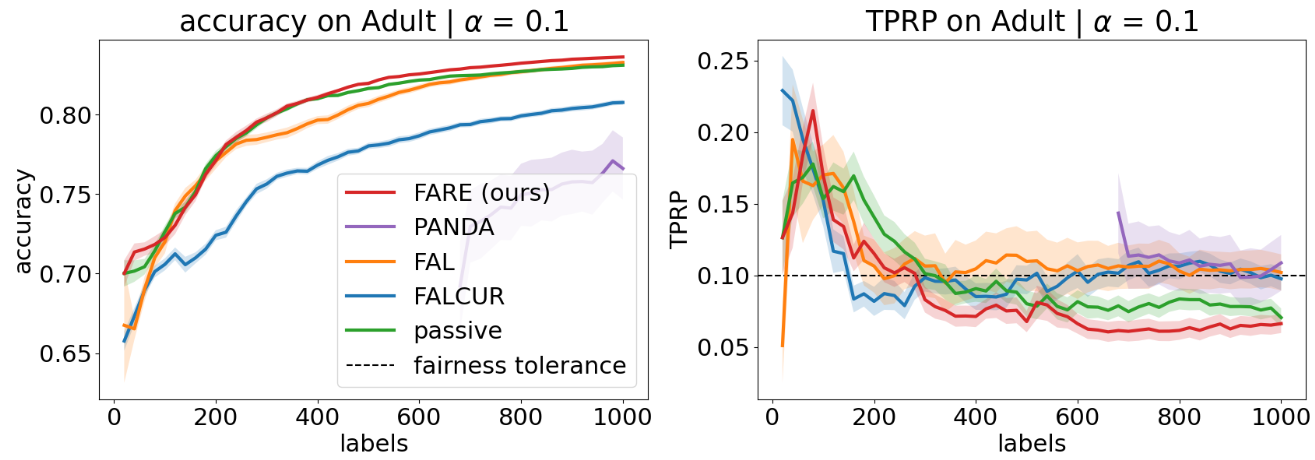}
    \caption{Performance on \texttt{Adult Income}}
    \label{fig:adult}
\end{minipage}
\\
\begin{minipage}[c]{0.45\linewidth}
    \centering
    \includegraphics[width=\textwidth]{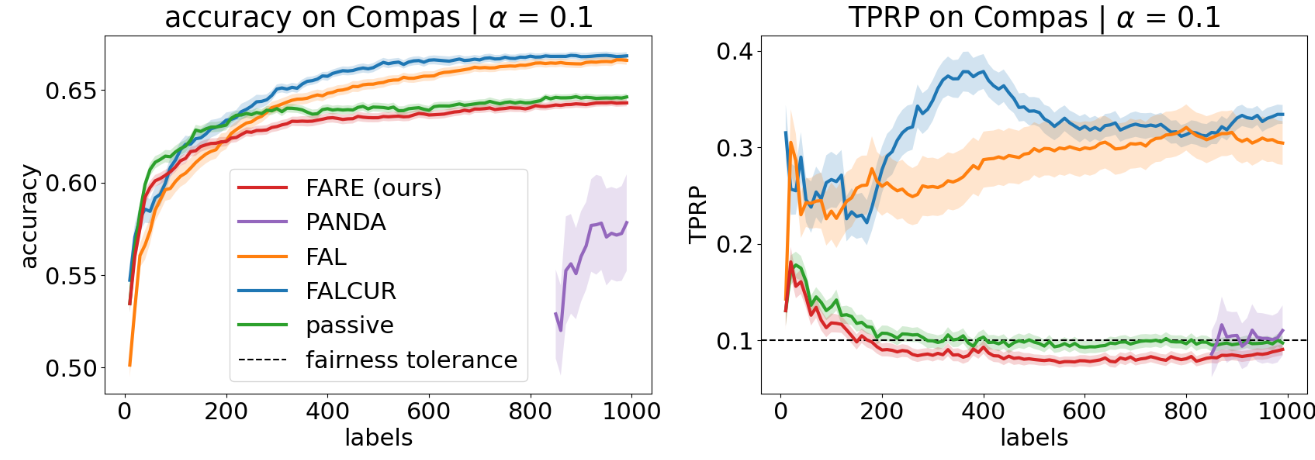}
    \caption{Performance on \texttt{Compas}}
    \label{fig:compas}
\end{minipage}
\hfill
%\begin{figure}[t]
\begin{minipage}[c]{0.45\linewidth}
    \centering
    \includegraphics[width=\textwidth]{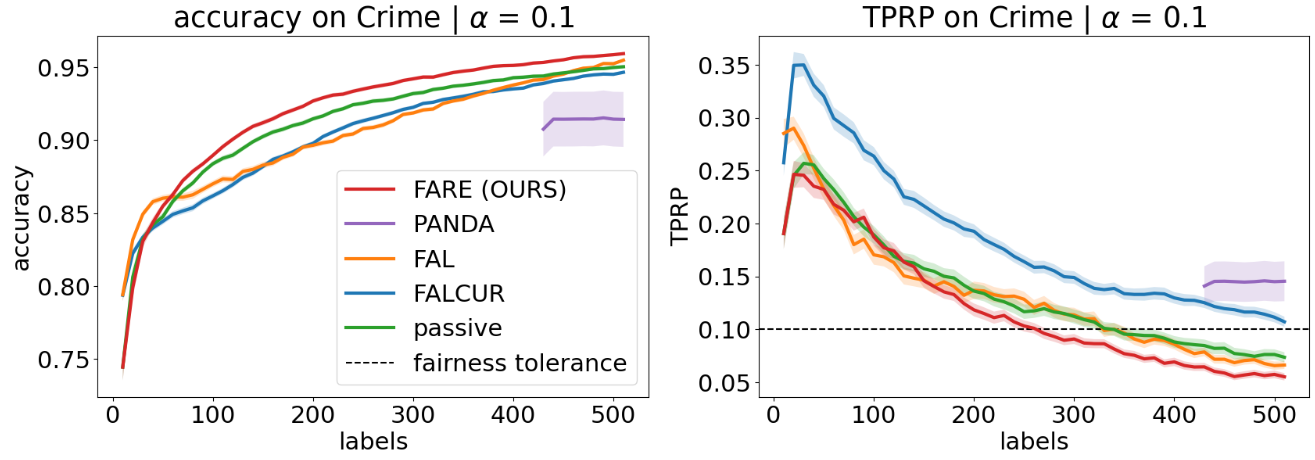}
    \caption{Performance on \texttt{Community and Crime}}
    \label{fig:crime}
\end{minipage}
\end{figure*}
\vspace{-.2cm}
}

\subsection{Baselines Methods} 
In order to benchmark \ours, we conduct experiments comparing it against state-of-the-art algorithms \citep{anahideh2021fair, sharaf2022promoting, fajri2022falcur} for fair active learning, and a passive baseline. 
\begin{enumerate}[leftmargin=*]
    \item \Panda \citep{sharaf2022promoting}: \Panda aims to learn a data selection policy via meta-learning. %approach to develop a policy that selects which labels to query in order to maximize classification accuracy while adhering to specified parity constraints. 
    This algorithm formulates the problem as a bi-level optimization task, where the inner level involves training a classifier with a subset of labeled data, while the outer level focuses on updating the selection policy to strike a balance between fairness and accuracy in the classifier's performance.
    \item \FAL \citep{anahideh2021fair}: \FAL uses a sampling rule that blends between two selection criteria: one based on uncertainty and another based on assessing fairness, which estimates the potential disparity impact when labeling a specific data point (by calculating the expected disparity across all potential labels). \FAL chooses which data points to label in order to strike a balance between model accuracy and equity.\loose
    \item \FALCUR \citep{fajri2022falcur}: \FALCUR incorporates an acquisition function that assesses the representative score of each sample under consideration. This score is calculated by taking into account two key factors: uncertainty and similarity. By carefully balancing these elements, \FALCUR selects samples that contribute to accuracy improvement and ensure that fairness is maintained.
    \item Passive + fair oracle: This passive baseline randomly selects points from the pool of examples $\cDmy$ and trains the model using the \texttt{EFO} oracle with the same $\alpha - \frac{1}{\sqrt{n}}$ constraint as \ours on its current samples. %, and reports the accuracy ($1-R_{\nu^{\rm te}}(.)$) and fairness violations.
\end{enumerate}

Each of these methods with the exception of the passive baseline assumes access to a pretraining dataset. As we are interested in the low-data regime, when we do not have access to a pretraining dataset, we simulate the pretraining dataset by allocating, for each method, some percentage of the label budget to uniform sampling to collect a ``pretrain'' dataset, and then run the algorithm in standard fashion from there. For each method, we sweep over the size of the pretrain dataset and plot performance for the best one. For all other hyperparameters, we use the values recommended by the original work.

\subsection{Performance Evaluation}

\newcommand{\blue}[1]{{\color{blue} #1}}
\newcommand{\red}[1]{{\color{red} #1}}
\iftoggle{arxiv}{
\begin{table*}
\begin{center}
\begin{tabular}{ |c||c|c|c|c|c||c|c|c|c|c| } 
  \hline
  & \multicolumn{5}{| c ||}{Accuracy (\% labeled correctly)} & \multicolumn{5}{| c |}{Fairness (TPRP, goal fairness = 0.1)} \\
  \cline{2-11}
 & \ours & \Panda & \FAL & \FALCUR & Passive & \ours & \Panda & \FAL & \FALCUR & Passive  \\ 
 \hline\hline
 \texttt{Drug} & \textbf{\blue{83.1}} & \red{79.0} & \red{83.2} & \red{82.2} & \blue{82.5} & \blue{0.098} & \red{0.131} & \red{0.144} & \red{0.160} & \blue{0.100} \\ 
& \footnotesize{\textbf{\blue{$\pm$ 0.2}}} & \footnotesize{\red{$\pm$ 2.1}} & \footnotesize{\red{$\pm$ 0.2}} & \footnotesize{\red{$\pm$ 0.2}} & \footnotesize{\blue{$\pm$ 0.2}} & \footnotesize{\blue{$\pm$ 0.006}} & \footnotesize{\red{$\pm$ 0.017}} & \footnotesize{\red{$\pm$ 0.0065}} & \footnotesize{\red{$\pm$ 0.006}} & \footnotesize{\blue{$\pm$ 0.005}} \\ 
  \hline
 \texttt{Bank} & \textbf{\blue{81.5}} & \blue{80.1} & \textbf{\blue{81.3}} & \blue{79.2} & \textbf{\blue{81.3}} & \blue{0.042} & \blue{0.054} & \blue{0.047} & \blue{0.032} & \blue{0.047} \\ 
 & \footnotesize{\textbf{\blue{$\pm$ 0.1}}} & \footnotesize{\blue{$\pm$ 0.4}} & \footnotesize{\textbf{\blue{$\pm$ 0.1}}} & \footnotesize{\blue{$\pm$ 0.1}} & \footnotesize{\textbf{\blue{$\pm$ 0.1}}} & \footnotesize{\blue{$\pm$ 0.003}} & \footnotesize{\blue{$\pm$ 0.009}} & \footnotesize{\blue{$\pm$ 0.003}} & \footnotesize{\blue{$\pm$ 0.002}} & \footnotesize{\blue{$\pm$ 0.001}} \\ 
  \hline
 \texttt{German} & \textbf{\blue{66.8}} & \textbf{\blue{66.4}} & \red{67.2} & \red{63.7} & \red{66.6} & \blue{0.097} & \blue{0.069} & \red{0.124} & \red{0.130} & \red{0.104} \\ 
& \footnotesize{\textbf{\blue{$\pm$ 0.3}}} & \footnotesize{\textbf{\blue{$\pm$ 1.4}}} & \footnotesize{\red{$\pm$ 0.4}} & \footnotesize{\red{$\pm$ 0.4}} & \footnotesize{\red{$\pm$ 0.3}} & \footnotesize{\blue{$\pm$ 0.007}} & \footnotesize{\blue{$\pm$ 0.016}} & \footnotesize{\red{$\pm$ 0.010}} & \footnotesize{\red{$\pm$ 0.010}} & \footnotesize{\red{$\pm$ 0.007}} \\ 
  \hline
 \texttt{Adult} & \textbf{\blue{83.6}} & \red{76.6} & \blue{83.2} & \blue{80.8} & \blue{83.1} & \blue{0.065} & \red{0.109} & \blue{0.102} & \blue{0.097} & \blue{0.068} \\ 
 & \footnotesize{\textbf{\blue{$\pm$ 0.0}}} & \footnotesize{\red{$\pm$ 2.0}} & \footnotesize{\blue{$\pm$ 0.1}} & \footnotesize{\blue{$\pm$  0.2}} & \footnotesize{\blue{$\pm$ 0.0}} & \footnotesize{\blue{$\pm$ 0.007}} & \footnotesize{\red{$\pm$ 0.019}} & \footnotesize{\blue{$\pm$ 0.013}} & \footnotesize{\blue{$\pm$ 0.008}} & \footnotesize{\blue{$\pm$ 0.006}} \\ 
  \hline
 \texttt{Compas} & \textbf{\blue{64.3}} & \red{57.8} & \red{66.6} & \red{66.8} & \textbf{\blue{64.6}} & \blue{0.088} &  \red{0.110} & \red{0.304} & \red{0.334} & \blue{0.099} \\ 
 & \footnotesize{\textbf{\blue{$\pm$ 0.1}}} & \footnotesize{\red{$\pm$ 1.3}} & \footnotesize{\red{$\pm$ 0.2}} & \footnotesize{\red{$\pm$ 0.2}} & \footnotesize{\textbf{\blue{$\pm$ 0.2}}} & \footnotesize{\blue{$\pm$ 0.006}} & \footnotesize{ \red{$\pm$ 0.026}} & \footnotesize{\red{$\pm$ 0.023}} & \footnotesize{\red{$\pm$ 0.009}} & \footnotesize{\blue{$\pm$ 0.007}} \\ 
  \hline
 \texttt{Crime} & \textbf{\blue{95.9}} & \red{91.4} & \blue{95.5} & \red{94.7} & \blue{95.0} & \blue{0.055} & \red{0.145}  & \blue{0.066} & \red{0.107} & \blue{0.074} \\ 
 & \footnotesize{\textbf{\blue{$\pm$ 0.1}}} & \footnotesize{\red{$\pm$ 1.7}} & \footnotesize{\blue{$\pm$ 0.1}} & \footnotesize{\red{$\pm$ 0.1}} & \footnotesize{\blue{$\pm$ 0.1}} & \footnotesize{\blue{$\pm$ 0.004}} & \footnotesize{\red{$\pm$ 0.019} } & \footnotesize{\blue{$\pm$ 0.004}} & \footnotesize{\red{$\pm$ 0.005}} & \blue{$\pm$ 0.005} \\
 \hline
\end{tabular}
\end{center}
\vspace{-0.8em}
\caption{Final accuracy and TPRP values for each method and dataset. \blue{Blue} indicates fairness threshold met, while \red{red} indicates threshold not met. Best accuracy among fair methods is indicated by \textbf{bold} font. Confidence intervals are standard errors based on $100$ trials.}
\label{tab:results}
\end{table*}
\begin{table*}
\begin{center}
\begin{tabular}{ |c||c|c|c|c|c||c|c|c|c|c| } 
  \hline
  & \multicolumn{5}{| c ||}{Accuracy (\% labeled correctly)} & \multicolumn{5}{| c |}{Fairness (TPRP, goal fairness = 0.1)} \\
  \cline{2-11}
 & \ours & \makecell{\ours \\ w/o $\lamfair$} & \FAL & \FALCUR & Passive & \ours & \makecell{\ours \\ w/o $\lamfair$} & \FAL & \FALCUR & Passive  \\ 
 \hline\hline
 \texttt{Synt.} & \textbf{\blue{58.8}} & \red{57.5} & \red{90.0} & \red{89.9} & \red{61.1} & \blue{0.095} & \red{0.123} & \red{0.402} & \red{0.303} & \red{0.123} \\ 
 & \footnotesize{\textbf{\blue{$\pm$  0.6}}} & \footnotesize{\red{$\pm$ 0.8}} & \footnotesize{\red{$\pm$ 1.7}} & \footnotesize{\red{$\pm$ 1.3}} & \footnotesize{\red{$\pm$ 0.8}} & \footnotesize{\blue{$\pm$ 0.009}} & \footnotesize{\red{$\pm$ 0.016}} & \footnotesize{\red{$\pm$ 0.022}} & \footnotesize{\red{$\pm$ 0.013}} & \footnotesize{\red{$\pm$ 0.013}} \\ 
 \hline
\end{tabular}
\end{center}
\vspace{-0.5em}
\caption{Ablation on the role of group-dependent sampling, $\lamfair$, on the synthetically generated dataset. Note that \Panda does not converge on this dataset, so we have omitted it from the table.  Confidence intervals are standard errors based on $100$ trials.}
\label{tab:results_ablation}
\vspace{-1em}
\end{table*}
}{
\begin{table*}
\begin{center}
\begin{tabular}{ |c||c|c|c|c|c||c|c|c|c|c| } 
  \hline
  & \multicolumn{5}{| c ||}{Accuracy (\% labeled correctly)} & \multicolumn{5}{| c |}{Fairness (TPRP, goal fairness = 0.1)} \\
  \cline{2-11}
 & \ours & \Panda & \FAL & \FALCUR & Passive & \ours & \Panda & \FAL & \FALCUR & Passive  \\ 
 \hline\hline
 \texttt{Drug} & \textbf{\blue{83.1}} & \red{79.0} & \red{83.2} & \red{82.2} & \blue{82.5} & \blue{0.098} & \red{0.131} & \red{0.144} & \red{0.160} & \blue{0.100} \\ 
 \texttt{Bank} & \textbf{\blue{81.5}} & \blue{80.1} & \textbf{\blue{81.3}} & \blue{79.2} & \textbf{\blue{81.3}} & \blue{0.042} & \blue{0.054} & \blue{0.047} & \blue{0.032} & \blue{0.047} \\ 
 \texttt{German} & \textbf{\blue{66.8}} & \blue{66.4} & \red{67.2} & \red{63.7} & \red{66.6} & \blue{0.097} & \blue{0.069} & \red{0.124} & \red{0.130} & \red{0.104} \\ 
 \texttt{Adult} & \textbf{\blue{83.6}} & \red{76.6} & \blue{83.2} & \blue{80.8} & \blue{83.1} & \blue{0.065} & \red{0.109} & \blue{0.102} & \blue{0.097} & \blue{0.068} \\ 
 \texttt{Compas} & \textbf{\blue{64.3}} & \red{57.8} & \red{66.6} & \red{66.8} & \textbf{\blue{64.6}} & \blue{0.088} &  \red{0.110} & \red{0.304} & \red{0.334} & \blue{0.099} \\ 
 \texttt{Crime} & \textbf{\blue{95.9}} & \red{91.4} & \blue{95.5} & \red{94.7} & \blue{95.0} & \blue{0.055} & \red{0.145}  & \blue{0.066} & \red{0.107} & \blue{0.074} \\ 
 \hline
\end{tabular}
\end{center}
\vspace{-0.8em}
\caption{Final accuracy and TPRP values for each method and dataset. \blue{Blue} indicates fairness threshold met, while \red{red} indicates threshold not met. Best accuracy among fair methods is indicated by \textbf{bold} font.}
\label{tab:results}
\end{table*}
\begin{table*}
\begin{center}
\begin{tabular}{ |c||c|c|c|c|c||c|c|c|c|c| } 
  \hline
  & \multicolumn{5}{| c ||}{Accuracy (\% labeled correctly)} & \multicolumn{5}{| c |}{Fairness (TPRP, goal fairness = 0.1)} \\
  \cline{2-11}
 & \ours & \makecell{\ours \\ w/o $\lamfair$} & \FAL & \FALCUR & Passive & \ours & \makecell{\ours \\ w/o $\lamfair$} & \FAL & \FALCUR & Passive  \\ 
 \hline\hline
 \texttt{Synthetic} & \textbf{\blue{58.8}} & \red{57.5} & \red{90.0} & \red{89.9} & \red{61.1} & \blue{0.095} & \red{.123} & \red{0.402} & \red{0.303} & \red{0.123} \\ 
 \hline
\end{tabular}
\end{center}
\vspace{-0.5em}
\caption{Ablation on the role of group-dependent sampling, $\lamfair$, on synthetically generated dataset. Note that \Panda does not converge on this dataset, so we have omitted it from the table.}
\label{tab:results_ablation}
\vspace{-1em}
\end{table*}
}

We first consider the case when the fairness constraint is TPRP with $\alpha = 0.1$, and illustrate the accuracy  and fairness vs. number of samples for our method and all baselines. For all methods and datasets, with the exception of \Panda, results are averaged over 100 trials---for \Panda results are averaged over only 50 trials, due to its large computational cost. Shaded regions denote one standard error. Note as well that the performance of \Panda starts at a later step since this method requires a large pretrain dataset to perform effectively, and in pretraining does not produce a classifier.\loose

Our results are given in \Cref{fig:drug,fig:bank,fig:german,fig:adult,fig:compas,fig:crime}, and we state the accuracy and fairness values obtained at the final step in \Cref{tab:results}. As these results illustrate, \ours consistently outperforms or matches the passive baseline, as well as all existing approaches to fair active classification. 
We highlight several key features of these results.

First, note that the only methods able to consistently produce classifiers which meet the fairness constraint of $\alpha = 0.1$ are \ours and the passive baselines. While all other methods frequently return classifiers that are unfair, both \ours and the passive baseline return classifiers that, by the final step, are fair on each dataset. We observe that, for very small number of labels, even \ours and the passive baselines produce classifiers which do not meet the fairness constraint---this is to be expected since, for a very small number of samples, it is difficult to estimate the fairness accurately enough to return a fair classifier. We emphasize that, though in some cases the accuracy of \ours is exceeded by baseline approaches, in most situations the baselines do not meet the fairness constraints. Since we are interested in \emph{fair} classification, accuracy values can only be compared in the regime where each classifier is fair.

Second, we highlight the difference in the number of samples required to achieve a given accuracy for \ours as compared to the passive baseline. In particular, on the \texttt{Drug}, \texttt{Adult}, and \texttt{Crime} datasets, \ours requires between 1.4-2x less samples than passive to achieve the final accuracy achieved by passive, while ensuring the fairness constraint is still met. While this gain is not present on every dataset---for \texttt{Bank} and \texttt{Compas} the performance of \ours and the passive baseline are comparable---these results illustrate that active learning can yield substantial gains over passive approaches for fair classification, while simultaneously ensuring fairness constraints are met. \loose

\vspace{-.2cm}

\paragraph{Fairness Constraints Beyond TPRP.}
The previously considered results illustrate the performance of each method when the fairness constraint is TPRP. To illustrate the generality of our approach, in \Cref{fig:adult_eo} we also consider the performance of each method when the fairness constraint is equalized odds. As with TPRP, we see that \ours produces a fair classifier while existing approaches fail to, and yields a marked improvement over the passive baseline in terms of accuracy.
% \vspace{-.2cm}

\iftoggle{arxiv}{
\begin{figure}[h!]
    \centering
    \includegraphics[width=0.85\textwidth]{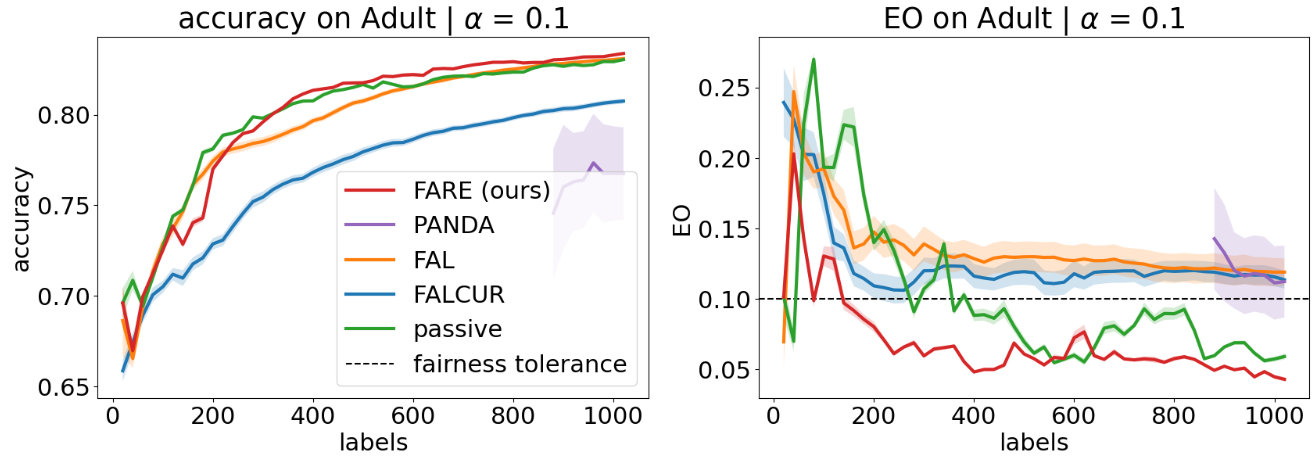}
    \caption{Performance on the \texttt{Adult Income} dataset for Equalized Odds}
    \label{fig:adult_eo}
\end{figure}
}{
\begin{figure}[h!]
    \centering
    \includegraphics[width=0.45\textwidth]{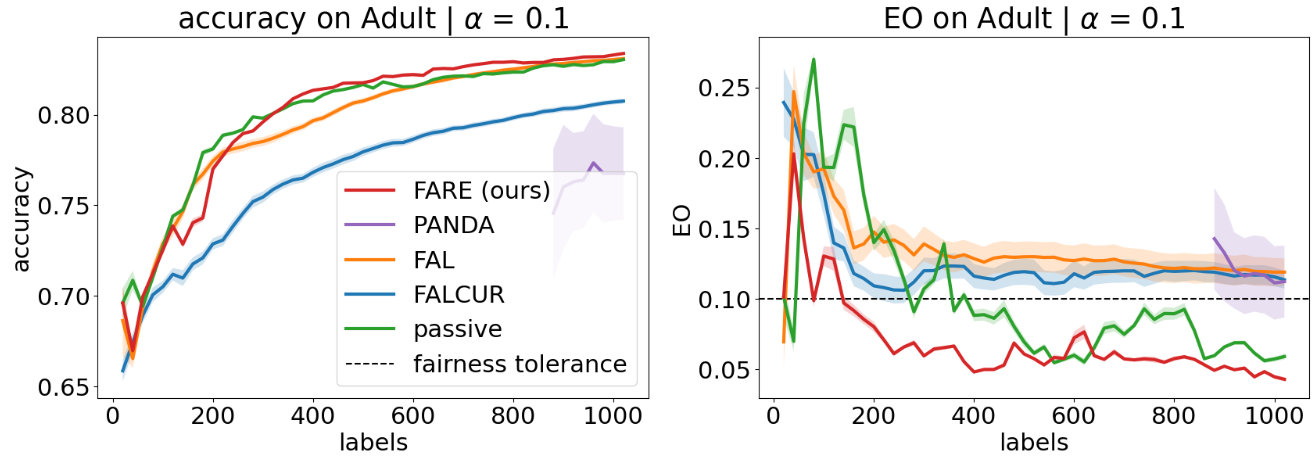}
    \caption{Performance on the \texttt{Adult Income} dataset for Equalized Odds}
    \label{fig:adult_eo}
\end{figure}
\vspace{-.1cm}
}

\subsection{Ablation Experiments}\label{sec:ablation}
In this section, we illustrate the critical nature of two features of \ours. First, in \Cref{fig:correction_ablation}, we compare the performance of \ours with the fairness tolerance $\alpha - 1/\sqrt{n}$, with the $1/\sqrt{n}$ term correcting for the estimation error in the fairness constraint,  to the performance with the fairness tolerance simply set to $\alpha$. As shown, with the $1/\sqrt{n}$ correction, the classifier returned by \ours is unfair, while with the correction it is fair. We remark as well that, though the $1/\sqrt{n}$ correction is not precisely what is justified by  \Cref{cor:fairness_est_simple}, this value nonetheless consistently produces fair classifiers.
\iftoggle{arxiv}{
\begin{figure}[h!]
    \centering
    \includegraphics[width=0.85\textwidth]{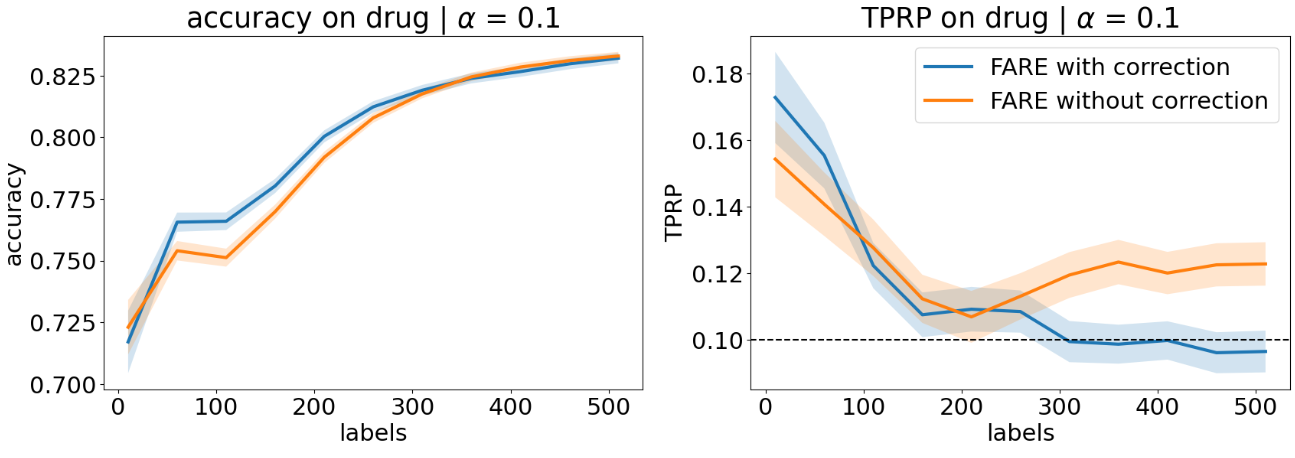}
    \caption{Ablation on fairness tolerance correction on \texttt{Drug} dataset}
    \label{fig:correction_ablation}
\end{figure} 
}{
\vspace{-.2cm}
\begin{figure}[h!]
    \centering
    \includegraphics[width=0.45\textwidth]{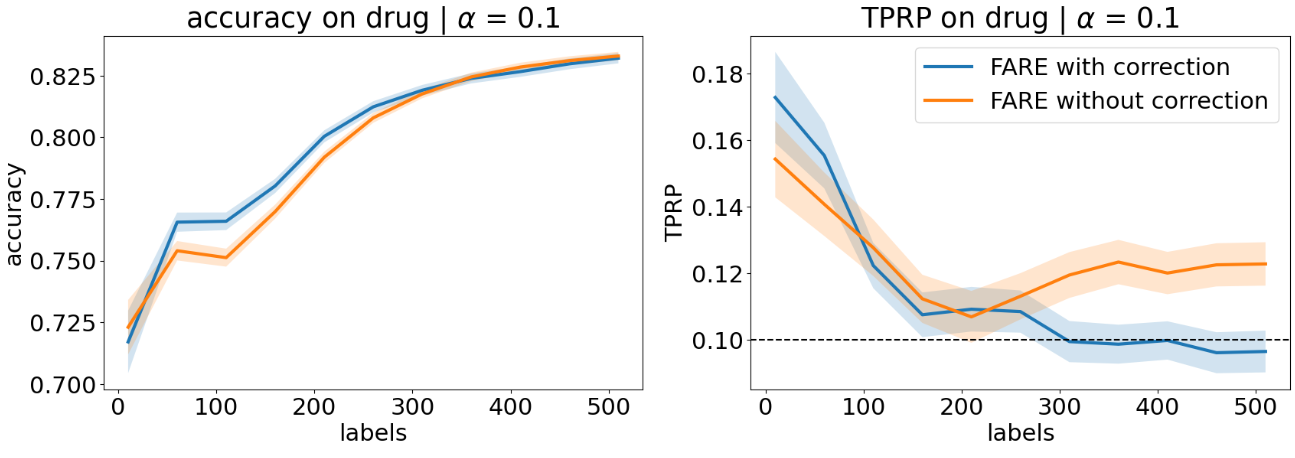}
    \caption{Ablation on fairness tolerance correction on \texttt{Drug} dataset}
    \label{fig:correction_ablation}
\end{figure} 
\vspace{-.2cm}
}

Lastly, in \Cref{tab:results_ablation}, we compare the performance of \ours with and without $\lamfair$, and additionally compare to the performance of the other baselines methods. We evaluate this on a synthetically generated dataset for which there is a large group imbalance---one group has significantly more examples in the dataset than the other. In this setting, if points are not explicitly sampled from the group with the smaller number of examples, virtually all samples will be taken from the larger group, which will cause the fairness estimates to be inaccurate, the resulting classifier unfair. This is illustrated in \Cref{tab:results_ablation}, where we see that without $\lamfair$, \ours produces an unfair classifier, similar to existing approaches. However, with $\lamfair$, \ours successfully achieves fairness. In conclusion, the inclusion of $\lamfair$ in \ours effectively ensures fairness constraints are met, especially when dealing with a significant group imbalance in the dataset.\loose

\iftoggle{arxiv}{
\paragraph{Generality of \ours.}
Finally, we delve into the versatility of our proposed method, \ours, by conducting experiments across different machine learning models, specifically comparing its performance on decision trees with that on logistic regression. Our objective is to demonstrate the general applicability of \ours and highlight its consistent effectiveness in guaranteeing fairness fairness across various model architectures.

To assess the generality of \ours, we compare results from experiments conducted on both decision trees and logistic regression models. In \Cref{fig:dt_lr_comparison}, we showcase the performance of \ours on decision trees compared to logistic regression. Notably, the observed gains in accuracy and the fairness guarantees achieved by \ours on decision trees closely parallel those attained on logistic regression, compared to the passive baseline. This emphasizes the validity of our proposed method across different model types.
\begin{figure}
    \centering
    \includegraphics[width=0.85\textwidth]{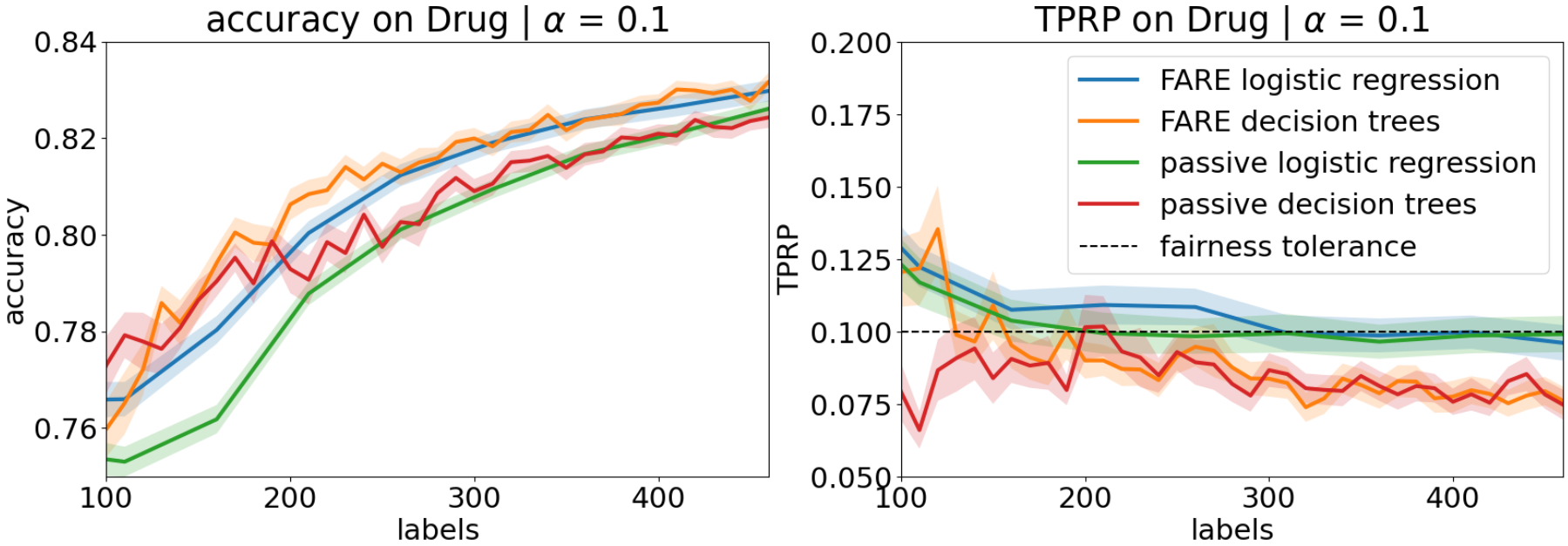}
    \caption{Performance on the \texttt{Drug Consumption} dataset for logistic regression and decision trees model}
    \label{fig:dt_lr_comparison}
\end{figure}
}{}

\section{Conclusion}

In conclusion, this paper introduces a novel active learning framework designed to tackle the challenges of bias reduction and accuracy improvement in data-scarce environments critical to machine learning applications. By combining an exploration procedure inspired by posterior sampling with a fair classification subroutine, our proposed approach effectively maximizes accuracy while ensuring fairness constraints in very data-scarce regimes. Through comprehensive evaluations on established real-world benchmark datasets, we demonstrate the efficacy of our framework, highlighting its superiority over state-of-the-art methods. This work contributes to advancing the development of fair models in situations where collecting large labeled datasets is impractical, offering a promising solution for critical applications in machine learning.

\newpage
\bibliographystyle{apalike}
\bibliography{ref}
\newpage

\newpage
\appendix

\section{Datasets description}

\textbf{\textit{\texttt{Adult income} dataset} \citep{lichman2013UCI}}: This dataset comprises $48,842$ examples with demographic information. The task is to predict whether an individual's income exceeds $50k\$$ annually. We chose the protected attribute to be binarized gender.

\textbf{\textit{\texttt{Compas} dataset} \citep{lichman2013UCI}}: This dataset, which was released by \citet{angwin2022machine}, encompasses $5,278$ data related to juvenile felonies. It includes details such as marital status, ethnicity, age, prior criminal history, and the severity of the current arrest charges. In our analysis, we identify binarized gender as a sensitive attribute. In line with established conventions \citep{corbett2017algorithmic, anahideh2021fair}, we adopt a two-year violent recidivism record as the ground truth for assessing recidivism.

\textbf{\textit{\texttt{Drug consumption} dataset \citep{fehrman2017factor}}}: 
%This dataset comes from the
%UCI machine learning repository \cite{dua2019UCI}, and 
This dataset consists of $1,885$ entries containing information about individuals, where each entry includes five demographic characteristics (such as Age, binarized Gender, or Education), seven measurements related to personality traits (such as Nscore indicating neuroticism and Ascore representing agreeableness), and 18 descriptors detailing the subject's most recent consumption of a specific substance (like Cannabis). We chose the task of predicting whether an individual consumed Cannabis in the last year and chose the protected attribute to be (binarized) Gender. 

\textbf{\textit{\texttt{German Credit} dataset \citep{hofmann1994statlog}}}: The German Credit dataset classifies people as good or bad credit risks using the profile and history of $1,000$ clients. We set the binarized gender as the sensitive attribute.

\textbf{\textit{\texttt{Community and Crime} dataset \citep{redmond2002data}}}: 
The Crime and Community dataset consists of $1,902$ instances of crimes with $128$ attributes related to the crime and the corresponding community. It uses `violent crimes' as the target variable and combines `percentage of non-white' as the protected attribute. The target variable is binarized to categorize communities as high or low crime based on a threshold of $500$. The protected attribute is also binarized, separating communities with non-white residents below $20\%$.\loose

\textbf{\textit{\texttt{Bank} dataset \citep{moro2014data}}}:  The task is to predict whether the client has subscribed to a term deposit service based on $11,162$ data points with features such as marital status and age. We set the client having tertiary education as the sensitive attribute.

\textbf{\textit{\texttt{Synthetic} dataset}}: We created the synthetic dataset in the following manner. It is depicted in \Cref{fig:synthetic}. The dataset consists of two dimensions, and data for group $0$ is generated by randomly sampling $10,000$ data points from a Gaussian distribution with a mean of $(0, 0)$, while group $1$ comprises $100$ data points sampled from $(10, 10)$. For group $0$ (and group $1$), labels are assigned a value of $1$ if the x-coordinate (or y-coordinate) of the data point is greater than $0$, and $0$ otherwise. This ensures that each group is linearly separable, but their combination is not.
\begin{figure}
    \centering
    \includegraphics[width=0.45\textwidth]{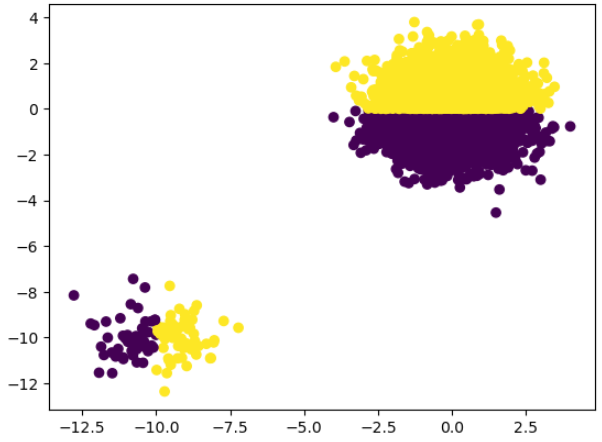}
    \caption{Synthetic dataset}
    \label{fig:synthetic}
\end{figure}

\section{Performance of baseline algorithms with different pre-trained dataset sizes}
We report the results of the sweeps over the size of the pretrain dataset in \Cref{fig:drug_panda,fig:bank_panda,fig:german_panda,fig:adult_panda,fig:compas_panda,fig:crime_panda,fig:drug_fal,fig:bank_fal,fig:german_fal,fig:adult_fal,fig:compas_fal,fig:crime_fal}.
Due to its large computational cost, we compared the performance of \Panda for two sizes of pretrain datasets.

\iftoggle{arxiv}{

\begin{figure*}
\begin{minipage}[c]{0.5\linewidth}
    \centering
    \includegraphics[width=\textwidth]{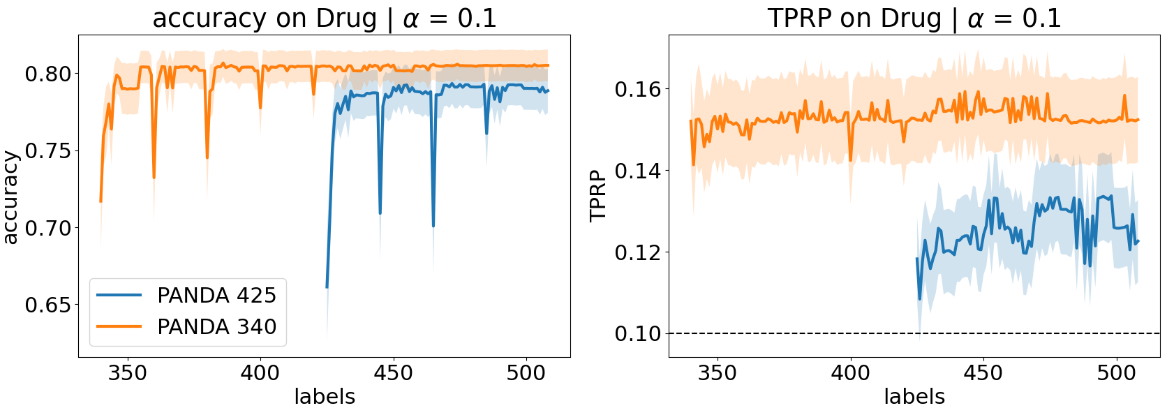}
    \caption{Performance on \texttt{Drug Consumption}}
    \label{fig:drug_panda}
\end{minipage}
\hfill
\begin{minipage}[c]{0.5\linewidth}
    \centering
    \includegraphics[width=\textwidth]{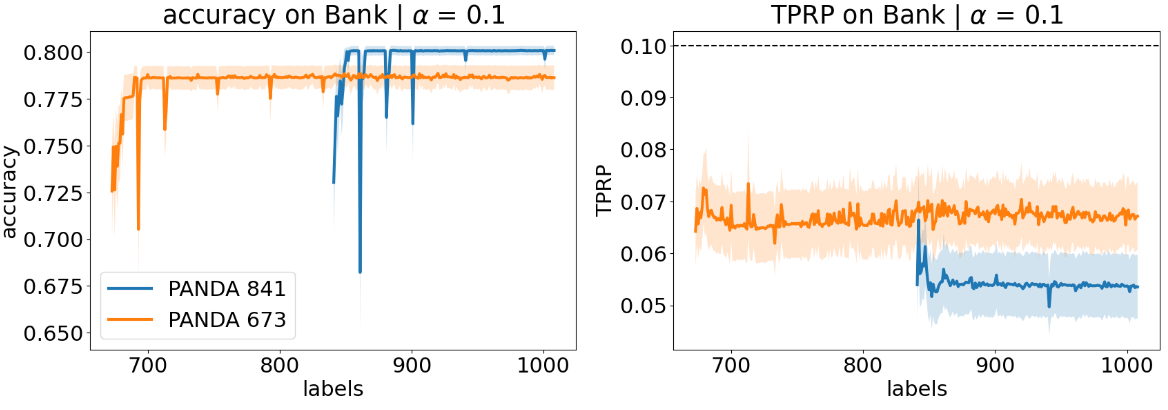}
    \caption{Performance on \texttt{Bank}}
    \label{fig:bank_panda}
\end{minipage}
\\
\begin{minipage}[c]{0.5\linewidth}
    \centering
    \includegraphics[width=\textwidth]{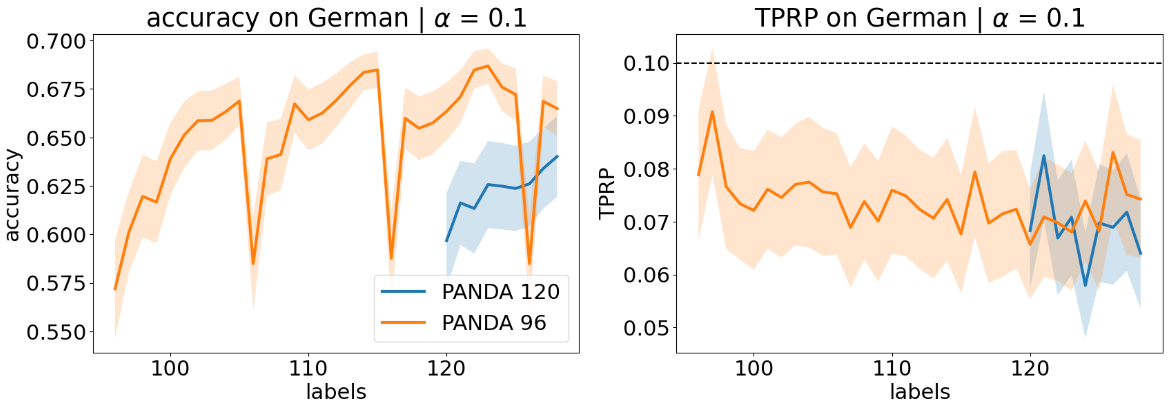}
    \caption{Performance on \texttt{German Credit}}
    \label{fig:german_panda}
\end{minipage}
\hfill
\begin{minipage}[c]{0.5\linewidth}
    \centering
    \includegraphics[width=\textwidth]{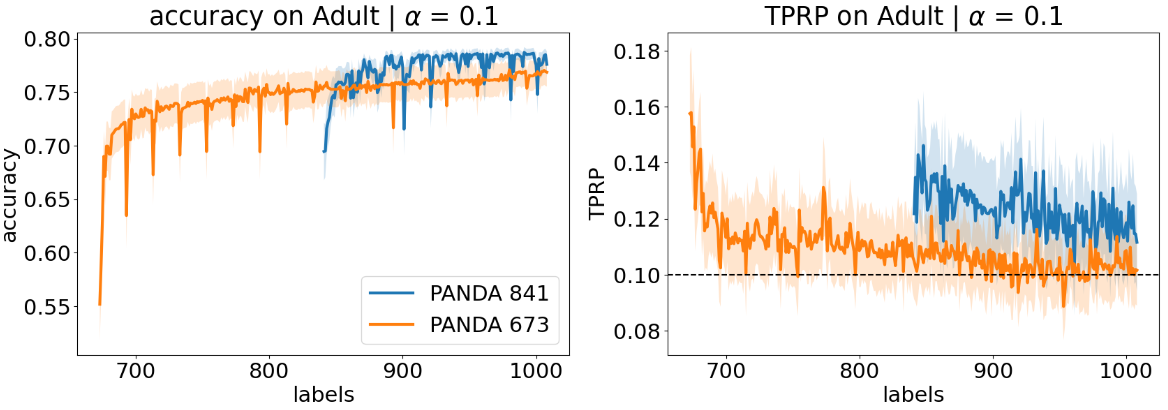}
    \caption{Performance on \texttt{Adult Income}}
    \label{fig:adult_panda}
\end{minipage}
\\
\begin{minipage}[c]{0.5\linewidth}
    \centering
    \includegraphics[width=\textwidth]{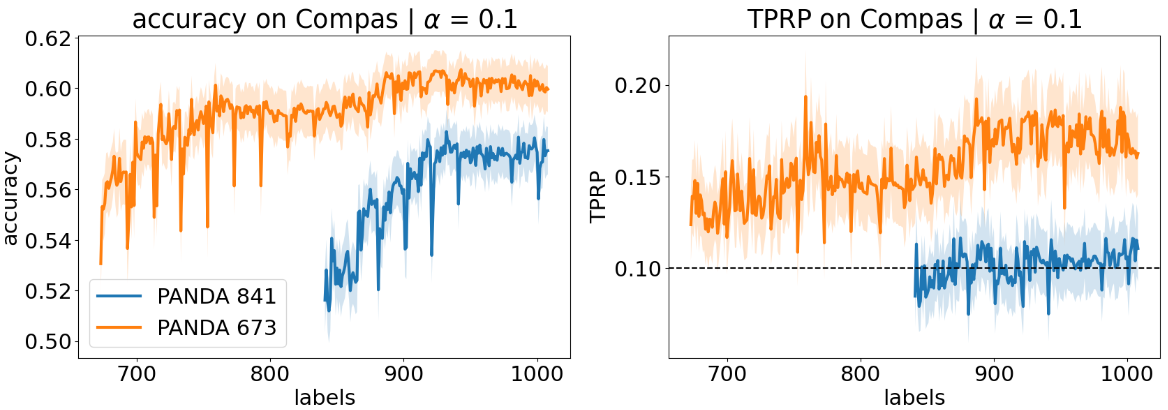}
    \caption{Performance on \texttt{Compas}}
    \label{fig:compas_panda}
\end{minipage}
\hfill
%\begin{figure}[t]
\begin{minipage}[c]{0.5\linewidth}
    \centering
    \includegraphics[width=\textwidth]{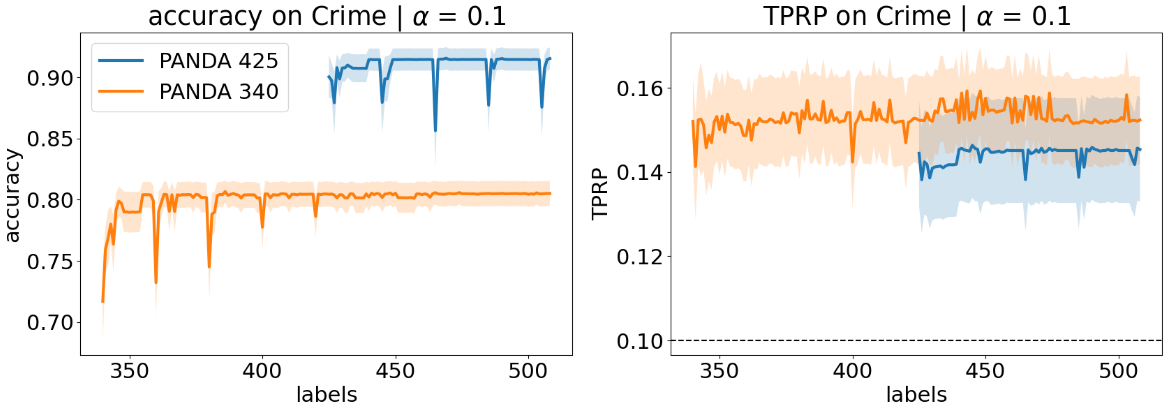}
    \caption{Performance on \texttt{Community and Crime}}
    \label{fig:crime_panda}
\end{minipage}
\end{figure*}

\begin{figure*}
\begin{minipage}[c]{0.5\linewidth}
    \centering
    \includegraphics[width=\textwidth]{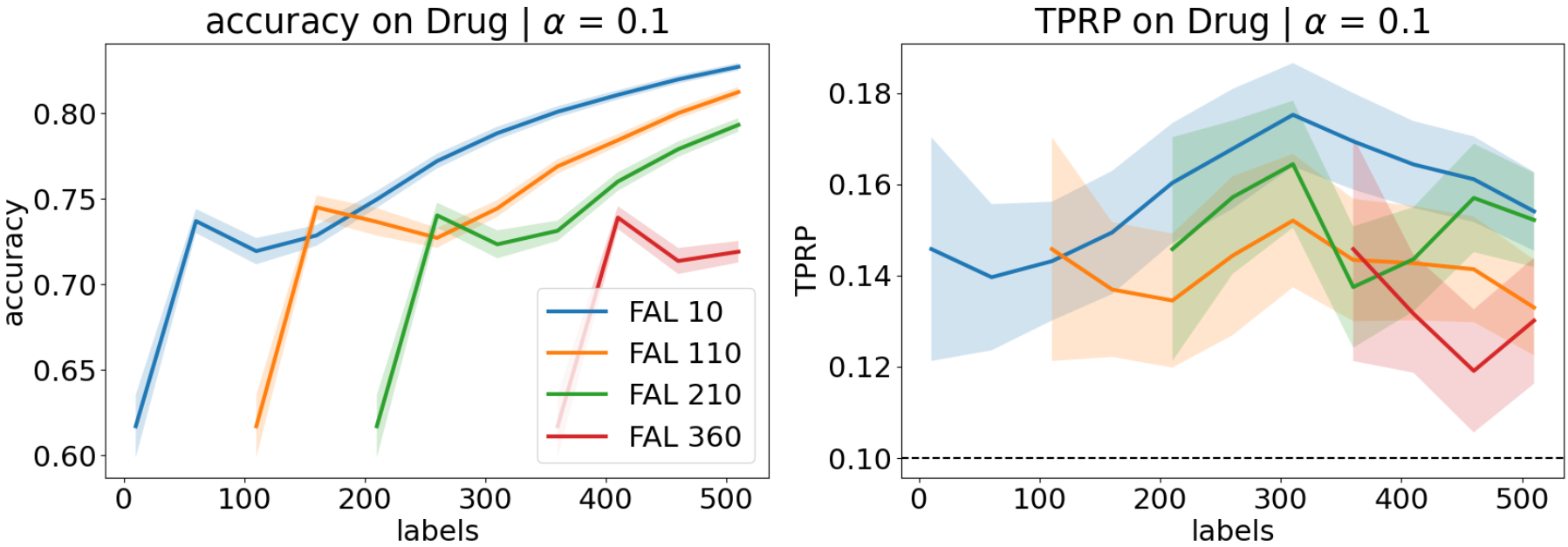}
    \caption{Performance on \texttt{Drug Consumption}}
    \label{fig:drug_fal}
\end{minipage}
\hfill
\begin{minipage}[c]{0.5\linewidth}
    \centering
    \includegraphics[width=\textwidth]{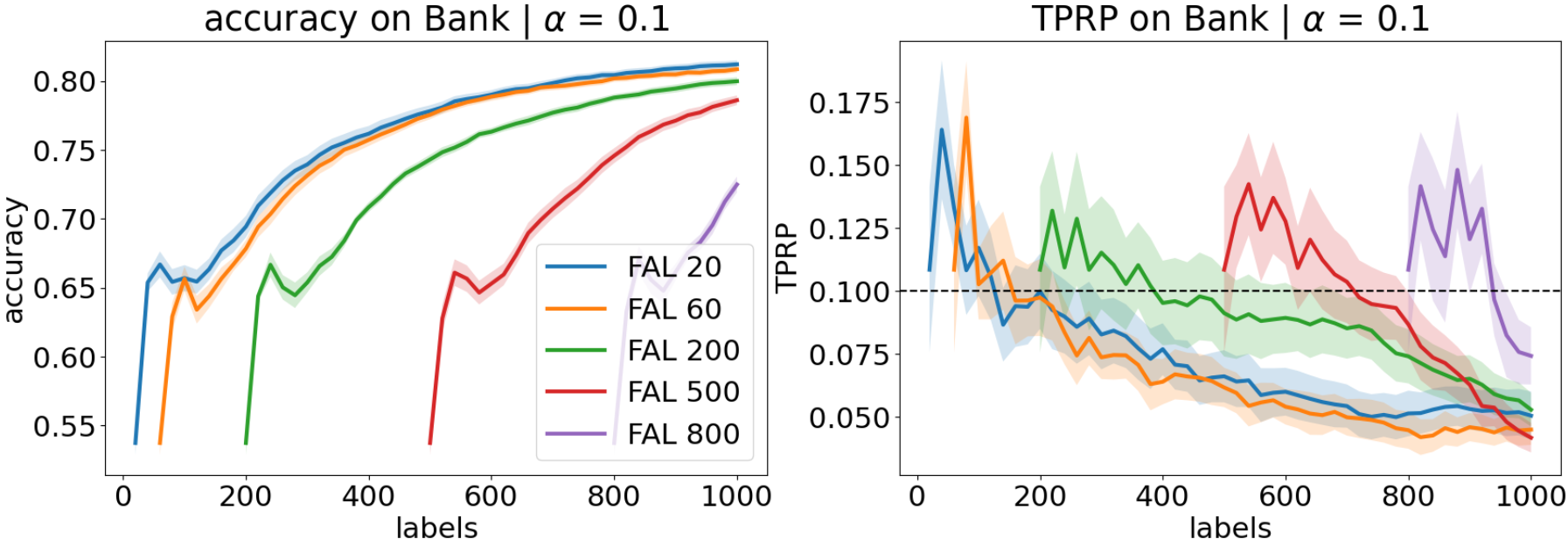}
    \caption{Performance on \texttt{Bank}}
    \label{fig:bank_fal}
\end{minipage}
\\
\begin{minipage}[c]{0.5\linewidth}
    \centering
    \includegraphics[width=\textwidth]{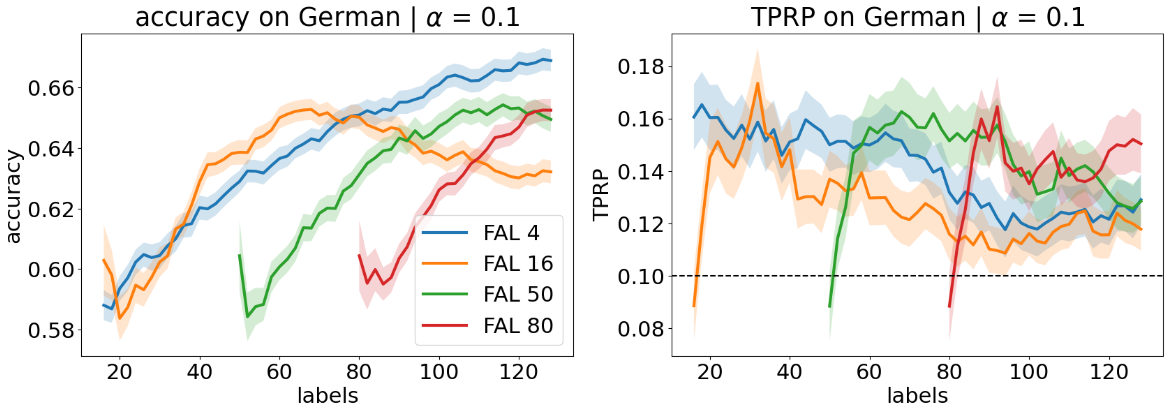}
    \caption{Performance on \texttt{German Credit}}
    \label{fig:german_fal}
\end{minipage}
\hfill
\begin{minipage}[c]{0.5\linewidth}
    \centering
    \includegraphics[width=\textwidth]{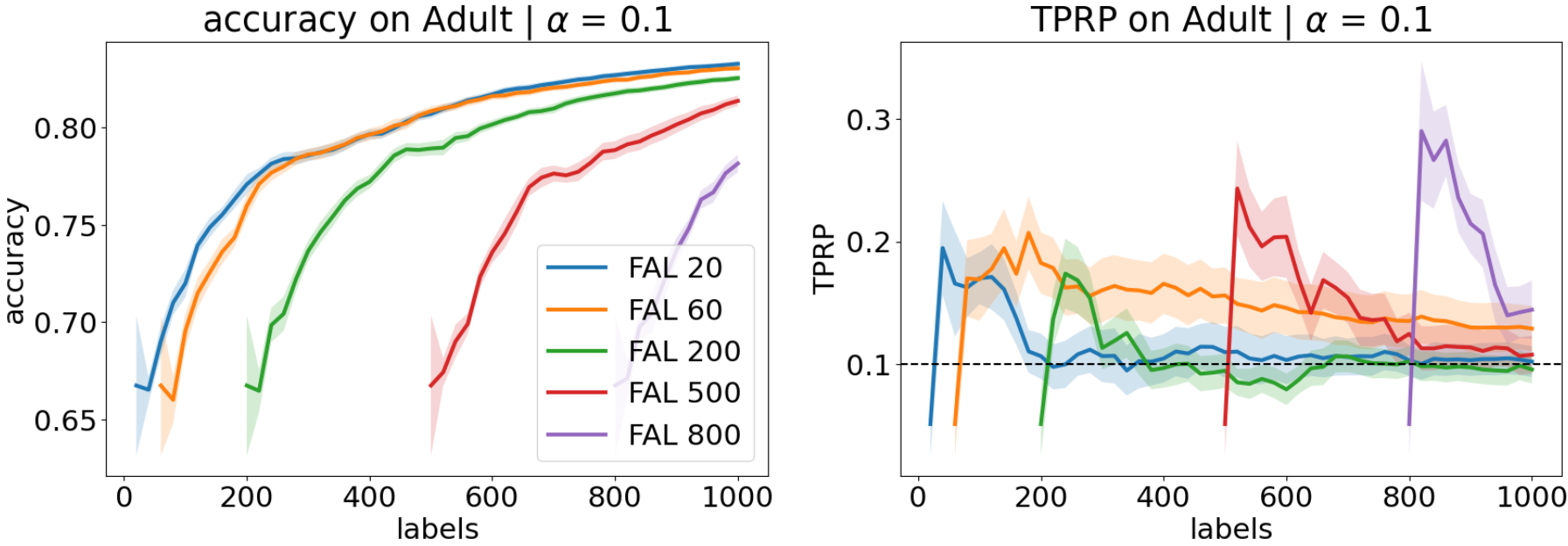}
    \caption{Performance on \texttt{Adult Income}}
    \label{fig:adult_fal}
\end{minipage}
\\
\begin{minipage}[c]{0.5\linewidth}
    \centering
    \includegraphics[width=\textwidth]{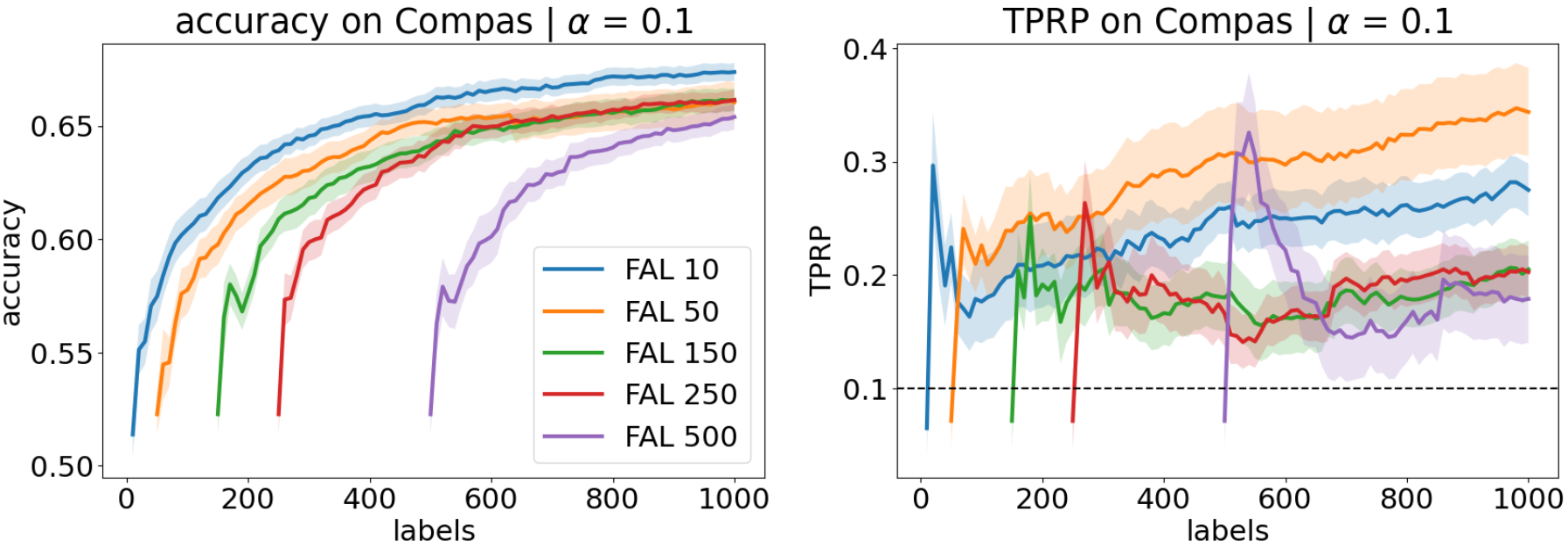}
    \caption{Performance on \texttt{Compas}}
    \label{fig:compas_fal}
\end{minipage}
\hfill
%\begin{figure}[t]
\begin{minipage}[c]{0.5\linewidth}
    \centering
    \includegraphics[width=\textwidth]{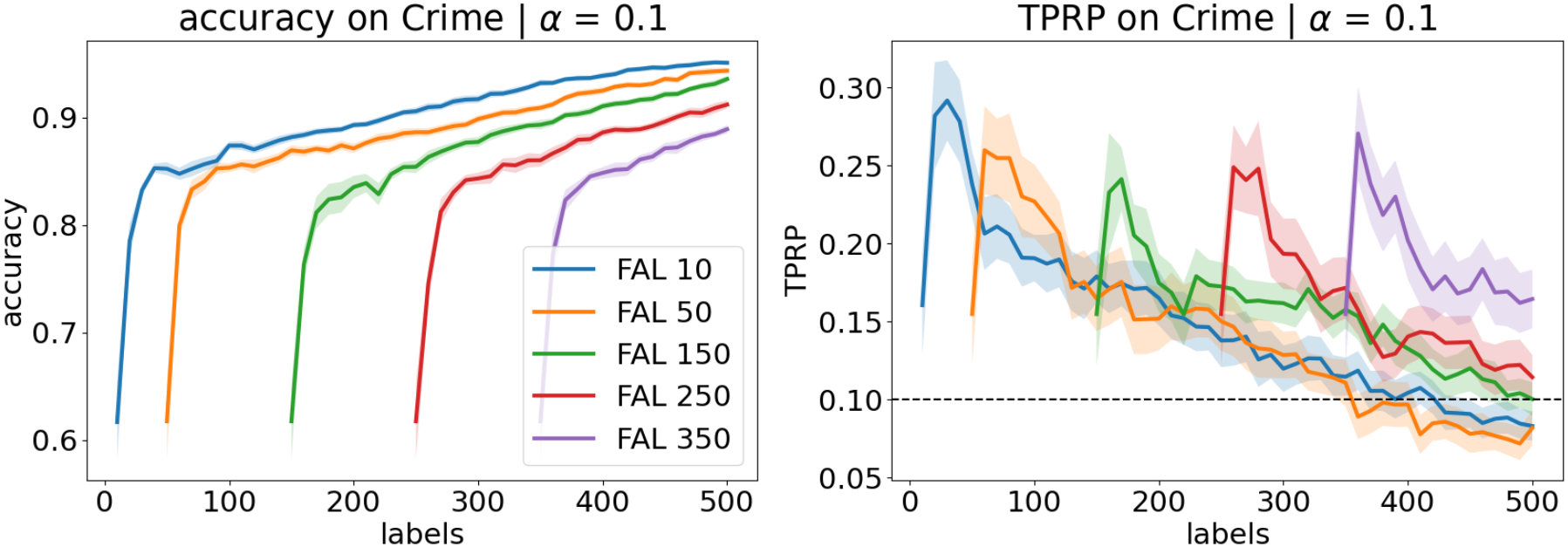}
    \caption{Performance on \texttt{Community and Crime}}
    \label{fig:crime_fal}
\end{minipage}
\end{figure*}

\begin{figure*}
\begin{minipage}[c]{0.5\linewidth}
    \centering
    \includegraphics[width=\textwidth]{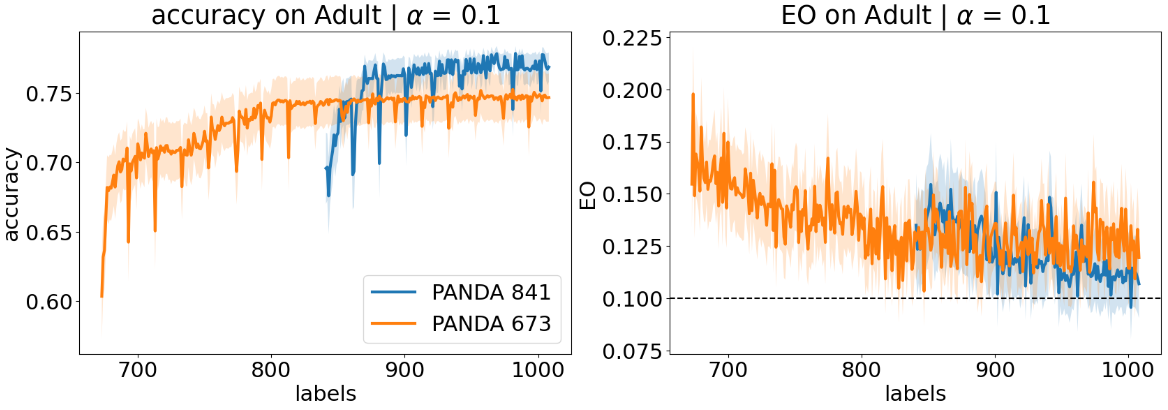}
    \caption{Performance on \texttt{Adult Income} for Equalized Odds}
    \label{fig:adulteo_panda}
\end{minipage}
\hfill
\begin{minipage}[c]{0.5\linewidth}
    \centering
    \includegraphics[width=\textwidth]{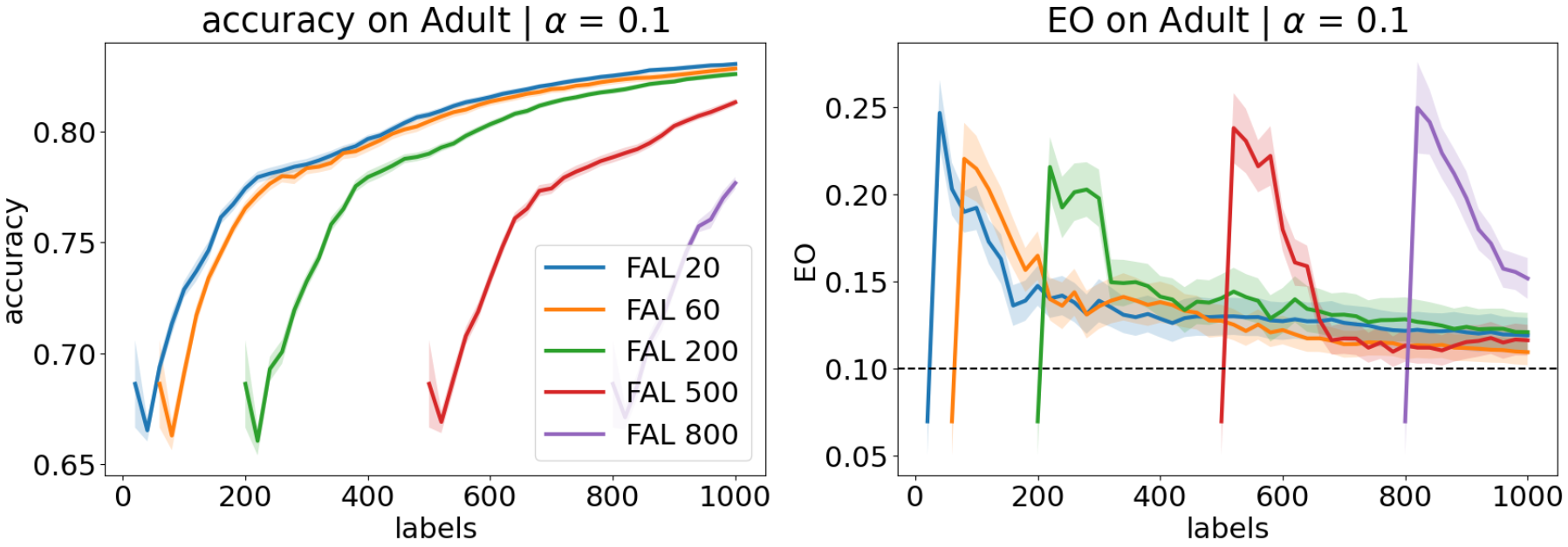}
    \caption{Performance on \texttt{Adult Income} for Equalized Odds}
    \label{fig:adulteo_fal}
\end{minipage}
\end{figure*}
}{

\begin{figure*}
\begin{minipage}[c]{0.45\linewidth}
    \centering
    \includegraphics[width=\textwidth]{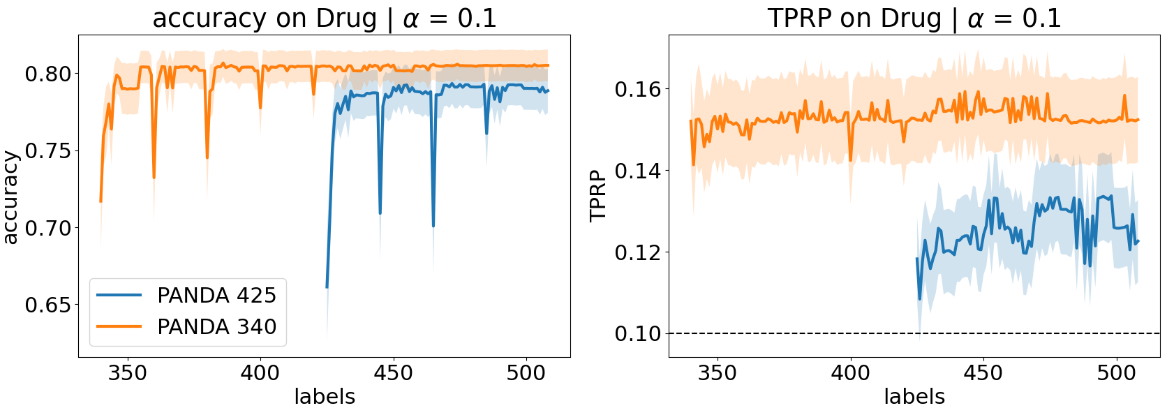}
    \caption{Performance on \texttt{Drug Consumption}}
    \label{fig:drug_panda}
\end{minipage}
\hfill
\begin{minipage}[c]{0.45\linewidth}
    \centering
    \includegraphics[width=\textwidth]{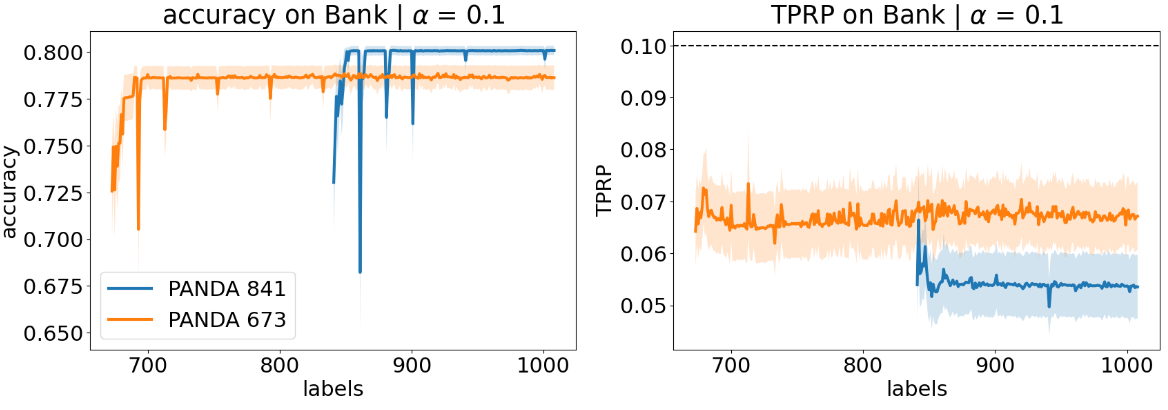}
    \caption{Performance on \texttt{Bank}}
    \label{fig:bank_panda}
\end{minipage}
\\
\begin{minipage}[c]{0.45\linewidth}
    \centering
    \includegraphics[width=\textwidth]{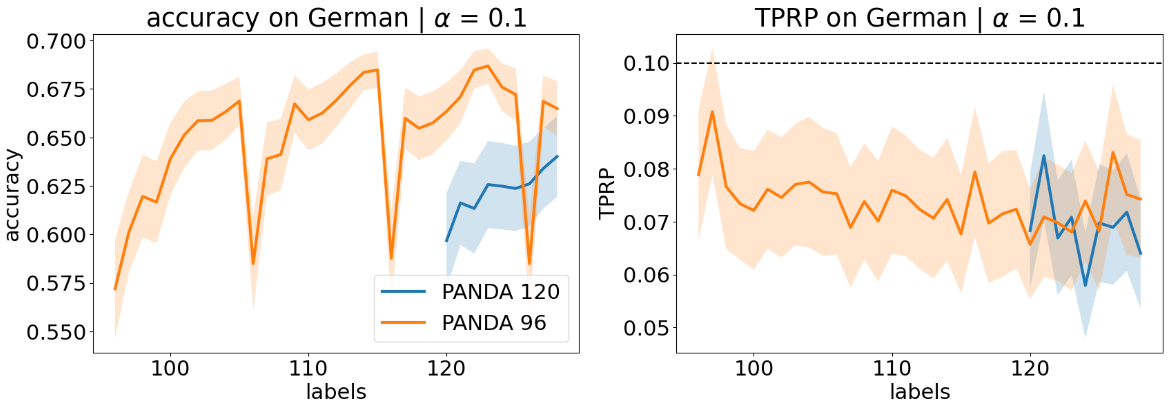}
    \caption{Performance on \texttt{German Credit}}
    \label{fig:german_panda}
\end{minipage}
\hfill
\begin{minipage}[c]{0.45\linewidth}
    \centering
    \includegraphics[width=\textwidth]{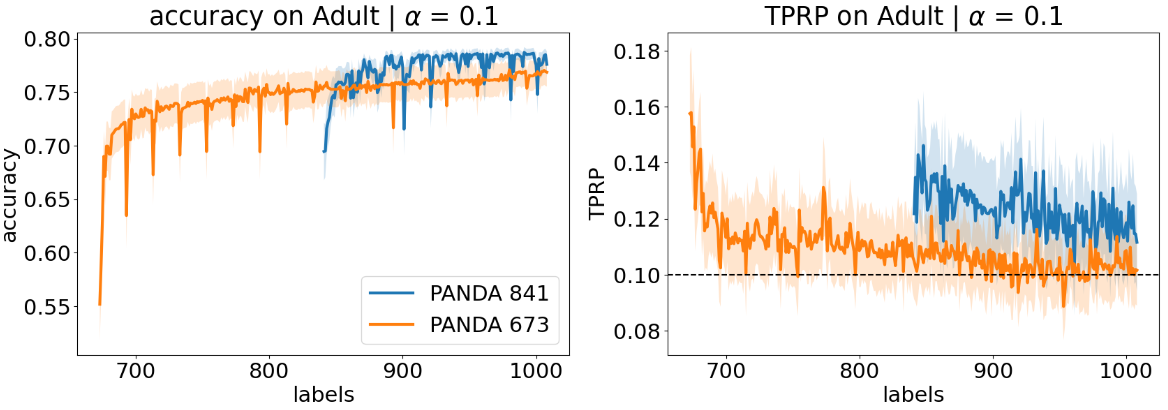}
    \caption{Performance on \texttt{Adult Income}}
    \label{fig:adult_panda}
\end{minipage}
\\
\begin{minipage}[c]{0.45\linewidth}
    \centering
    \includegraphics[width=\textwidth]{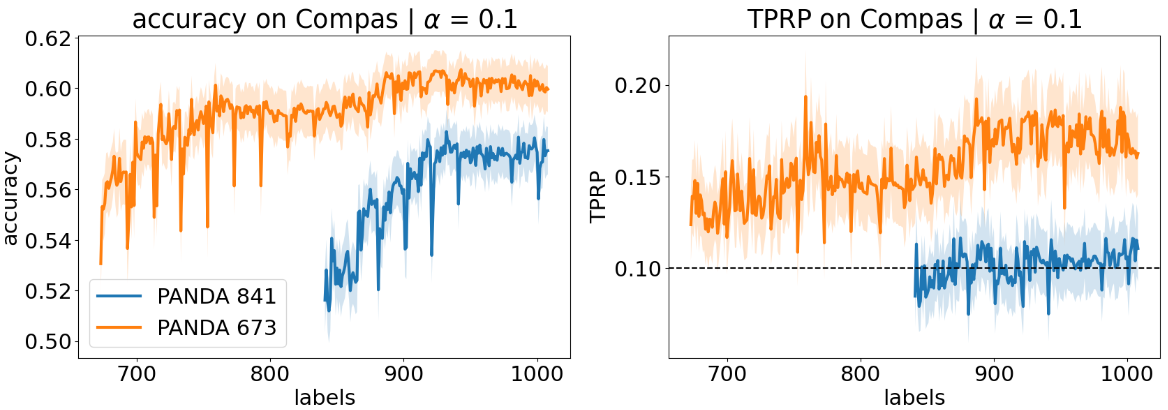}
    \caption{Performance on \texttt{Compas}}
    \label{fig:compas_panda}
\end{minipage}
\hfill
%\begin{figure}[t]
\begin{minipage}[c]{0.45\linewidth}
    \centering
    \includegraphics[width=\textwidth]{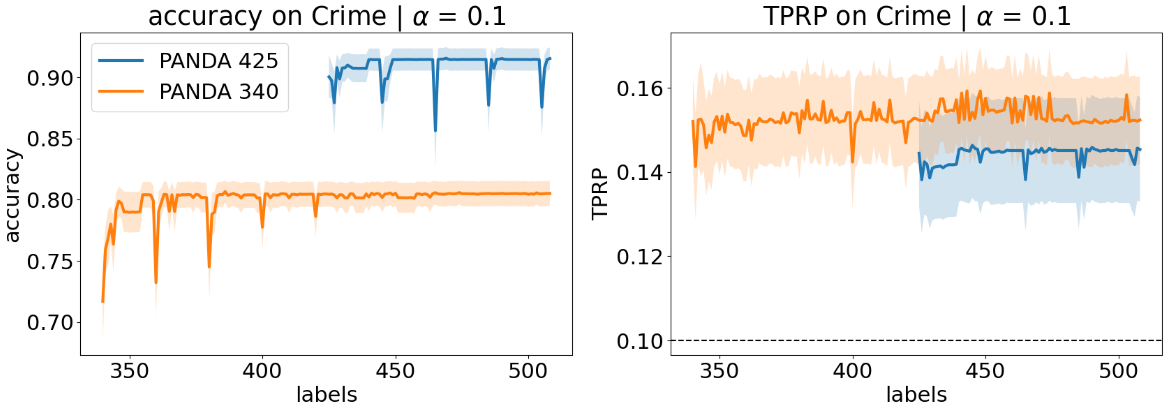}
    \caption{Performance on \texttt{Community and Crime}}
    \label{fig:crime_panda}
\end{minipage}
\end{figure*}

\begin{figure*}
\begin{minipage}[c]{0.45\linewidth}
    \centering
    \includegraphics[width=\textwidth]{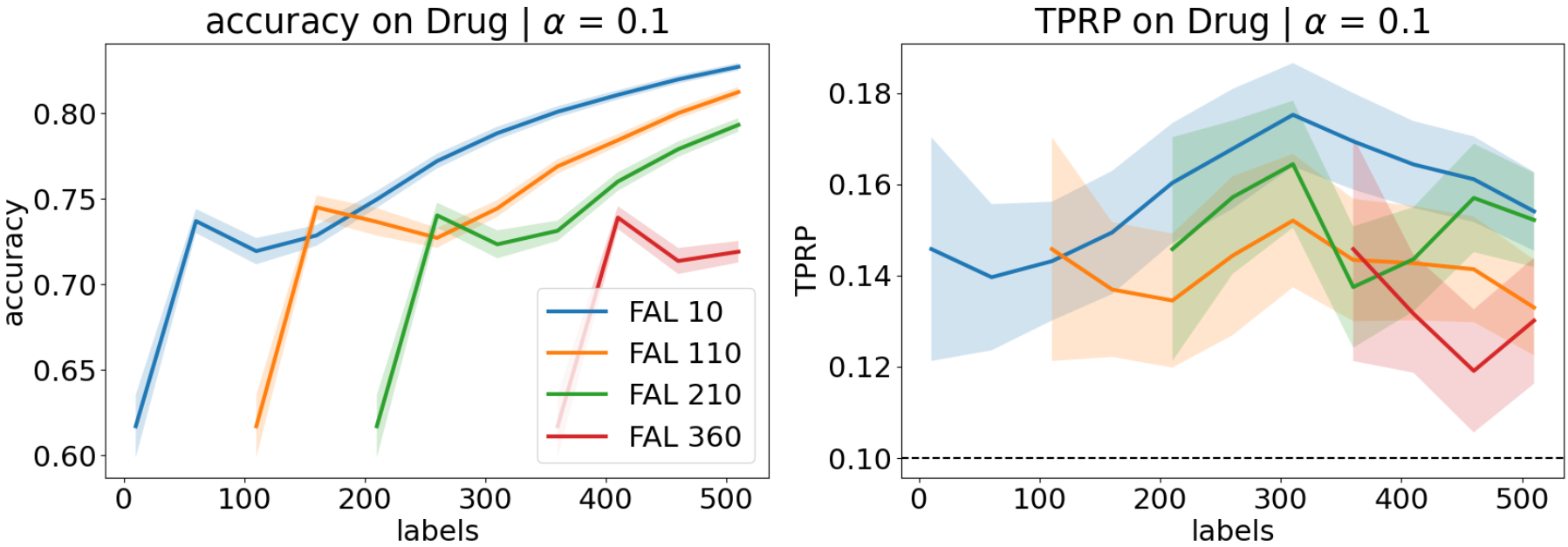}
    \caption{Performance on \texttt{Drug Consumption}}
    \label{fig:drug_fal}
\end{minipage}
\hfill
\begin{minipage}[c]{0.45\linewidth}
    \centering
    \includegraphics[width=\textwidth]{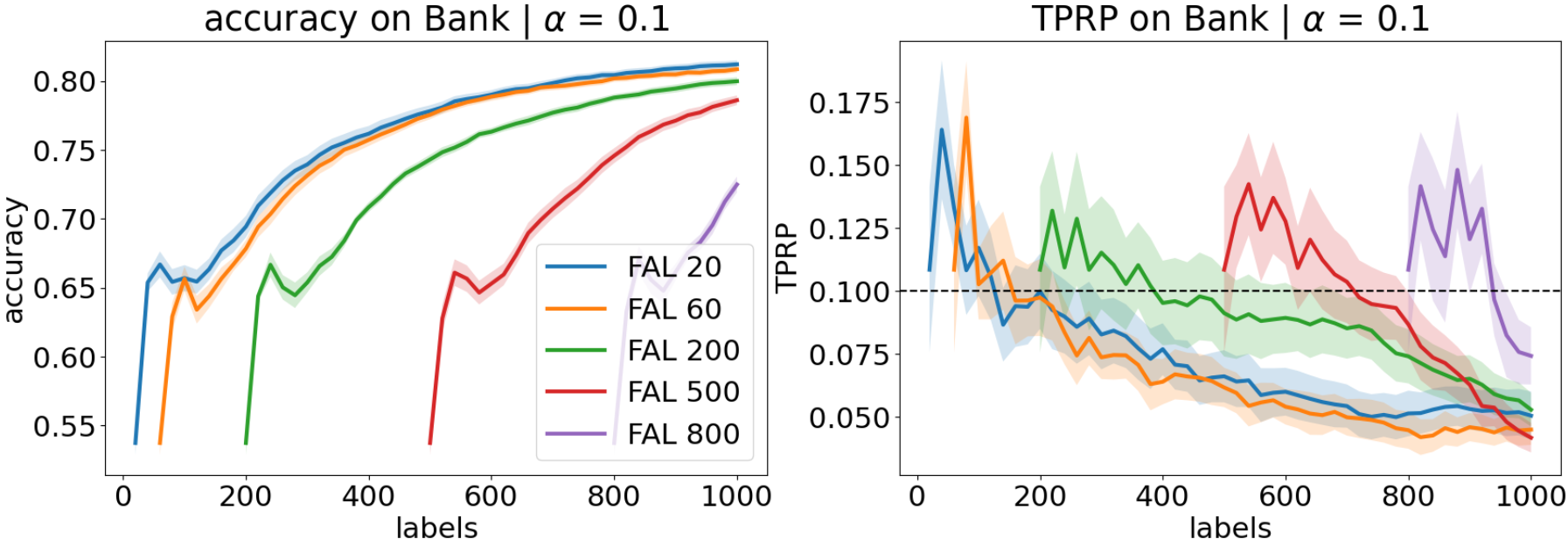}
    \caption{Performance on \texttt{Bank}}
    \label{fig:bank_fal}
\end{minipage}
\\
\begin{minipage}[c]{0.45\linewidth}
    \centering
    \includegraphics[width=\textwidth]{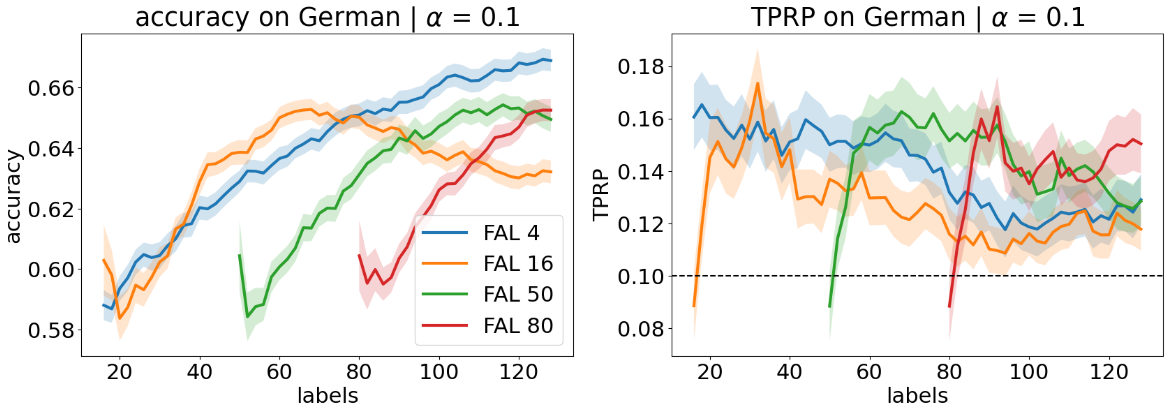}
    \caption{Performance on \texttt{German Credit}}
    \label{fig:german_fal}
\end{minipage}
\hfill
\begin{minipage}[c]{0.45\linewidth}
    \centering
    \includegraphics[width=\textwidth]{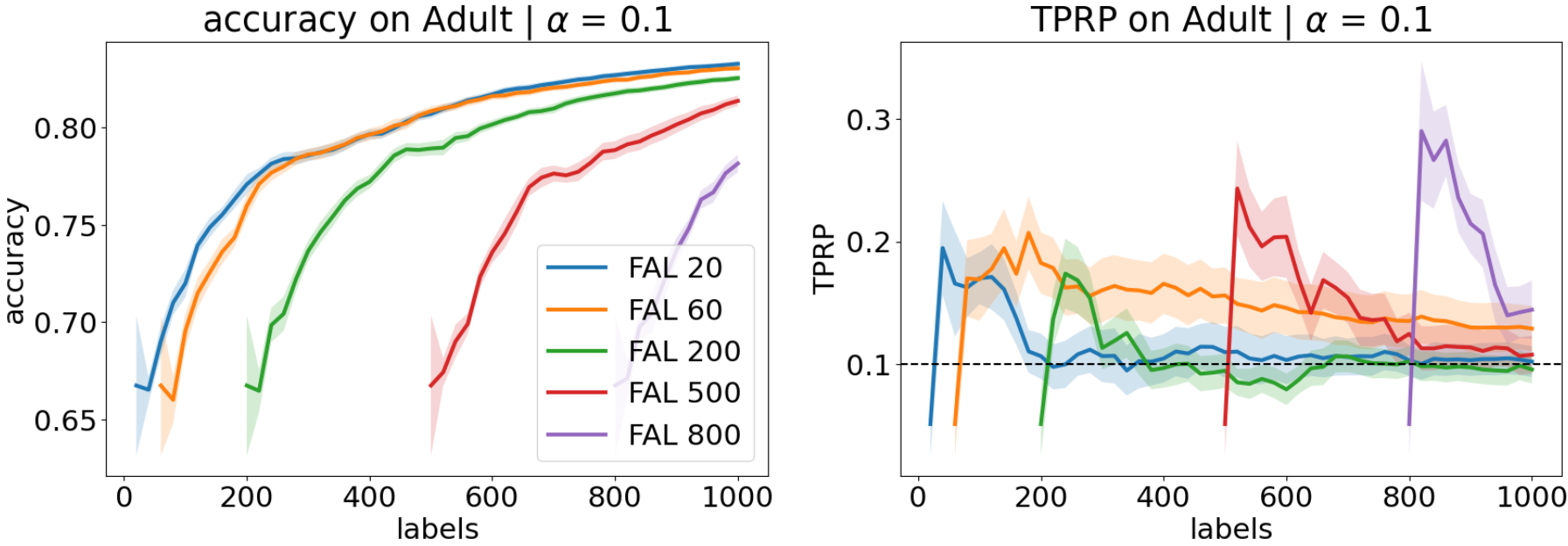}
    \caption{Performance on \texttt{Adult Income}}
    \label{fig:adult_fal}
\end{minipage}
\\
\begin{minipage}[c]{0.45\linewidth}
    \centering
    \includegraphics[width=\textwidth]{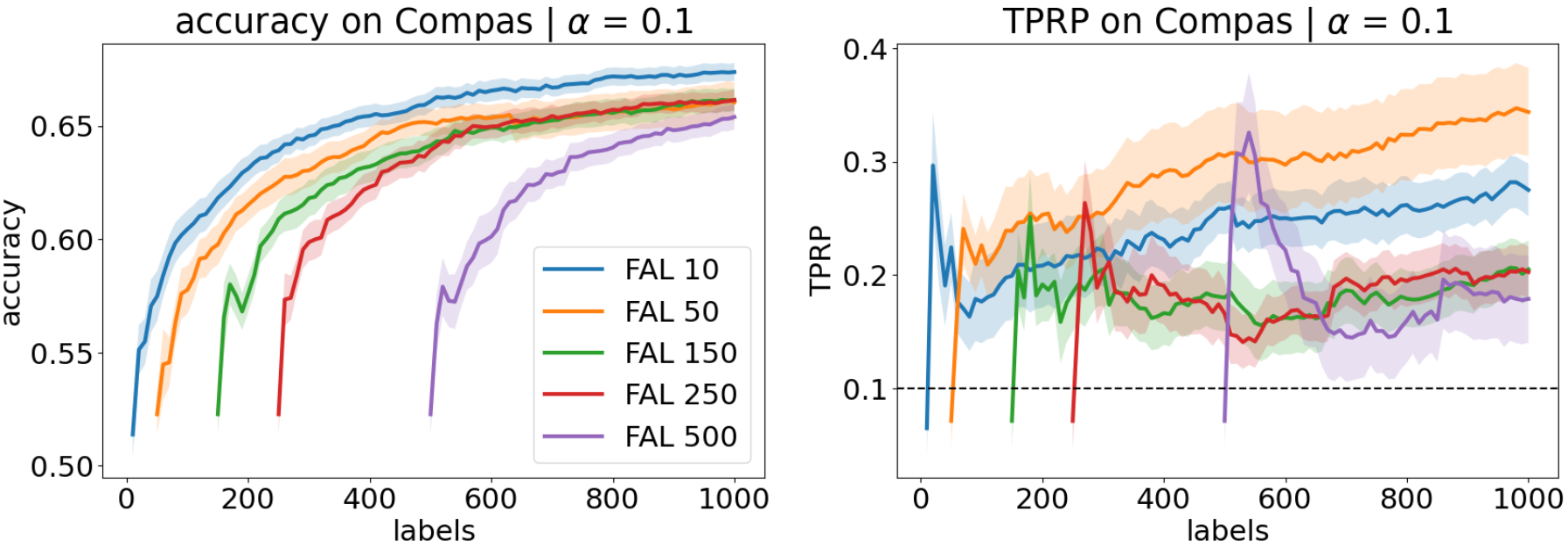}
    \caption{Performance on \texttt{Compas}}
    \label{fig:compas_fal}
\end{minipage}
\hfill
%\begin{figure}[t]
\begin{minipage}[c]{0.45\linewidth}
    \centering
    \includegraphics[width=\textwidth]{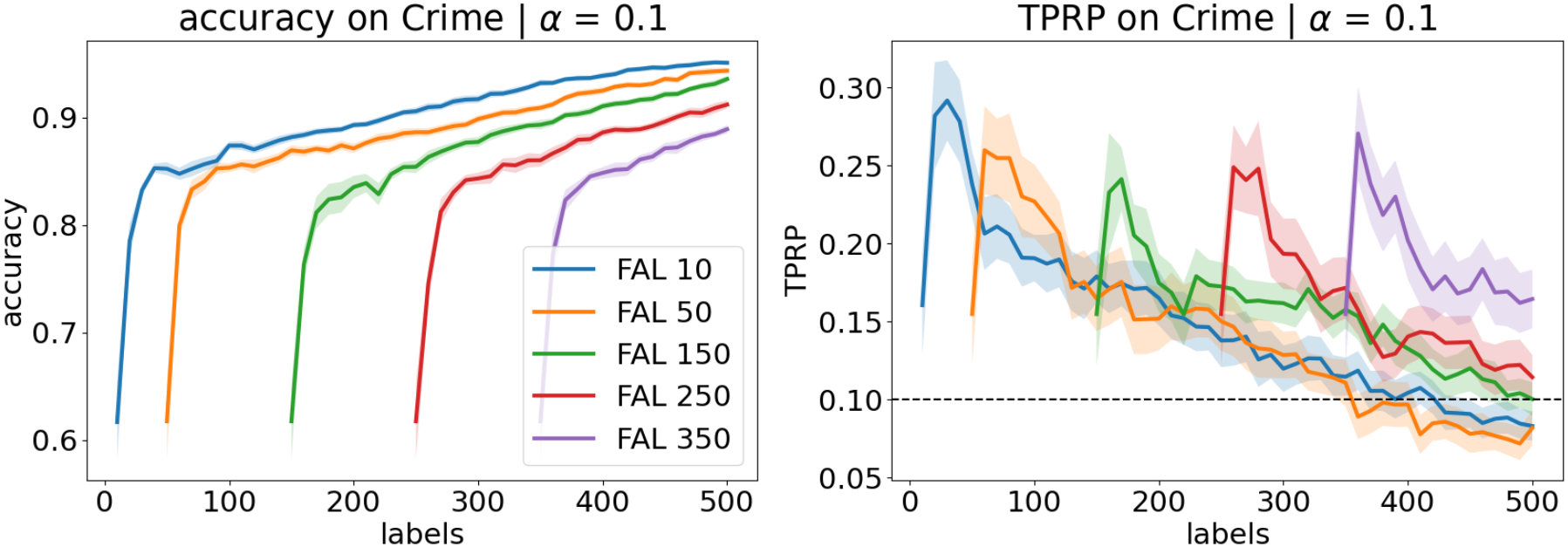}
    \caption{Performance on \texttt{Community and Crime}}
    \label{fig:crime_fal}
\end{minipage}
\end{figure*}

\begin{figure*}
\begin{minipage}[c]{0.45\linewidth}
    \centering
    \includegraphics[width=\textwidth]{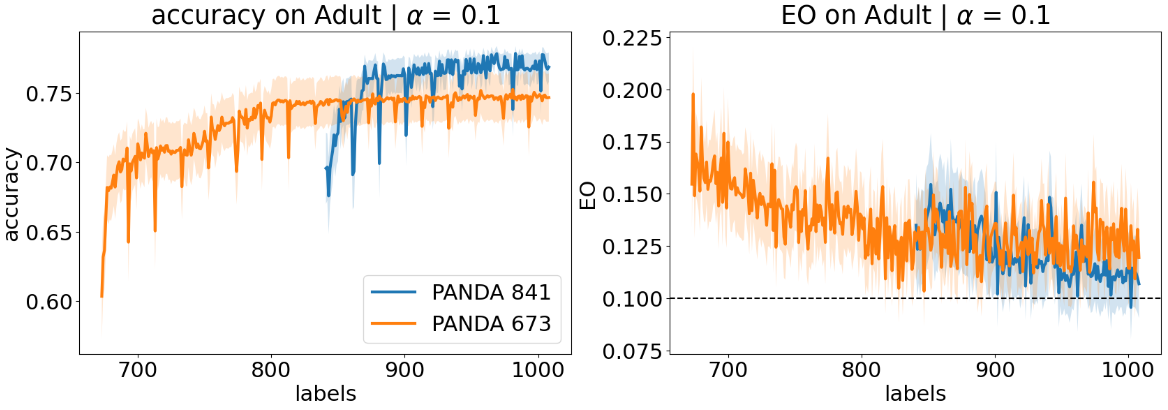}
    \caption{Performance on \texttt{Adult Income} for Equalized Odds}
    \label{fig:adulteo_panda}
\end{minipage}
\hfill
\begin{minipage}[c]{0.45\linewidth}
    \centering
    \includegraphics[width=\textwidth]{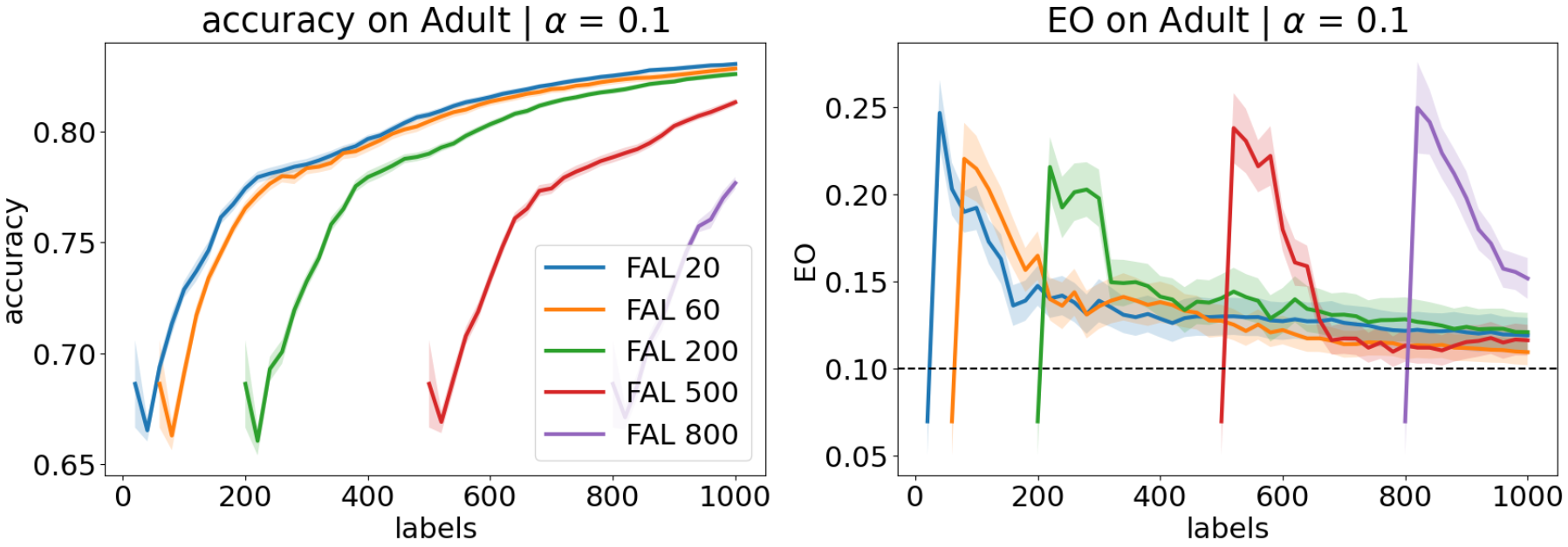}
    \caption{Performance on \texttt{Adult Income} for Equalized Odds}
    \label{fig:adulteo_fal}
\end{minipage}
\end{figure*}

}

\section{Theoretical results - proof of \Cref{cor:fairness_est_simple}}
\subsection{Full theorem}
We have the following result.
\begin{theorem}\label{thm:fairness_est}
Let the train set be  $\mc{D} = \{(x_1, a_1, y_1), \ldots,(x_n, a_n, y_n)\}$. If $\mc{D}\sim \nu$, then it holds with probability $1-\delta$ that:
\begin{align*}
    &|L^{\rm EO}_\nu(h) - \widehat{L}^{\rm EO}_\mc{D}(h)| \leq C_{0,0}+C_{0,1}+C_{1,0}+C_{1,1},\\
    &|L^{\rm TP}_\nu(h) - \widehat{L}^{\rm TP}_\mc{D}(h)| \leq C_{0,1}+C_{1,1},\\
    &|L^{\rm FP}_\nu(h) - \widehat{L}^{\rm FP}_\mc{D}(h)| \leq C_{0,0}+C_{1,0},
\end{align*}
with confidence terms 
\iftoggle{arxiv}{
\begin{align*}
    &C_{j, k} =\left( \widehat{p}_{j,k} +  \sqrt{2\widehat{\V}^{(1)}_{j,k} \frac{\log(2/\delta)}{n}} + \frac{\log(2/\delta)}{n}\right)\times\frac{\sqrt{2\widehat{\V}^{(2)}_{j,k} \frac{\log(2/\delta)}{n}} + \frac{\log(2/\delta)}{n}}{\left(\frac{1}{n}\sum_{i=1}^n \1\{y_i = k,  a_i = j \}\right)^2} + \frac{ \sqrt{2\widehat{\V}^{(1)}_{j, k} \frac{\log(2/\delta)}{n}} + \frac{\log(2/\delta)}{n}}{\frac{1}{n}\sum_{i=1}^n \1\{y_i = k, a_i = j \}} 
\end{align*}
}{
\begin{align*}
    &C_{j, k} =\left( \widehat{p}_{j,k} +  \sqrt{2\widehat{\V}^{(1)}_{j,k} \frac{\log(2/\delta)}{n}} + \frac{\log(2/\delta)}{n}\right)\times\\
    &\qquad\qquad\qquad\qquad\times\frac{\sqrt{2\widehat{\V}^{(2)}_{j,k} \frac{\log(2/\delta)}{n}} + \frac{\log(2/\delta)}{n}}{\left(\frac{1}{n}\sum_{i=1}^n \1\{y_i = k,  a_i = j \}\right)^2}
    \\
    &\qquad\qquad
    + \frac{ \sqrt{2\widehat{\V}^{(1)}_{j, k} \frac{\log(2/\delta)}{n}} + \frac{\log(2/\delta)}{n}}{\frac{1}{n}\sum_{i=1}^n \1\{y_i = k, a_i = j \}} 
\end{align*}
}
% \awcomment{maybe state in words what $j$ and $k$ correspond to, eg $j$ is the protected attribute}
for label $k\in\{0, 1\}$ and protected attribute $j\in\{0, 1\}$, where $\widehat{p}_{j,k}=\frac{1}{n}\sum_{i=1}^n \1\{h(x_i) = 1,y_i = k , a_i = j \}$ and the empirical variances defined as
\iftoggle{arxiv}{
\begin{align*}
&\widehat{\V}^{(1)}_{j,k} = \frac{1}{n(n-1)}\sum_{1 \leq \ell < \ell' \leq n} (\1\{h(x_\ell) = 1, y_\ell = k, a_\ell = j \} - \1\{h(x_{\ell'}) = 1, y_{\ell'} = k, a_{\ell'} = j \} )^2,\\
&\widehat{\V}^{(2)}_{j,k} = \frac{1}{n(n-1)}\sum_{1 \leq \ell < \ell' \leq n} (\1\{y_\ell = k, a_\ell = j \} - \1\{y_{\ell'} = k, a_{\ell'} = j \} )^2.
\end{align*}}{
\begin{align*}
&\widehat{\V}^{(1)}_{j,k} = \frac{1}{n(n-1)}\sum_{1 \leq \ell < \ell' \leq n} (\1\{h(x_\ell) = 1, y_\ell = k, a_\ell = j \} \\
&\qquad\qquad\qquad - \1\{h(x_{\ell'}) = 1, y_{\ell'} = k, a_{\ell'} = j \} )^2,\\
&\widehat{\V}^{(2)}_{j,k} = \frac{1}{n(n-1)}\sum_{1 \leq \ell < \ell' \leq n} (\1\{y_\ell = k, a_\ell = j \}\\
&\qquad\qquad\qquad - \1\{y_{\ell'} = k, a_{\ell'} = j \} )^2.
\end{align*}
}
\end{theorem}
This theorem provides a confidence bound on the concentration rate of the empirical fairness violation. 
% \awcomment{is it actually anytime? anytime would imply that it holds for all values of $n$ simultaneously, but I think it just holds for a fixed $n$?} 
\begin{proof}
Let us start by proving the statement for $\rm TPRP$. Recall 
\iftoggle{arxiv}{
\begin{align*}
    &L^{\rm TP}_\nu(h) = \Bigg | \frac{P_{(x, a, y)\sim\nu}(h(x)=1,a=0, y=1)}{P_{(x, a, y)\sim\nu}(a=0, y=1)} - \frac{P_{(x, a, y)\sim\nu}(h(x)=1,a=1, y=1)}{P_{(x, a, y)\sim\nu}(a=1, y=1)} \Bigg |\\
    &\widehat{L}^{\rm TP}_\mc{D}(h) = \Bigg | \sum_{i=1}^{n}\frac{\1\{h(x_i)=1,y_i=1,a_i=1\}}{\sum_{i=1}^{n}\1\{y_i=1,a_i=1\}} - \sum_{i=1}^{n}\frac{\1\{h(x_i)=1,y_i=1,a_i=0\}}{\sum_{i=1}^{n}\1\{y_i=1,a_i=0\}} \Bigg |.
\end{align*}
}{
\begin{align*}
    &L^{\rm TP}_\nu(h) = \Bigg | \frac{P_{(x, a, y)\sim\nu}(h(x)=1,a=0, y=1)}{P_{(x, a, y)\sim\nu}(a=0, y=1)} \\
    &\qquad\qquad\qquad- \frac{P_{(x, a, y)\sim\nu}(h(x)=1,a=1, y=1)}{P_{(x, a, y)\sim\nu}(a=1, y=1)} \Bigg |\\
    &\widehat{L}^{\rm TP}_\mc{D}(h) = \Bigg | \sum_{i=1}^{n}\frac{\1\{h(x_i)=1,y_i=1,a_i=1\}}{\sum_{i=1}^{n}\1\{y_i=1,a_i=1\}} \\
    &\qquad\qquad\qquad- \sum_{i=1}^{n}\frac{\1\{h(x_i)=1,y_i=1,a_i=0\}}{\sum_{i=1}^{n}\1\{y_i=1,a_i=0\}} \Bigg |.
\end{align*}
}
and write these for short
\begin{align*}
&L^{\rm TP}_\nu(h) = | \text{num}_0/ \text{den}_0 -  \text{num}_1 / \text{den}_1|,\\
&\widehat{L}^{\rm TP}_\mc{D}(h) = | \widehat{\text{num}}_0/ \widehat{\text{den}}_0 -  \widehat{\text{num}}_1 / \widehat{\text{den}}_1|,
\end{align*}
with for protected attribute $j\in\{0, 1\}$,
\begin{align*}
&\text{num}_j = P_{(x, a, y)\sim\nu}(h(x)=1,a=j, y=1)\\
&\widehat{\text{num}}_j = \frac{1}{n}\sum_{i=1}^{n}\1\{h(x_i)=1,y_i=1,a_i=j\}\\
&\text{den}_j = P_{(x, a, y)\sim\nu}(a=j, y=1)\\
&\widehat{\text{den}}_j = \frac{1}{n}\sum_{i=1}^{n}\1\{y_i=1,a_i=j\}.
\end{align*}
Applying Bernstein's concentration bound it holds that for $j\in\{0,1\}$ with probability at least $1-\delta$
\iftoggle{arxiv}{
\begin{align*}
    |\widehat{\text{num}}_j - \text{num}_j| 
    &= \Bigg|\frac{1}{n}\sum_{i=1}^n \1\{h(x_i) = 1, y_i = 1, a_i = j \} - \P_{(x, a, y)\sim\nu}(h(x) = 1, y = 1, a = j )\Bigg|\\
    &\leq \sqrt{2\widehat{\V}^{(1)}_{j,1} \frac{\log(2/\delta)}{n}} + \frac{\log(2/\delta)}{n} =: \alpha^{(\text{num})}_j,
\end{align*}
}{
\begin{align*}
    |\widehat{\text{num}}_j - \text{num}_j| 
    &= \Bigg|\frac{1}{n}\sum_{i=1}^n \1\{h(x_i) = 1, y_i = 1, a_i = j \} \\
    &\qquad\qquad- \P_{(x, a, y)\sim\nu}(h(x) = 1, y = 1, a = j )\Bigg|\\
    &\leq \sqrt{2\widehat{\V}^{(1)}_{j,1} \frac{\log(2/\delta)}{n}} + \frac{\log(2/\delta)}{n} =: \alpha^{(\text{num})}_j,
\end{align*}
}
where we defined 
\iftoggle{arxiv}{
\begin{align*}
&\widehat{\V}^{(1)}_{j,k} = \frac{1}{n(n-1)}\sum_{1 \leq \ell < \ell' \leq n} (\1\{h(x_\ell) = 1, y_\ell = k, a_\ell = j \} - \1\{h(x_{\ell'}) = 1, y_{\ell'} = k, a_{\ell'} = j \} )^2.
\end{align*}
}{
\begin{align*}
&\widehat{\V}^{(1)}_{j,k} = \frac{1}{n(n-1)}\sum_{1 \leq \ell < \ell' \leq n} (\1\{h(x_\ell) = 1, y_\ell = k, a_\ell = j \} \\
&\qquad\qquad\qquad\qquad- \1\{h(x_{\ell'}) = 1, y_{\ell'} = k, a_{\ell'} = j \} )^2.
\end{align*}
}
Also applying Bernstein's concentration bound it holds that for $j\in\{0,1\}$ with probability at least $1-\delta$
\iftoggle{arxiv}{
\begin{align*}
    |\widehat{\text{den}}_j - \text{den}_j| 
    &= \Bigg|\frac{1}{n}\sum_{i=1}^n \1\{y_i = 1, a_i = j \} - \P_{(x, a, y)\sim\nu}(y = 1, a = j )\Bigg|\\
    &\leq \sqrt{2\widehat{\V}^{(2)}_{j,1} \frac{\log(2/\delta)}{n}} + \frac{\log(2/\delta)}{n} =: \alpha^{(\text{den})}_j,
\end{align*}
}{
\begin{align*}
    |\widehat{\text{den}}_j - \text{den}_j| 
    &= \Bigg|\frac{1}{n}\sum_{i=1}^n \1\{y_i = 1, a_i = j \} -\\
    &\qquad\qquad\qquad\qquad \P_{(x, a, y)\sim\nu}(y = 1, a = j )\Bigg|\\
    &\leq \sqrt{2\widehat{\V}^{(2)}_{j,1} \frac{\log(2/\delta)}{n}} + \frac{\log(2/\delta)}{n} =: \alpha^{(\text{den})}_j,
\end{align*}
}
where we defined 
\iftoggle{arxiv}{
\begin{align*}
&\widehat{\V}^{(2)}_{j,k} = \frac{1}{n(n-1)}\sum_{1 \leq \ell < \ell' \leq n} (\1\{y_\ell = k, a_\ell = j \} - \1\{y_{\ell'} = k, a_{\ell'} = j \} )^2.
\end{align*}
}{
\begin{align*}
&\widehat{\V}^{(2)}_{j,k} = \frac{1}{n(n-1)}\sum_{1 \leq \ell < \ell' \leq n} (\1\{y_\ell = k, a_\ell = j \} \\
&\qquad\qquad\qquad\qquad- \1\{y_{\ell'} = k, a_{\ell'} = j \} )^2.
\end{align*}
}
Then, as soon as for both $j=1$ and $j=2$, $\alpha^{(\text{den})}_j\leq \widehat{\text{den}}_j/2$, holds the inequality
\begin{align*}
    \left|\frac{1}{\widehat{\text{den}}_j} - \frac{1}{\text{den}_j}\right| \leq \frac{\alpha^{(\text{den})}_j}{\widehat{\text{den}}_j^2},
\end{align*}
so that for $j\in\{0, 1\}$, we have
\begin{align*}
    \left|\frac{\widehat{\text{num}}_j}{\widehat{\text{den}}_j} - \frac{\text{num}_j}{\text{den}_j}\right| 
    &\!=\! \left|\frac{\widehat{\text{num}}_j}{\widehat{\text{den}}_j} - \frac{\text{num}_j}{\widehat{\text{den}}_j} - \frac{\text{num}_j}{\widehat{\text{den}}_j} - \frac{\text{num}_j}{\text{den}_j}\right| \\
    &\!\leq \!\left|\frac{\widehat{\text{num}}_j}{\widehat{\text{den}}_j} \!-\! \frac{\text{num}_j}{\widehat{\text{den}}_j}\right| \!+\! \left|\frac{\text{num}_j}{\widehat{\text{den}}_j} \!-\! \frac{\text{num}_j}{\text{den}_j}\right| \\
    &\!\leq \! \frac{\alpha^{(\text{num})}_j}{\widehat{\text{den}}_j} + \frac{\text{num}_j \alpha^{(\text{den})}_j}{\widehat{\text{den}}_j^2}\\
    &\!\leq \! \frac{\alpha^{(\text{num})}_j}{\widehat{\text{den}}_j} + \frac{(\alpha^{(\text{num})}_j + \widehat{\text{num}}_j) \alpha^{(\text{den})}_j}{\widehat{\text{den}}_j^2}.
\end{align*}
Note that $C_{j, 1}$ is exactly the last upper bound above, 
\iftoggle{arxiv}{
\begin{align*}
    &C_{j, 1} =\left( \widehat{p}_{j,1} +  \sqrt{2\widehat{\V}^{(1)}_{j,1} \frac{\log(2/\delta)}{n}} + \frac{\log(2/\delta)}{n} \right)\times\frac{\sqrt{2\widehat{\V}^{(2)}_{j,1} \frac{\log(2/\delta)}{n}} + \frac{\log(2/\delta)}{n}}{\left(\frac{1}{n}\sum_{i=1}^n \1\{y_i = 1,  a_i = j \}\right)^2} + \frac{ \sqrt{2\widehat{\V}^{(1)}_{j, 1} \frac{\log(2/\delta)}{n}} + \frac{\log(2/\delta)}{n}}{\frac{1}{n}\sum_{i=1}^n \1\{y_i = 1, a_i = j \}} 
\end{align*}
}{
\begin{align*}
    &C_{j, 1} =\left( \widehat{p}_{j,1} +  \sqrt{2\widehat{\V}^{(1)}_{j,1} \frac{\log(2/\delta)}{n}} + \frac{\log(2/\delta)}{n} \right)\times\\
    &\qquad\qquad\qquad\qquad\times\frac{\sqrt{2\widehat{\V}^{(2)}_{j,1} \frac{\log(2/\delta)}{n}} + \frac{\log(2/\delta)}{n}}{\left(\frac{1}{n}\sum_{i=1}^n \1\{y_i = 1,  a_i = j \}\right)^2}\\
    &\qquad\qquad+ \frac{ \sqrt{2\widehat{\V}^{(1)}_{j, 1} \frac{\log(2/\delta)}{n}} + \frac{\log(2/\delta)}{n}}{\frac{1}{n}\sum_{i=1}^n \1\{y_i = 1, a_i = j \}} 
\end{align*}
}
where $\widehat{p}_{j,1}=\frac{1}{n}\sum_{i=1}^n \1\{h(x_i) = 1,y_i = 1 , a_i = j \}$. Putting it together
\iftoggle{arxiv}{
\begin{align*}
    |L^{\rm TP}_\nu(h) - \widehat{L}^{\rm TP}_\mc{D}(h) | 
    &= | | \text{num}_0/ \text{den}_0 -  \text{num}_1 / \text{den}_1|  - | \widehat{\text{num}}_0/ \widehat{\text{den}}_0 - \widehat{\text{num}}_1 / \widehat{\text{den}}_1| |,\\
    &\leq | \text{num}_0/ \text{den}_0 -  \text{num}_1 / \text{den}_1  -  \widehat{\text{num}}_0/ \widehat{\text{den}}_0 + \widehat{\text{num}}_1 / \widehat{\text{den}}_1| ,\\
    &\leq | \text{num}_0/ \text{den}_0 - \widehat{\text{num}}_0 / \widehat{\text{den}}_0| +| \widehat{\text{num}}_1/ \widehat{\text{den}}_1  -  \text{num}_1 / \text{den}_1 |,\\
    &\leq C_{0, 1} + C_{1, 1}.
\end{align*}
}{
\begin{align*}
    |L^{\rm TP}_\nu(h) - \widehat{L}^{\rm TP}_\mc{D}(h) | 
    &= | | \text{num}_0/ \text{den}_0 -  \text{num}_1 / \text{den}_1| \\
    &\qquad- | \widehat{\text{num}}_0/ \widehat{\text{den}}_0 - \widehat{\text{num}}_1 / \widehat{\text{den}}_1| |,\\
    &\leq | \text{num}_0/ \text{den}_0 -  \text{num}_1 / \text{den}_1 \\
    &\qquad-  \widehat{\text{num}}_0/ \widehat{\text{den}}_0 + \widehat{\text{num}}_1 / \widehat{\text{den}}_1| ,\\
    &\leq | \text{num}_0/ \text{den}_0 - \widehat{\text{num}}_0 / \widehat{\text{den}}_0|\\
    &\qquad+| \widehat{\text{num}}_1/ \widehat{\text{den}}_1  -  \text{num}_1 / \text{den}_1 |,\\
    &\leq C_{0, 1} + C_{1, 1}.
\end{align*}
}
which is the conclusion for TPRP.

As $\widehat{L}^{\rm FP}_\mc{D}(h)$ was defined as the empirical estimate of the FPRP violation by conditioning on $\1\{y_i=0\}$ (instead of $\1\{y_i=1\}$ for TPRP), the proof for the concentration bound on FPRP is analogous to the one of TPRP, with the exception of the conditioning on $\1\{y_i=0\}$ instead of $\1\{y_i=1\}$ for TPRP.

We defined the empirical estimate of the EO violation as the maximum of empirical estimate of the TPRP violation and the empirical estimate of the FPRP violation, $\widehat{L}^{\rm EO}_\mc{D}(h) = \max\{\widehat{L}^{\rm TP}_\mc{D}(h), \widehat{L}^{\rm FP}_\mc{D}(h)\}$, so holds
$$\widehat{L}^{\rm EO}_\mc{D}(h) \leq \widehat{L}^{\rm TP}_\mc{D}(h) + \widehat{L}^{\rm FP}_\mc{D}(h),$$
which immediately leads to the conclusion of \Cref{thm:fairness_est}.
\end{proof}

\subsection{Proof of \Cref{cor:fairness_est_simple}}
We first state the full result that leads to the statement of \Cref{cor:fairness_est_simple}. 
\begin{proposition}\label{cor:fairness_est_simple_full}
    Let the train set be  $\mc{D} = \{(x_1, a_1, y_1), \ldots,(x_n, a_n, y_n)\}$. If $\mc{D}\sim \nu$, then it holds with probability $1-\delta$ that:
\iftoggle{arxiv}{
\begin{align*}
    &|L^{\rm TP}_\nu(h) - \widehat{L}^{\rm TP}_D(h)| \leq 2\max_{j\in\{0, 1\}} \Bigg\{2\left(\frac{\sqrt{2\frac{\log(2/\delta)}{n}} + \frac{\log(2/\delta)}{n} }{\frac{1}{n}\sum_{i=1}^n \1\{y_i = 1, a_i = j \}}\right)  + \left(\frac{\sqrt{2\frac{\log(2/\delta)}{n}} + 2\frac{\log(2/\delta)}{n}}{\frac{1}{n}\sum_{i=1}^n \1\{y_i = 1, a_i = j \}}\right)^2\Bigg\},\\
    &|L^{\rm EO}_\nu(h) - \widehat{L}^{\rm EO}_D(h)| \leq 4\max_{0\leq j, k \leq 1} \Bigg\{2\left(\frac{\sqrt{2\frac{\log(2/\delta)}{n}} + \frac{\log(2/\delta)}{n} }{\frac{1}{n}\sum_{i=1}^n \1\{y_i = k, a_i = j \}}\right) + \left(\frac{\sqrt{2\frac{\log(2/\delta)}{n}} + 2\frac{\log(2/\delta)}{n}}{\frac{1}{n}\sum_{i=1}^n \1\{y_i = k, a_i = j \}}\right)^2\Bigg\}.
\end{align*}
}{
\begin{align*}
    &|L^{\rm TP}_\nu(h) - \widehat{L}^{\rm TP}_D(h)| \leq\\
    &\qquad\qquad 2\max_{j\in\{0, 1\}} \Bigg\{2\left(\frac{\sqrt{2\frac{\log(2/\delta)}{n}} + \frac{\log(2/\delta)}{n} }{\frac{1}{n}\sum_{i=1}^n \1\{y_i = 1, a_i = j \}}\right) \\
    &\qquad\qquad\qquad\qquad+ \left(\frac{\sqrt{2\frac{\log(2/\delta)}{n}} + 2\frac{\log(2/\delta)}{n}}{\frac{1}{n}\sum_{i=1}^n \1\{y_i = 1, a_i = j \}}\right)^2\Bigg\},\\
    &|L^{\rm EO}_\nu(h) - \widehat{L}^{\rm EO}_D(h)| \leq \\
    &\qquad\qquad 4\max_{0\leq j, k \leq 1} \Bigg\{2\left(\frac{\sqrt{2\frac{\log(2/\delta)}{n}} + \frac{\log(2/\delta)}{n} }{\frac{1}{n}\sum_{i=1}^n \1\{y_i = k, a_i = j \}}\right) \\
    &\qquad\qquad\qquad\qquad+ \left(\frac{\sqrt{2\frac{\log(2/\delta)}{n}} + 2\frac{\log(2/\delta)}{n}}{\frac{1}{n}\sum_{i=1}^n \1\{y_i = k, a_i = j \}}\right)^2\Bigg\}.
\end{align*}
    }
\end{proposition}

\begin{proof}[Proof of \Cref{cor:fairness_est_simple} and \ref{cor:fairness_est_simple_full}]
We use \Cref{thm:fairness_est} and for label $k\in\{0, 1\}$ and protected attribute $j\in\{0, 1\}$ we bound $C_{j, k}$.

We first have that the empirical variances are such that $\widehat{\V}^{(1)}_{j,k} \leq 1$ and $\widehat{\V}^{(2)}_{j,k} \leq 1$. Also, 
\iftoggle{arxiv}{
\begin{align*}
\widehat{p}_{j,k} 
&= \frac{1}{n}\sum_{i=1}^n \1\{h(x_i) = 1,y_i = k , a_i = j \}\leq \frac{1}{n}\sum_{i=1}^n \1\{y_i = k , a_i = j \}.
\end{align*}
}{
\begin{align*}
\widehat{p}_{j,k} 
&= \frac{1}{n}\sum_{i=1}^n \1\{h(x_i) = 1,y_i = k , a_i = j \} \\
&\leq \frac{1}{n}\sum_{i=1}^n \1\{y_i = k , a_i = j \}.
\end{align*}
}
Thus, we can bound
\iftoggle{arxiv}{
\begin{align*}
    &C_{j, k} = \left( \widehat{p}_{j,k} +  \sqrt{2\widehat{\V}^{(1)}_{j,k} \frac{\log(2/\delta)}{n}}  + \frac{\log(2/\delta)}{n}\right)\times \frac{\sqrt{2\widehat{\V}^{(2)}_{j,k} \frac{\log(2/\delta)}{n}} + \frac{\log(2/\delta)}{n}}{\left(\frac{1}{n}\sum_{i=1}^n \1\{y_i = k,  a_i = j \}\right)^2} + \frac{ \sqrt{2\widehat{\V}^{(1)}_{j, k} \frac{\log(2/\delta)}{n}}  + \frac{\log(2/\delta)}{n}}{\frac{1}{n}\sum_{i=1}^n \1\{y_i = k, a_i = j \}} \\
    &\qquad\leq 2\left(\frac{\sqrt{2\frac{\log(2/\delta)}{n}} + \frac{\log(2/\delta)}{n} }{\frac{1}{n}\sum_{i=1}^n \1\{y_i = k, a_i = j \}}\right) + \left(\frac{\sqrt{2\frac{\log(2/\delta)}{n}}2\frac{\log(2/\delta)}{n}}{\frac{1}{n}\sum_{i=1}^n \1\{y_i = k, a_i = j \}}\right)^2.
\end{align*}
}{
\begin{align*}
    &C_{j, k} = \left( \widehat{p}_{j,k} +  \sqrt{2\widehat{\V}^{(1)}_{j,k} \frac{\log(2/\delta)}{n}}  + \frac{\log(2/\delta)}{n}\right)\times\\
    &\qquad\qquad\qquad\qquad\times\frac{\sqrt{2\widehat{\V}^{(2)}_{j,k} \frac{\log(2/\delta)}{n}} + \frac{\log(2/\delta)}{n}}{\left(\frac{1}{n}\sum_{i=1}^n \1\{y_i = k,  a_i = j \}\right)^2}\\
    &\qquad\qquad+ \frac{ \sqrt{2\widehat{\V}^{(1)}_{j, k} \frac{\log(2/\delta)}{n}}  + \frac{\log(2/\delta)}{n}}{\frac{1}{n}\sum_{i=1}^n \1\{y_i = k, a_i = j \}} \\
    &\qquad\leq 2\left(\frac{\sqrt{2\frac{\log(2/\delta)}{n}} + \frac{\log(2/\delta)}{n} }{\frac{1}{n}\sum_{i=1}^n \1\{y_i = k, a_i = j \}}\right) \\
    &\qquad\qquad+ \left(\frac{\sqrt{2\frac{\log(2/\delta)}{n}}2\frac{\log(2/\delta)}{n}}{\frac{1}{n}\sum_{i=1}^n \1\{y_i = k, a_i = j \}}\right)^2.
\end{align*}
}
With that result, we conclude for TPRP that
\iftoggle{arxiv}{
\begin{align*}
    |L^{\rm TP}_\nu(h) - \widehat{L}^{\rm TP}_D(h)| &\leq C_{0, 1} + C_{1, 1} \\
    &\leq 2\max_{j\in\{0, 1\}} C_{j, 1} \\
    &\leq 2\max_{j\in\{0, 1\}} \Bigg\{2\left(\frac{\sqrt{2\frac{\log(2/\delta)}{n}} + \frac{\log(2/\delta)}{n} }{\frac{1}{n}\sum_{i=1}^n \1\{y_i = 1, a_i = j \}}\right) + \left(\frac{\sqrt{2\frac{\log(2/\delta)}{n}} + 2\frac{\log(2/\delta)}{n}}{\frac{1}{n}\sum_{i=1}^n \1\{y_i = 1, a_i = j \}}\right)^2\Bigg\}\\
    &= 4\max_{j\in\{0, 1\}} \frac{\sqrt{2\frac{\log(2/\delta)}{n}}}{\frac{1}{n}\sum_{i=1}^n \1\{y_i = 1, a_i = j \}}+\mc{O}\left(\frac{1}{n}\right).
\end{align*}
}{
\begin{align*}
    &|L^{\rm TP}_\nu(h) - \widehat{L}^{\rm TP}_D(h)| \\
    &\qquad\leq C_{0, 1} + C_{1, 1} \\
    &\qquad\leq 2\max_{j\in\{0, 1\}} C_{j, 1} \\
    &\qquad\leq 2\max_{j\in\{0, 1\}} \Bigg\{2\left(\frac{\sqrt{2\frac{\log(2/\delta)}{n}} + \frac{\log(2/\delta)}{n} }{\frac{1}{n}\sum_{i=1}^n \1\{y_i = 1, a_i = j \}}\right) \\
    &\qquad\qquad+ \left(\frac{\sqrt{2\frac{\log(2/\delta)}{n}} + 2\frac{\log(2/\delta)}{n}}{\frac{1}{n}\sum_{i=1}^n \1\{y_i = 1, a_i = j \}}\right)^2\Bigg\}\\
    &\qquad= 4\max_{j\in\{0, 1\}} \frac{\sqrt{2\frac{\log(2/\delta)}{n}}}{\frac{1}{n}\sum_{i=1}^n \1\{y_i = 1, a_i = j \}}+\mc{O}\left(\frac{1}{n}\right).
\end{align*}
}

Analogous bounds conclude for EO:
\iftoggle{arxiv}{
\begin{align*}
    |L^{\rm EO}_\nu(h) - \widehat{L}^{\rm EO}_D(h)| &\leq C_{0, 0} + C_{1, 0} + C_{0, 1} + C_{1, 1} \\
    &\leq 4\max_{j\in\{0, 1\}} C_{j, k} \\
    &\leq 4\max_{0\leq j, k \leq 1} \Bigg\{2\left(\frac{\sqrt{2\frac{\log(2/\delta)}{n}} \frac{\log(2/\delta)}{n} }{\frac{1}{n}\sum_{i=1}^n \1\{y_i = k, a_i = j \}}\right)  + \left(\frac{\sqrt{2\frac{\log(2/\delta)}{n}} + 2\frac{\log(2/\delta)}{n}}{\frac{1}{n}\sum_{i=1}^n \1\{y_i = k, a_i = j \}}\right)^2 \Bigg\}\\
    &=8\max_{0\leq j, k \leq 1} \frac{\sqrt{2\frac{\log(2/\delta)}{n}}}{\frac{1}{n}\sum_{i=1}^n \1\{y_i = k, a_i = j \}}+\mc{O}\left(\frac{1}{n}\right).
\end{align*}
}{
\begin{align*}
    &|L^{\rm EO}_\nu(h) - \widehat{L}^{\rm EO}_D(h)| \\
    &\qquad\leq C_{0, 0} + C_{1, 0} + C_{0, 1} + C_{1, 1} \\
    &\qquad\leq 4\max_{j\in\{0, 1\}} C_{j, k} \\
    &\qquad\leq 4\max_{0\leq j, k \leq 1} \Bigg\{2\left(\frac{\sqrt{2\frac{\log(2/\delta)}{n}} \frac{\log(2/\delta)}{n} }{\frac{1}{n}\sum_{i=1}^n \1\{y_i = k, a_i = j \}}\right) \\
    &\qquad\qquad+ \left(\frac{\sqrt{2\frac{\log(2/\delta)}{n}} + 2\frac{\log(2/\delta)}{n}}{\frac{1}{n}\sum_{i=1}^n \1\{y_i = k, a_i = j \}}\right)^2 \Bigg\}\\
    &\qquad=8\max_{0\leq j, k \leq 1} \frac{\sqrt{2\frac{\log(2/\delta)}{n}}}{\frac{1}{n}\sum_{i=1}^n \1\{y_i = k, a_i = j \}}+\mc{O}\left(\frac{1}{n}\right).
\end{align*}
}
\end{proof}

\end{document}